%% file: arxiv_fixed.tex
\documentclass{article} 

\input{iclr2025/math_commands}

\include{packages}
\usepackage{hyperref}
\hypersetup{
    colorlinks=true,
    linkcolor=blue,
    filecolor=magenta,      
    urlcolor=cyan,
    pdftitle={Overleaf Example},
    pdfpagemode=FullScreen,
    }
\usepackage{url}
\usepackage[a4paper,left=0.8in,right=0.8in,top=1in,bottom=1in]{geometry}

\usepackage{authblk}
\date{}
\title{Partial Gromov-Wasserstein Metric}

\author[1]{Yikun Bai}
\author[2]{Rocio Diaz Martin\textsuperscript{*}}
\author[1]{Abihith Kothapalli\textsuperscript{*}}
\author[3]{\\Hengrong Du}
\author[1]{Xinran Liu}
\author[1]{Soheil Kolouri}
\affil[1]{Department of Computer Science, Vanderbilt University}
\affil[2]{Department of Mathematics, Tufts University}
\affil[3]{Department of Mathematics, University of California, Irvine}

\affil[1]{\href{mailto:yikun.bai@vanderbilt.edu}{yikun.bai}, \href{mailto:abi.kothapalli@vanderbilt.edu}{abi.kothapalli}, \href{mailto:xinran.liu@vanderbilt.edu}{xinran.liu}, \href{mailto:soheil.kolouri@vanderbilt.edu}{soheil.kolouri@vanderbilt.edu}}
\affil[2]{\href{mailto:rocio.diaz_martin@tufts.edu}{rocio.diaz\_martin@tufts.edu}}
\affil[3]{\href{mailto:hengrond@uci.edu}{hengrond@uci.edu}}

\renewcommand{\eqref}[1]{(\ref{#1})}

\begin{document}
\maketitle
\renewcommand{\thefootnote}{*}
\footnotetext{These authors contributed equally to this work.}
\renewcommand{\thefootnote}{\arabic{footnote}} 


\begin{abstract}\label{sec:abstract}
The Gromov-Wasserstein (GW) distance has gained increasing interest in the machine learning community in recent years, as it allows for the comparison of measures in different metric spaces. To overcome the limitations imposed by the equal mass requirements of the classical GW problem, researchers have begun exploring its application in unbalanced settings. However, Unbalanced GW (UGW) can only be regarded as a discrepancy rather than a rigorous metric/distance between two metric measure spaces (mm-spaces). In this paper, we propose a particular case of the UGW problem, termed Partial Gromov-Wasserstein (PGW). We establish that PGW is a well-defined metric between mm-spaces and discuss its theoretical properties, including the existence of a minimizer for the PGW problem and the relationship between PGW and GW, among others. We then propose two variants of the Frank-Wolfe algorithm for solving the PGW problem and show that they are mathematically and computationally equivalent. Moreover, based on our PGW metric, we introduce the analogous concept of barycenters for mm-spaces. Finally, we validate the effectiveness of our PGW metric and related solvers in applications such as shape matching, shape retrieval, and shape interpolation, comparing them against existing baselines. Our code is available at \url{https://github.com/mint-vu/PGW_Metric}.


\end{abstract}

\section{Introduction}\label{sec:intro}
The classical optimal transport (OT) problem \cite{Villani2009Optimal} seeks to match two probability measures while minimizing the expected transportation cost. At the heart of classical OT theory lies the principle of mass conservation, which aims to optimize the transfer between two probability measures, assuming they have the same total mass and strictly preserving it. Statistical distances that arise from OT, such as Wasserstein distances, have been widely applied across various machine learning domains, ranging from generative modeling \cite{arjovsky2017wasserstein,gulrajani2017improved} to domain adaptation \cite{courty2017joint} and representation learning \cite{kolouri2020wasserstein}. Recent advancements have extended the OT problem to address certain limitations within machine learning applications. These advancements include: 1) facilitating the comparison of non-negative measures that possess different total masses via unbalanced \cite{chizat2018unbalanced} and partial OT \cite{figalli2010optimal}, and 2) enabling the comparison of probability measures across distinct metric spaces through Gromov-Wasserstein distances \cite{memoli2011gromov}, with applications spanning from quantum chemistry \cite{gilmer2017neural} to natural language processing \cite{alvarez2018gromov}.

Regarding the first aspect, many applications in machine learning involve comparing non-negative measures (often empirical measures) with varying total amounts of mass, e.g., domain adaptation \cite{fatras2021unbalanced}. Moreover, OT distances (or dissimilarity measures) are often not robust against outliers and noise, resulting in potentially high transportation costs for outliers. Many recent publications have focused on variants of the OT problem that allow for comparing non-negative measures with unequal mass. For instance, the optimal partial transport problem \cite{figalli2010optimal,caffarelli2010free,figalli2010new, nguyen2024POT,georgiou2008metrics,piccoli2014generalized}, Kantorovich--Rubinstein norm \cite{guittet2002extended,heinemann2023kantorovich,lellmann2014imaging}, and the Hellinger--Kantorovich distance \cite{chizat2018interpolating,Liero2018Optimal}. Recent works formulating the metric properties of partial OT with total variation constraints include \cite{raghvendra2024new,nietert2023robust}.
These methods fall under the broad category of ``unbalanced optimal transport.'' In this regard, we also highlight \cite{balaji2020robust,nguyen2023unbalanced,le2021robust}, which enhance OT's robustness in the presence of outliers. 

 Regarding the second aspect, comparing probability measures across different metric spaces is essential in many machine learning applications, ranging from computer graphics, where shapes and surfaces are compared 
\cite{bronstein2006generalized,memoli2009spectral}, to graph partitioning and matching problems \cite{xu2019scalable}. Source and target distributions often arise from varied conditions, such as different times, contexts, or measurement techniques, creating substantial differences in intrinsic distances among data points. The conventional OT framework necessitates a meaningful distance across diverse domains, a requirement that is not always achievable. To circumvent this issue, the Gromov-Wasserstein (GW) distances were proposed in \cite{memoli2011gromov,memoli2009spectral} as an adaptation of the Gromov-Hausdorff distance, which measures the discrepancy between two metric spaces \cite{edwards1975structure,gromov1981structures, gromov1981groups, burago2001course}. The GW distance \cite{memoli2011gromov,sturm2023space} extends OT-based distances to metric measure spaces (mm-spaces) up to isometries. Its invariance across isomorphic mm-spaces makes the GW distance particularly valuable for applications like shape comparison and matching, where invariance to rigid motion transformations is crucial.

The main computational challenge of the GW metric is the non-convexity of its formulation \cite{memoli2011gromov}. The conventional computational approach relies on the Frank-Wolfe (FW) algorithm \cite{frank1956algorithm,lacoste2016convergence}. Optimal transport (OT) computational methods \cite{guittet2002extended,cuturi2013sinkhorn,papadakis2014optimal,benamou2014numerical,benamou2015iterative,peyre2019computational,chizat2018scaling, Bonneel2019sliced,bai2022sliced}, such as the Sinkhorn algorithm, can be incorporated into FW iterations, which yields the classical GW solvers \cite{peyre2016gromov,xu2019gromov,titouan2019optimal}.

Given that the GW distance is limited to the comparison of probability mm-spaces, recent works have introduced unbalanced and partial variations \cite{sejourne2021unbalanced,chapel2020partial,de2022entropy}. These variations 
have been applied in diverse contexts, including partial graph matching for social network analysis \cite{liu2020partial} and the alignment of brain images \cite{thual2022aligning}. Although solving these unbalanced variants of the GW problem yields notions of \textit{discrepancies} between mm-spaces, their \textit{metric} properties remain unclear in the literature.

Motivated by the emerging applications of the GW problem in unbalanced settings, this paper focuses on developing a metric between general (not necessarily probability) mm-spaces and providing efficient solvers for its computation. Our proposed metric arises from formulating a variant of the GW problem for unbalanced contexts, rooted in the framework provided by \cite{sejourne2021unbalanced}, which we named the \textit{Partial Gromov-Wasserstein} (PGW) problem. In contrast to \cite{sejourne2021unbalanced}, which introduces a KL-divergence penalty and a Sinkhorn solver, we employ a total variation penalty, demonstrate the resulting metric properties, and provide novel, efficient solvers for this problem.
To the best of our knowledge, this paper presents the first metric for non-probability mm-spaces based on the GW distance. 

\textbf{Contributions.} Our specific contributions to this paper are:
\begin{itemize}
\item \textbf{GW metric in unbalanced settings.} We propose the Partial Gromov-Wasserstein (PGW) problem and prove that it gives rise to a metric between arbitrary mm-spaces. 
\item \textbf{PGW solver. \footnote{Rigorously speaking, due to the non-convexity of the GW problem and its variants, current methods achieve only local minima rather than global minima. We use the term ``solver'' following the convention of previous related works \cite{sejourne2021unbalanced,chapel2020partial}. However, it should be emphasized that the proposed methods aim to find local minima rather than global minima, similar to related and classical computational GW works.}} Analogous to the technique presented in \cite{caffarelli2010free}, we show that the PGW problem can be turned into a variant of the GW problem. Based on this relation, we propose two mathematically equivalent, but distinct in numerical implementation, Frank-Wolfe solvers for the discrete PGW problem. Inspired by the results of \cite{lacoste2016convergence}, we prove that similar to the Frank-Wolfe solver presented in \cite{chapel2020partial}, our proposed solvers for the PGW problem converge linearly to a stationary point. 
    
\item \textbf{Numerical experiments.} We demonstrate the performance of our proposed algorithms in terms of computation time and efficacy on a series of tasks: shape-matching with outliers between 2D and 3D objects, shape retrieval between 2D shapes, and shape interpolation using the concept of PGW barycenters. We compare the performance of our proposed algorithms against existing baselines for each task.

\end{itemize}
\section{Background}\label{sec: background} 
In this section, we review the basics of OT theory, one of its variants in unbalanced contexts called Partial OT (POT), and their connection as established in \cite{caffarelli2010free}. We then introduce the GW distance. 

\subsection{Optimal Transport and Partial Optimal Transport}\label{subsec: ot pot}
Let $\Omega\subseteq\mathbb{R}^d$ be, for simplicity, a compact subset of $\mathbb{R}^d$, and $\mathcal{P}(\Omega)$ be the space of probability measures defined on the Borel $\sigma$-algebra of $\Omega$.

\textbf{The Optimal Transport (OT) problem} for $\mu,\nu\in\mathcal{P}(\Omega)$,  with transportation cost $c(x,y): \Omega\times \Omega\to \mathbb{R}_+$ being a
lower-semi continuous function is defined as: 
\begin{equation}
OT(\mu,\nu):=\min_{\gamma\in \Gamma(\mu,\nu)}\gamma(c), \qquad \text{ where } \quad 
\gamma(c):=\int_{\Omega^2}c(x,y) \, d\gamma(x,y) 
\label{eq: OT}
\end{equation}
and where $\Gamma(\mu,\nu)$ denotes the set of all joint probability measures on $\Omega^2:=\Omega\times \Omega$ with marginals $\mu,\nu$, i.e., 
$\gamma_1:=\pi_{1\#}\gamma=\mu,
 \gamma_2:= \pi_{2\#}\gamma=\nu$,
where $\pi_1,\pi_2:\Omega^2\to \Omega$ are the canonical projections  $\pi_1(x,y):=x,\pi_2(x,y):=y$. 
A minimizer for \eqref{eq: OT} always exists \cite{Villani2009Optimal,villani2021topics} and when $c(x,y)=\|x-y\|^p$, for $p\ge 1$, 
it defines a metric on $\mathcal{P}(\Omega)$, which is referred to as the ``$p$-Wasserstein distance'':
\begin{align}
W_p^p(\mu,\nu):=\min_{\gamma\in\Gamma(\mu,\nu)}\int_{\Omega^2}\|x-y\|^pd\gamma(x,y). \label{eq: Wp}
\end{align}
\textbf{The Partial Optimal Transport (POT) problem} \cite{chizat2018unbalanced, figalli2010new,piccoli2014generalized} extends the OT problem to the set of Radon measures $\mathcal{M}_+(\Omega)$, i.e., non-negative and finite measures. For $\lambda>0$ and $\mu,\nu\in \mathcal{M}_+(\Omega)$, the POT problem is defined as:
\begin{equation}
POT(\mu,\nu;\lambda):=\inf_{\gamma\in\mathcal{M}_+(\Omega^2)}\gamma(c)+\lambda(|\mu-\gamma_1|+|\nu-\gamma_2|),
\label{eq: opt}
\end{equation}
where, in general, $|\sigma|$ denotes the total variation norm of a measure $\sigma$, i.e., $|\sigma|:=\sigma(\Omega)$.  
The constraint $\gamma\in \mathcal{M}_+(\Omega^2)$ in \eqref{eq: opt} can be further restricted to $\gamma \in \Gamma_\leq(\mu,\nu)$: 
$$\Gamma_\leq(\mu,\nu):=\{\gamma\in \mathcal{M}_+(\Omega^2): \, \gamma_1\leq \mu,\gamma_2\leq \nu\},$$ denoting $\gamma_1\leq \mu$ if for any Borel set $B\subseteq\Omega$, $\gamma_1(B)\leq \mu(B)$ (respectively, for $\gamma_2\leq \nu$) \cite{figalli2010optimal}. Roughly speaking, the linear penalization indicates that if the classical transportation cost exceeds 
$2\lambda$, it is better to create/destroy' mass (see \cite{bai2022sliced} for further details).  

\noindent\textbf{The relationship between POT and OT.}
By using the techniques in \cite{caffarelli2010free}, the POT problem can be transferred into an OT problem, and thus, OT solvers (e.g., network simplex) can be employed to solve the POT problem. 

\begin{proposition}\label{pro: opt and ot}\cite{caffarelli2010free,bai2022sliced} Given $\mu,\nu\in\mathcal{M}_+(\Omega)$, construct the following measures on $\hat\Omega:=\Omega\cup \{\hat\infty\}$,  for an auxiliary point $\hat\infty$:
\begin{equation}\label{eq: mu_hat}
\hat\mu=\mu+|\nu|\delta_{\hat\infty}   \quad \text{ and } \quad
\hat\nu=\nu+|\mu|\delta_{\hat\infty} .      
\end{equation}
Consider the following OT problem 
\begin{equation}\label{eq: OT variant}
    \text{OT}(\hat\mu,\hat\nu)=\min_{\hat\gamma\in\Gamma(\hat\mu,\hat\nu)}\hat{\gamma}(\hat{c}),
    \qquad \text{where}  \quad \hat{c}(x,y)
 :=\begin{cases}
     c(x,y)-2\lambda &\text{ if }x,y\in\Omega,\\ 
     0  &\text{ elsewhere}.
\end{cases} 
\end{equation}
Then, there exists a bijection $F: \Gamma_\leq(\mu,\nu)\to\Gamma(\hat\mu,\hat\nu)$ 
given by 
\begin{equation}\label{eq: ot and opt plan}
F(\gamma):=
\gamma+(\mu-\gamma_1)\otimes\delta_{\hat\infty}+\delta_{\hat\infty}\otimes (\nu-\gamma_2)+|\gamma|\delta_{\hat\infty,\hat\infty}.
\end{equation}
such that $\gamma$ is optimal for the POT problem \eqref{eq: opt} if and only if $F(\gamma)$ is optimal for the OT problem \eqref{eq: OT variant}.   
\end{proposition}


It is worth noting that instead of considering the same underlying space $\Omega$ for both measures $\mu$ and $\nu$, the OT and POT problems can be formulated in the scenario where $\mu$ and $\nu$ are defined on different metric spaces
$X$ and $Y$, respectively. In this setting, one needs a cost function $c:X\times Y\to\mathbb{R}_+$ to formulate the OT and POT problems. However, in practice, it is usually difficult to define reasonable `distance' or \textit{ground cost} $c(\cdot,\cdot)$ between the two spaces $X$ and $Y$. In particular, the $p$-Wasserstein distance cannot be adopted if $\mu,\nu$ are defined on different spaces.  
To relax this requirement, in the next section, we will review the fundamentals of the \textit{Gromov-Wasserstein} problem \cite{memoli2011gromov}.
\subsection{The Gromov-Wasserstein (GW) Problem}\label{subsec: GW}
A metric measure space (mm-space) consists of a set $X$ endowed with a metric structure, that is, a notion of distance $d_X$ between its elements, and equipped with a Borel measure $\mu$. 
As in \cite[Ch. 5]{memoli2011gromov}, we will assume that $X$ is compact and that  $\operatorname{supp}(\mu)=X$.
Given two probability mm-spaces $\mathbb{X}=(X,d_X,\mu)$, $\mathbb{Y}=(Y,d_Y,\nu)$, with $\mu\in \mathcal{P}(X)$ and $\nu\in\mathcal{P}(Y)$, and a non-negative lower semi-continuous cost function  $L: \mathbb{R}^2\to \mathbb{R}_+$  (e.g., the Euclidean distance or the KL-loss),  the Gromov-Wasserstein (GW) matching problem is defined as:
\begin{equation}
GW^L(\mathbb{X},\mathbb{Y}):=\inf_{\gamma\in \Gamma(\mu,\nu)}\gamma^{\otimes 2}(L(d_X(\cdot,\cdot),d_Y(\cdot,\cdot))), 
\label{eq:gw}
\end{equation}
 where, for brevity, we employ the notation $\gamma^{\otimes 2}$ for the product measure $d\gamma^{\otimes 2}((x,y),(x',y'))=d\gamma(x,y)d\gamma(x',y')$. 
If 
$L(a,b)=|a-b|^p$, 
for $1\leq p< \infty$, 
we denote $GW^L(\cdot,\cdot)$ simply by $GW^p(\cdot,\cdot)$.
In this case, the expression \eqref{eq:gw} defines an equivalence relation $\sim$ among probability mm-spaces, i.e.,  $\mathbb{X}\sim\mathbb{Y}$ if and only if  $GW^p(\mathbb{X},\mathbb{Y})=0$\footnote{Moreover, given two probability mm-spaces $\mathbb{X}$ and $\mathbb{Y}$, $GW(\mathbb{X},\mathbb{Y})=0$ if and only if there exists a bijective isometry $\phi:X\to Y$ such that $\phi_\#\mu=\nu$. 
In particular, the GW distance is invariant under rigid transformations (translations and rotations) of a given probability mm-space.}. A minimizer of the GW problem \eqref{eq:gw} always exists, and thus, we can replace $\inf$ by $\min$. Moreover,  similar to OT, the above GW problem defines a distance for probability mm-spaces after taking the quotient under $\sim$. For details, we refer to \cite[Ch. 5 and 10]{memoli2011gromov}. 



\section{The Partial Gromov-Wasserstein (PGW) Problem}\label{sec: pgw}
The Unbalanced Gromov-Wasserstein (UGW) problem for general (compact)
 mm-spaces $\mathbb{X}=(X,d_X,\mu),\mathbb{Y}=(Y,d_Y,\nu)$, with $\mu\in\mathcal{M}_+(X),\nu\in\mathcal{M}_+(Y)$, studied in \cite{sejourne2021unbalanced,kong2024outlier} is defined as:
\begin{align}
UGW^L_\lambda(\mathbb{X},\mathbb{Y}):=&\inf_{\gamma\in \mathcal{M}_+(X\times Y)}\gamma^{\otimes2}(L(d_X,d_Y))
+\lambda(D_\phi(\gamma_1^{\otimes2}
\parallel\mu^{\otimes2})+D_\phi(\gamma_2^{\otimes2}\parallel\nu^{\otimes2})),\label{eq:ugw}
\end{align}
where $\lambda>0$ is a fixed linear penalization parameter, and $D_\phi$ is a  Csiszár or $\phi$-divergence. 
The above formulation extends the classical GW problem \eqref{eq:gw} into the unbalanced setting ($\mu$ and $\nu$ are no longer necessarily probability measures but general Radon measures). 

We underline two points: First, as discussed in \cite{sejourne2021unbalanced}, while the above quantity allows us to `compare' the mm-spaces $\mathbb{X}$ and $\mathbb{Y}$, its \textit{metric} property is unclear. Secondly, when $D_\phi$ is the KL divergence, a Sinkhorn solver has been proposed in \cite{sejourne2021unbalanced}. However, a solver for general $\phi$-divergences has not yet been proposed.

In this paper, we will analyze the case when $D_{\phi}$ is the total variation norm. Specifically, for $q\geq 1$, we consider the following problem, which we refer to as the \textit{Partial Gromov-Wasserstein} (PGW) problem:
\begin{equation}
PGW_{\lambda,q}^L(\mathbb{X},\mathbb{Y}):=\inf_{\gamma \in \mathcal{M}_+(X\times Y)}\gamma^{\otimes2}(L(d_X^q,d_Y^q))+\lambda(|\mu^{\otimes2}-\gamma_1^{\otimes2}|+|\nu^{\otimes2}-\gamma_2^{\otimes2}|).\label{eq:pgw_ori}%
\end{equation}

\begin{remark}\label{remark: rewrite pGW}
   If  $\gamma\in\Gamma\leq(\mu,\nu)$, the above cost functional can be rewritten as
\begin{align}
\gamma^{\otimes2}(L(d_X^q,d_Y^q))+\lambda(|\mu^{\otimes2}-\gamma_1^{\otimes2}|+|\nu^{\otimes2}-\gamma_2^{\otimes2}|)=\gamma^{\otimes2}\left(L(d_X^q,d_Y^q)- 2\lambda\right) +\underbrace{\lambda\left(|\mu|^2+|\nu|^2\right)}_{\text{does not depend on $\gamma$}}.\nonumber
\end{align}
\end{remark}

\begin{proposition}\label{pro: partial GW version2}
Given mm-spaces $\mathbb{X}=(X,d_X,\mu), \mathbb{Y}=(Y,d_Y,\nu)$, the minimization problem \eqref{eq:pgw_ori} can be restricted to the set
$\Gamma_{\leq}(\mu,\nu)=\{\gamma\in \mathcal{M}_+(X\times Y): \,  \gamma_1\leq \mu,\gamma_2\leq \nu\}$. That is,
\begin{equation}
PGW_{\lambda,q}^L(\mathbb{X},\mathbb{Y})=\inf_{\gamma \in \Gamma_{\leq}(\mu,\nu)} \gamma^{\otimes2}\left(L(d_X^q,d_Y^q)-2\lambda\right)+\lambda(|\mu|^2+|\nu|^2).\label{eq:pgw} 
\end{equation}
\end{proposition}
For the proof inspired by \cite{piccoli2014generalized}, we direct the reader to Appendix \ref{sec: pro: partial GW version2}. 

We notice that a similar Partial Gromov-Wasserstein problem (and its solver) has been studied \cite{chapel2020partial}. Indeed, in \cite{chapel2020partial}, the $\lambda$-penalization in the optimization problem \eqref{eq:pgw} is avoided, but the constraint set is replaced by the subset of all $\gamma\in\Gamma_{\leq}(\mu,\nu)$ such that $|\gamma|= \rho$ for a fixed $\rho\in[0,\min\{|\mu|,|\nu|\}]$. We will call this formulation the \textit{Mass-Constrained Partial Gromov-Wasserstein} (MPGW) problem. In Appendix \ref{sec: related work chapel}, we explore the relations between PGW and MPGW, and in Section \ref{sec: experiments} and Appendices \ref{sec:shape_retrieval_2}, \ref{sec:wall_time}, 
we analyze the performance of the different solvers through different experiments.

\begin{proposition}\label{pro: pgw minimizer}
If $L(r_1,r_2)=|r_1-r_2|^p$, for $p\in[1,\infty)$, 
we use $PGW^p_{\lambda,q}$ to denote $PGW^L_{\lambda,q}$. In this case, 
\eqref{eq:pgw_ori} and \eqref{eq:pgw} admit a minimizer.  
\end{proposition}
The proof is given in Appendix \ref{sec: pro: pgw minimizer}: Its idea extends results from \cite{memoli2011gromov} from probability mm-spaces to arbitrary mm-spaces.

Next, we state one of our main results: The PGW problem gives rise to a metric between mm-spaces. The rigorous statement, as well as its proof, is given in Appendix \ref{sec: metric property}: The formal statement is based on the definition of \textit{equivalence classes} among mm-spaces (see Remark \ref{remark: equivalence class}). The most difficult part of the proof is the triangle inequality, and the main technique used relies on the relation between PGW and GW (see Appendix \ref{app: main technique of main proof}).

\begin{proposition}\label{pro: pgw is metric}
Let $\lambda>0$, $1\leq q,p<\infty$ and $L(r_1,r_2)=|r_1-r_2|^p$. Then $(PGW^p_{\lambda,q}(\cdot,\cdot))^{1/p}$ defines a  metric between mm-spaces. 
\end{proposition}

Finally, for consistency, we provide the following result when the penalization tends to infinity. Its proof is given in Appendix \ref{sec: limit lambda}.

\begin{proposition}\label{prop: limit for lambda}
Consider probability mm-spaces $\mathbb{X}=(X,d_X,\mu)$, $\mathbb{Y}=(Y,d_Y,\nu)$, that is,  $|\mu|=|\nu|=1$. Assume that $L$ is a continuous function.
Then 
$\lim_{\lambda\to \infty}PGW_{\lambda,q}^L(\mathbb{X},\mathbb{Y})=GW^L(\mathbb{X},\mathbb{Y}).$
\end{proposition}

\section{Computation of the Partial GW Distance}\label{sec: computation}
In the discrete setting, consider mm-spaces $\mathbb{X}=(X,d_X,\sum_{i=1}^np_{i}^X\delta_{x_i})$, $\mathbb{Y}=(Y,d_Y,\sum_{j=1}^mq_j^Y\delta_{y_j})$, where  $X=\{x_1,\dots,x_n\}$, $Y=\{y_1,\dots, y_m\}$, the weights $p_i^X$, $q_j^Y$ are non-negative numbers, and the distances $d_X$, $d_Y$ are determined by the matrices $C^X\in\mathbb{R}^{n\times n}$, $C^{Y}\in\mathbb{R}^{m\times m}$ defined by 
\begin{equation}\label{eq: C^X}
  C^{X}_{i,i'}:=d_X^q(x_i,x_{i'}) \quad  \forall i,i'\in[1:n]\quad \text{and} \quad C^{Y}_{j,j'}:=d_Y^q(y_j,y_{j'})\quad  \forall j,j'\in[1:m].  
\end{equation}
Let $\mathrm{p}:=[q_1^X,\dots,q_n^X]^\top$ and $\mathrm{q}:=[q_1^Y,\dots,q_m^Y]^\top$ denote the weight vectors corresponding to the given discrete measures.
We view the sets of transportation plans $\Gamma(\mathrm{p},\mathrm{q})$ and $\Gamma_{\leq}(\mathrm{p},\mathrm{q})$ for the GW and PGW problems, respectively, as the subsets of $n\times m$ matrices
\begin{equation}\label{eq: discrete plans}
  \Gamma(\mathrm{p},\mathrm{q}):=\{\gamma\in\mathbb{R}_+^{n\times m}:\,  \gamma 1_m= \mathrm{p}, \gamma^\top1_n= \mathrm{q} \},  \quad \text{if }  |\mathrm{p}|=\sum_{i=1}^n{p_i^X}=1=\sum_{j=1}^m{q_j^Y}=|\mathrm{q}|; 
\end{equation}
\begin{equation}\label{eq: partial discrete plans}
  \Gamma_\leq(\mathrm{p},\mathrm{q}):=\{\gamma\in\mathbb{R}_+^{n\times m}:\,  \gamma 1_m\leq \mathrm{p}, \gamma^\top1_n\leq \mathrm{q} \},   
 \end{equation}
for any pair of non-negative vectors $\mathrm{p}\in \mathbb{R}^n_+$, $\mathrm{q}\in\mathbb{R}^m_+$,
where $1_n$ is the vector with all ones in $\mathbb{R}^n$ (resp. $1_m$), and  $\gamma 1_m\leq \mathrm{p}$ means that component-wise the $\leq$ relation holds.

Given by a non-negative function $L:\mathbb{R}^{n\times n}\times\mathbb{R}^{m\times m}\to\mathbb{R}_+$, the transportation cost $M$ and the `partial' transportation cost $\tilde{M}$ are represented by the $n\times m\times n\times m$ tensors:  
\begin{equation}\label{eq: tensor M}
M_{i,j,i',j'}=L(C_{i,i'}^{X},C_{j,j'}^{Y}) \qquad \text{ and } \qquad \tilde{M}:=M-2\lambda:=M-2\lambda 1_{n,m,n,m},
\end{equation}
where $1_{n,m,n,m}$ is the tensor with ones in all its entries.
For each $n\times m\times n\times m$ tensor $M$ and each $n\times m$ matrix $\gamma$, we define tensor-matrix multiplication
$M\circ \gamma\in \mathbb{R}^{n\times m}$
 by
 $$(M\circ \gamma)_{ij}=\sum_{i',j'}(M_{i,j,i',j'})\gamma_{i',j'}.$$
Then, the Partial GW problem in  \eqref{eq:pgw} can be written as 
\begin{equation}
PGW_{\lambda}^L(\mathbb{X},\mathbb{Y})=\min_{\gamma\in\Gamma_\leq(\mathrm{p},\mathrm{q})}\mathcal{L}_{\tilde{M}}(\gamma)+\lambda(|\mathrm{p}|^2+|\mathrm{q}|^2), \quad \text{ where}\label{eq:pgw_discrete}
\end{equation}    
\begin{equation}
\mathcal{L}_{\tilde{M}}(\gamma):=\tilde M \gamma^{\otimes2}:=\sum_{i,j,i',j'}\tilde{M}_{i,j,i',j'}\gamma_{i,j}\gamma_{i',j'}=\sum_{ij}(\tilde{M}\circ\gamma)_{ij}\gamma_{ij}  =:\langle \tilde{M}\circ\gamma,\gamma\rangle_F, \label{eq:L_tilde}
\end{equation}
and $\langle\cdot,\cdot\rangle_F$ stands for the Frobenius dot product. The constant term $\lambda(|\mathrm{p}|^2+|\mathrm{q}|^2)$ will be ignored in the rest of this paper since it does not depend on $\gamma$. 
\subsection{Frank-Wolfe for the PGW Problem -- Solver 1}\label{subsec: alg 1}
In this section, we discuss the Frank-Wolfe (FW) algorithm for the PGW problem \eqref{eq:pgw_discrete}. A second variant of the FW solver is provided in the Appendix \ref{sec:solver2}. 

As a summary, in our proposed method, we address the discrete PGW problem \eqref{eq:pgw_discrete}, highlighting that the \textit{direction-finding subproblem} in the Frank-Wolfe (FW) algorithm is a POT problem for \eqref{eq:pgw_discrete}. 
Specifically, \eqref{eq:pgw_discrete} is treated as a discrete POT problem in our Solver 1, where we apply Proposition \ref{pro: opt and ot} to solve a discrete OT problem. 

For each iteration $k$, the procedure is summarized in three steps detailed below.

The convergence analysis, detailed in Appendix \ref{sec: convergence}, applies the results from \cite{lacoste2016convergence} to our context, showing that the FW algorithm achieves a stationary point at a rate of $\mathcal{O}(1/\sqrt{k})$ for non-convex objectives with a Lipschitz continuous gradient in a convex and compact domain.

\subsection*{Step 1. Computation of gradient and optimal direction.}

It is straightforward to verify that 
the gradient of the objective function \eqref{eq:L_tilde} in \eqref{eq:pgw_discrete} is given by  
\begin{align}
\nabla\mathcal{L}_{\tilde{M}}(\gamma)&=2\tilde{M}\circ \gamma \label{eq: gradient of L}.
\end{align} 
The classical method to compute $M\circ \gamma$ is the following: 
First, convert $M$ into an $(n\times m)\times (n\times m)$ matrix, denoted as $v(M)$, and convert $\gamma$ into an $(n\times m)\times 1$ vector $v(\gamma)$. 
Then, the computation of $M\circ \gamma$ is equivalent to the matrix multiplication $v(M)v(\gamma)$.  
The computational cost and the required storage space are  $\mathcal{O}(n^2m^2)$. In certain conditions, the above computation can be reduced to $\mathcal{O}(n^2+m^2)$. We refer to Appendices \ref{sec: tensor prod} and \ref{sec:gradient} for details. 

Next, we aim to solve the following problem: 
\begin{equation*}    
\gamma^{(k)'}\gets\arg\min_{\gamma\in \Gamma_\leq(\mathrm{p},\mathrm{q})}\langle  \nabla \mathcal{L}_{\tilde M}(\gamma^{(k)}),\gamma\rangle_F,\label{eq: tilde gamma k version 1}  
\end{equation*}
which is a discrete POT problem since it is equivalent to
$$\min_{\gamma\in\Gamma_\leq(\mathrm{p},\mathrm{q})}\langle 2M \circ \gamma^{(k)},\gamma\rangle_F+\lambda |\gamma^{(k)}|(|\mathrm{p}|+|\mathrm{q}|-2|\gamma|).$$
The solver can be obtained by firstly converting the POT problem into an OT problem via Proposition \ref{pro: opt and ot} and then solving the proposed OT problem. 

\subsection*{Step 2: Line search method.} 

In this step, at the $k$-th iteration, we need to determine the optimal step size:
$$\alpha^{(k)}=\arg\min_{\alpha\in[0,1]}\{\mathcal{L}_{\tilde{M}}((1-\alpha)\gamma^{(k)}+\alpha \gamma^{(k)'})\}.$$
The optimal  $\alpha^{(k)}$ takes the following values (see Appendix \ref{sec:line_search_v1} for details): 
\begin{equation}
\text{Let } \alpha^{(k)} =\begin{cases}
0 &\text{if }a\leq 0, a+b>0,\\
1 &\text{if }a\leq 0,a+b\leq 0,\\
\operatorname{clip}(\frac{-b}{2a},[0,1]) &\text{if }a>0,
\end{cases}
\text{ where }
\begin{cases}
&\delta \gamma^{(k)}=\gamma^{(k)'}-\gamma^{(k)}, \\
&a=\langle \tilde{M}\circ \delta\gamma^{(k)},\delta\gamma^{(k)} \rangle_F \\
&b=2\langle \tilde{M}\circ  \gamma^{(k)},\delta\gamma^{(k)}\rangle_F.    
\end{cases}, \label{eq:a,b,v1}
\end{equation}
and
$\operatorname{clip}(\frac{-b}{2a},[0,1])=\min\{\max \{-\frac{b}{2a},0\},1 \}$.

\subsection*{Step 3: Update $\gamma^{(k+1)}\gets (1-\alpha^{(k)})\gamma^{(k)}+\alpha^{(k)}\gamma^{(k)'}$.} 

\begin{algorithm}[bt]
   \caption{Frank-Wolfe Algorithm for PGW, ver 1}
   \label{alg:pgw_v1}
\begin{algorithmic}
   \STATE {\bfseries Input:}
   $\mu=\sum_{i=1}^np_i^X\delta_{x_i}, \nu=\sum_{j=1}^mq_j^Y\delta_{y_j},\gamma^{(1)}$
   \STATE {\bfseries Output:}
   $\gamma^{(final)}$
   \STATE Compute $C^X,C^Y$\\ 
   \FOR{$k=1,2,\ldots$}
   \STATE 
$G^{(k)}\gets 2\tilde{M}\circ \gamma^{(k)}$ // Compute gradient 
   \STATE $\gamma^{(k)'}\gets \arg\min_{\gamma\in \Gamma_\leq(\mathrm{p},\mathrm{q})}\langle G^{(k)}, \gamma\rangle_F$ // Solve the POT problem.  \\ 
   \STATE Compute $\alpha^{(k)}\in[0,1]$ via  \eqref{eq:a,b,v1} //  Line Search
   \STATE $\gamma^{(k+1)}\gets (1-\alpha^{(k)})\gamma^{(k)}+\alpha^{(k)} \gamma^{(k)'}$//  Update $\gamma$  
   \STATE if convergence, break
   \ENDFOR
\STATE $\gamma^{(final)}\gets \gamma^{(k)}$
\end{algorithmic}
\end{algorithm}
\subsection{Numerical Implementation Details}
\paragraph{The initial guess, $\gamma^{(1)}$.} In the GW problem, if there is no prior knowledge, the initial guess is set to $\gamma^{(1)} = \mathrm{p}\,\mathrm{q}^\top$.

In PGW, however, as $\mu,\nu$ may not necessarily be probability measures (i.e., $\sum_i p_i^X, \sum_j q_j^Y \neq 1$ in general), we set
$\gamma^{(1)}=\frac{\mathrm{p}\mathrm{q}^\top}{\max(|\mathrm{p}|,|\mathrm{q}|)}.$
It is straightforward to verify that $\gamma^{(1)}\in \Gamma_\leq(\mathrm{p},\mathrm{q})$ as
\begin{align}
 \gamma^{(1)}1_m=\frac{|\mathrm{q}|\mathrm{p}}{\max(|\mathrm{p}|,|\mathrm{q}|)}\leq \mathrm{p},~~ 
 \gamma^{(1)\top}1_n=\frac{|\mathrm{p}|\mathrm{q}}{\max(|\mathrm{p}|,|\mathrm{q}|)}\leq \mathrm{q}.\nonumber
\end{align}
\textbf{Column/Row-Reduction.}
According to the interpretation of the penalty weight parameter in the Partial OT problem (e.g. see Lemma 3.2 in \cite{bai2022sliced}), during the POT solving step, 
for each $i\in[1:n]$ (or $j\in[1:m]$), if the $i^{th}$ row ($j^{th}$ column) of $\tilde M\circ \gamma^{(k)}$ contains a non-negative entry, all the mass of $p_i^X$ ($q_j^Y$) will be destroyed (created). Thus, we can remove the corresponding row (column) to improve the computational efficiency.

\section{Experiments}\label{sec: experiments}


\subsection{Toy Example: Shape Matching with Outliers}\label{sec:shape_mathcing}
We use the \href{https://scikit-learn.org/stable/modules/generated/sklearn.datasets.make_moons.html}{moon} dataset and synthetic 2D/3D spherical data in this experiment. 
Let $\{x_i\}_{i=1}^n,\{y_j\}_{j=1}^n$ denote the source and target point clouds. In addition, we add $\eta n$ (where $\eta=20\%$) outliers to the target point cloud. See Figure \ref{fig:2D_shape} for visualization. 

We visualize the transportation plans given by the GW \cite{memoli2011gromov}, MPGW \cite{chapel2020partial}, UGW \cite{sejourne2021unbalanced}, and our proposed PGW problems. For MPGW, UGW, and PGW, we set the mass to be 1 for each point in the source and target point clouds. For GW, we normalize the mass of these points so that the source and target have the same total mass. From Figure \ref{fig:2D_shape}, we observe that PGW and MPGW induce a one-by-one relation in both cases, and no outlier points are matched to the source point cloud. Meanwhile, GW matches all of the outliers. For UGW, as it applies the Sinkhorn algorithm, we observe mass-splitting transportation plans in both cases. Moreover, we observe that some mass from the outliers has been matched, which is not desired.

\begin{figure}[h!]
\centering
\includegraphics[width=1.0\textwidth]{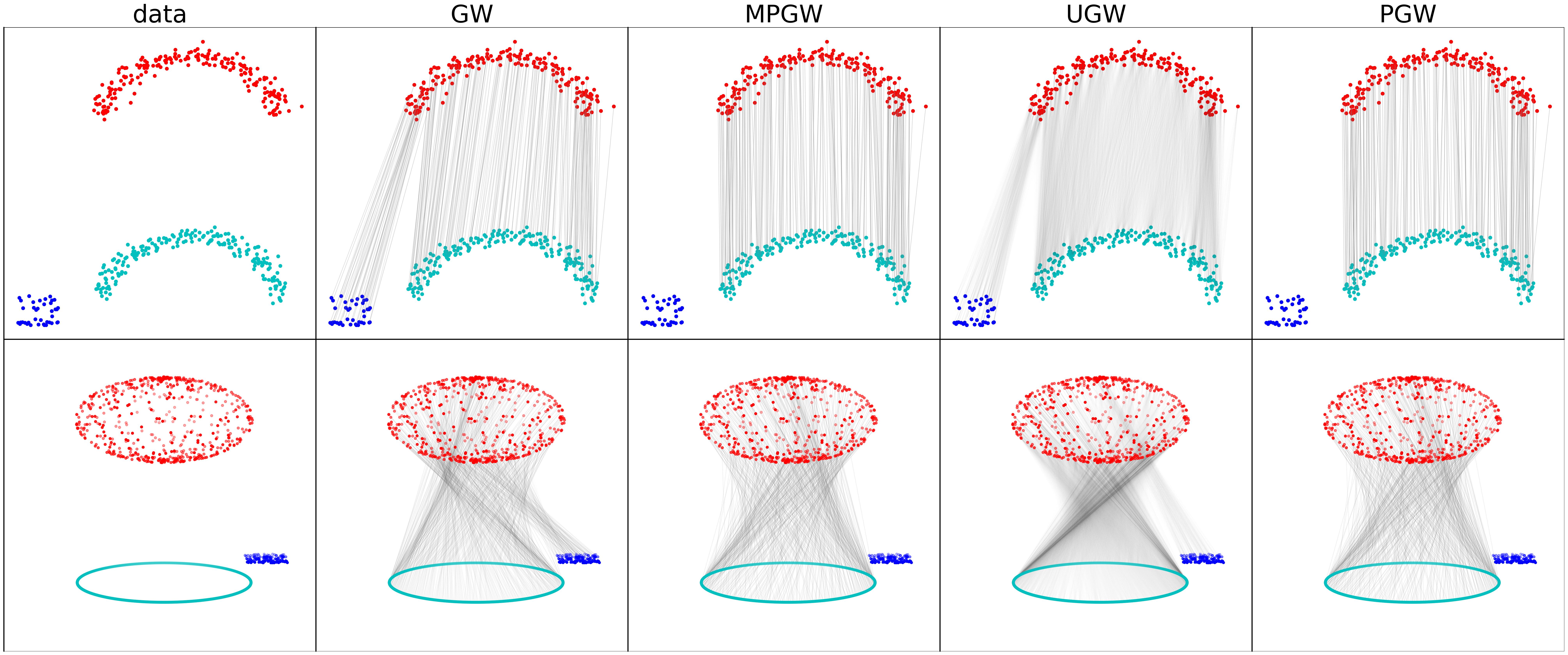}
\caption{The set of red points comprises the source point cloud. The union of the dark blue (outliers) and light blue points comprises the target point cloud. For UGW, MPGW, and PGW, we set the mass for each point to be the same. For GW, we normalize the mass for the balanced mass constraint setting.}
\label{fig:2D_shape}
\end{figure}

\begin{figure}[t]
    \begin{subfigure}[t]{0.325\textwidth}
    \includegraphics[width=\columnwidth]{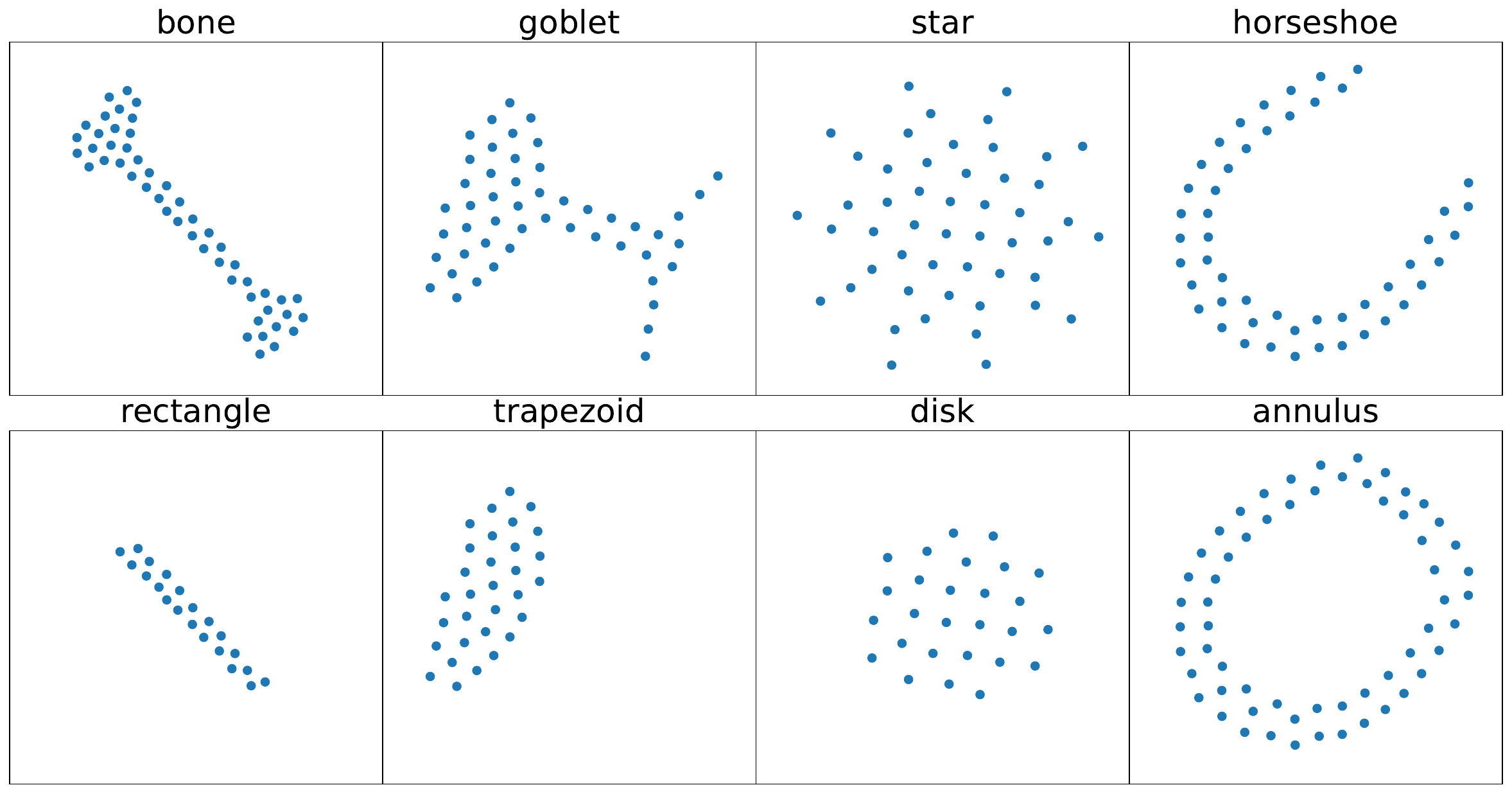}
    \end{subfigure}
    \begin{subfigure}[t]{0.675\textwidth}
        \centering
    \includegraphics[width=\columnwidth]{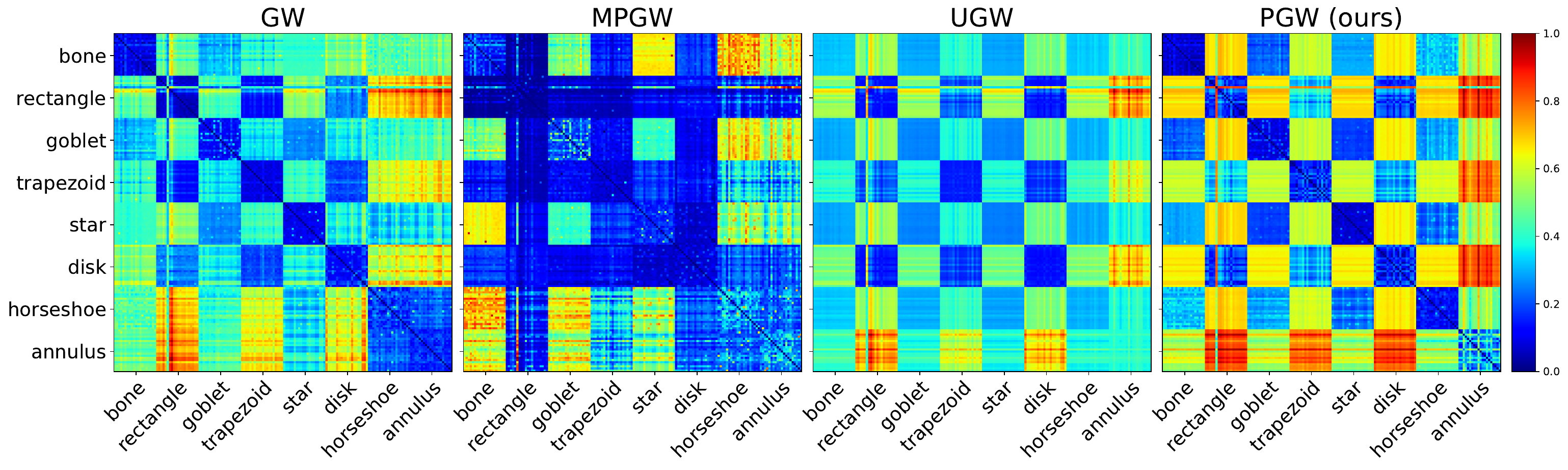}
    \end{subfigure}
    
    \begin{subfigure}[t]{0.325\textwidth}
    \includegraphics[width=\columnwidth]{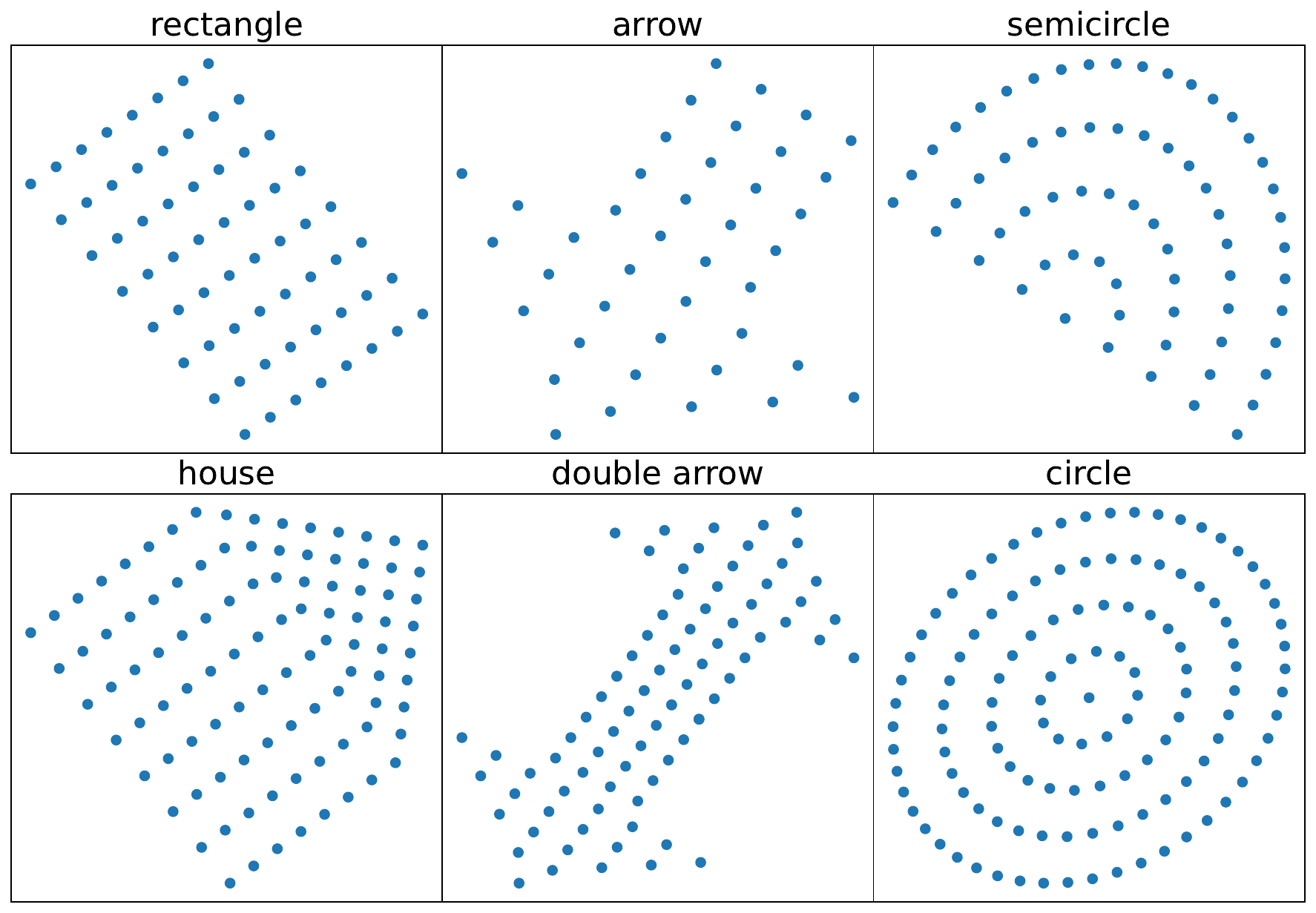}
    \end{subfigure}
    \begin{subfigure}[t]{0.675\textwidth}
        \centering
        \includegraphics[width=\columnwidth]{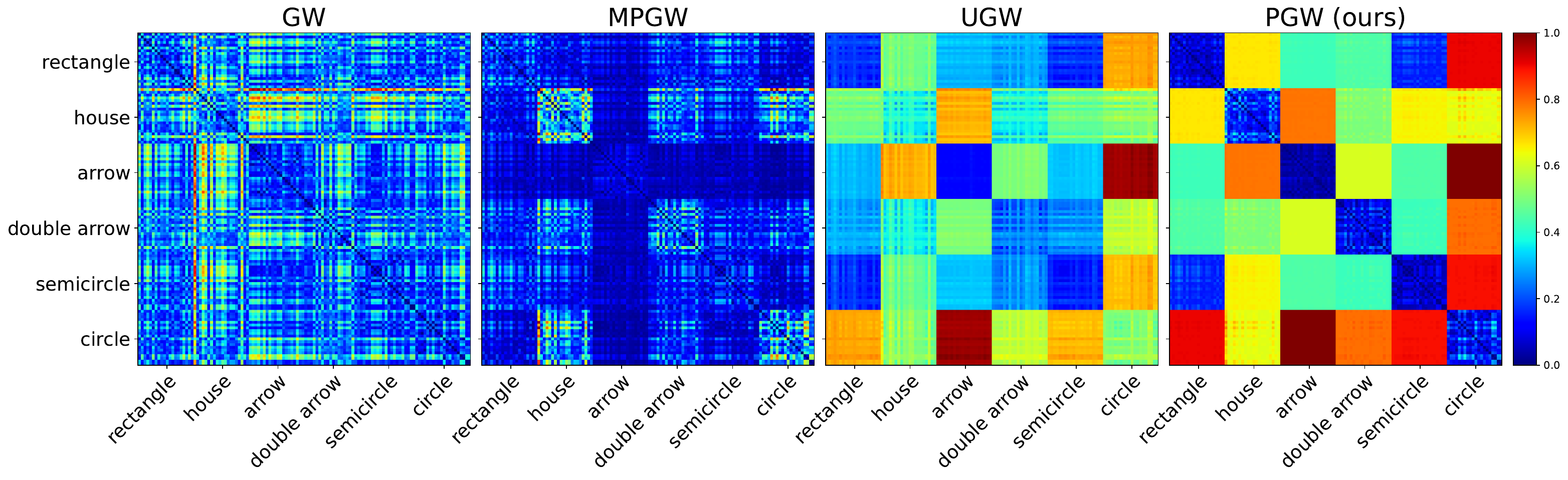}
    \end{subfigure}
    \caption{In each row, the first figure visualizes an example shape from each class, and the second figure visualizes the resulting pairwise distance matrices. The first row corresponds to Dataset I, and the second corresponds to Dataset II.}
    \label{fig:shape_retrieval}
\end{figure} 

\subsection{Shape Retrieval}

\textbf{Experiment setup.} We now employ the PGW distance to distinguish between 2D shapes, as done in \cite{beier2022linear}, and use GW, MPGW, and UGW as baselines for comparison. Given a series of 2D shapes, we represent the shapes as mm-spaces $\mathbb{X}^i= (\mathbb{R}^2, \| \cdot \|_2, \mu^i)$, where $\mu^i = \sum_{k=1}^{n^i} \alpha^i \delta_{x^i_k}$. For the GW method, we normalize the mass for the balanced mass constraint setting (i.e., $\alpha^i = \frac{1}{n^i}$), and for the remaining methods, we let $\alpha^i =\alpha$ for all the shapes, where $\alpha>0$ is a fixed constant. In this manner, we compute the pairwise distances between the shapes.

We then use the computed distances for nearest neighbor classification. We do this by choosing a representative at random from each class in the dataset and then classifying each shape according to its nearest representative. This is repeated over 10,000 iterations, and we generate a confusion matrix for each distance used. Finally, using the approach given by \cite{beier2022linear, vay2019fgw}, we combine each distance with a support vector machine (SVM), applying stratified 10-fold cross-validation. In each iteration of cross-validation, we train an SVM using $\exp(-\sigma D)$ as the kernel, where $D$ is the matrix of pairwise distances (w.r.t. one of the considered distances) restricted to 9 folds, and compute the accuracy of the model on the remaining fold. We report the accuracy averaged over all 10 folds for each model.

\textbf{Dataset setup.} We test two datasets in this experiment, which we refer to as Dataset I and Dataset II. We construct Dataset I by adapting the 2D shape dataset given in \cite{beier2022linear}, consisting of 20 shapes in each of the classes: bone, goblet, star, and horseshoe. For each class, we augment the dataset with an additional class by selecting either a subset of points from each shape of that class (rectangle/bone, trapezoid/goblet, disk/star) or adding additional points to each shape of that class (annulus/horseshoe). Hence, the final dataset consists of 160 shapes across 8 total classes. This dataset is visualized in Figure \ref{fig:sr_data_1}.

For Dataset II, we generate 20 shapes for each of the classes: rectangle, house, arrow, double arrow, semicircle, and circle. These shapes were generated in pairs, such that each shape of class rectangle is a subset of the corresponding shape of class house, and similarly for arrow/double arrow and semicircle/circle. This dataset is visualized in Figure \ref{fig:sr_data_2}.

\textbf{Performance analysis}. 
We refer to Appendix \ref{sec:shape_retrieval_2} for full numerical details, parameter settings, and the visualization of the resulting confusion matrices. We visualize the two considered datasets and the resulting pairwise distance matrices in Figure \ref{fig:shape_retrieval}. For the SVM experiments, GW achieves the highest accuracy on Dataset I, 98.13\%, while the second best method is PGW, 96.25\%. For Dataset II, PGW achieves the highest accuracy, correctly classifying 100\% of the samples. The complete set of accuracies for all considered distances on each dataset is reported in Table \ref{tb:svm_accs}.

In addition, we report the wall-clock time required to compute all pairwise distances for each distance in Table \ref{tb:svm_time}. We observe that GW, MPGW, and PGW have similar wall-clock times across both experiments (30-50 seconds for Dataset I, 80-140 seconds for Dataset II), with PGW admitting a slightly faster runtime in both cases. Meanwhile, UGW requires almost 1500 seconds on the experiment with Dataset I and over 500 seconds on the experiment with Dataset II.

\begin{table}[h!]
    \centering
    \begin{subtable}[t]{0.4\textwidth}
        \centering
        \begin{tabular}{lcc}
            \toprule
            Distance & Dataset I & Dataset II \\
            \midrule
            GW   & \textbf{0.9813} & 0.8083 \\
            MPGW & 0.2375 & 0.2500 \\
            UGW  & 0.89375 & 0.9000 \\
            PGW (ours)  & 0.9625 & \textbf{1.0000} \\
            \bottomrule
        \end{tabular}
        \caption{Mean accuracy of SVM using each distance in the kernel.}
        \label{tb:svm_accs}
    \end{subtable}
    ~\hspace{.45in}~
    \begin{subtable}[t]{0.4\textwidth}
        \begin{tabular}{lcc}
        \toprule
        Distance & Dataset I & Dataset II \\
        \midrule
        GW   & 49.02s & 137.12s \\
        MPGW & 49.10s & 93.90s \\
        UGW  & 1484.49s & 519.91s \\
        PGW (ours)  & \textbf{35.92s} & \textbf{79.27s}  \\
        \bottomrule
        \end{tabular}
        \caption{Wall-clock time comparison.}
        \label{tb:svm_time}
    \end{subtable}
\caption{Accuracy and wall-clock time comparison for shape retrieval experiment.}
\end{table}

\subsection{Partial Gromov-Wasserstein Barycenter and Shape Interpolation}

By \cite{peyre2016gromov}, Gromov-Wasserstein can be applied to interpolate two shapes via the concept of \textit{Gromov-Wasserstein Barycenters}. In this paper, we introduce \textit{Partial Gromov-Wasserstein Barycenters} by extending the GW Barycenter to the setting of PGW as follows.

Consider the discrete mm-spaces $\mathbb{X}^1,\ldots, \mathbb{X}^K$, where $\mathbb{X}^k=(X^k,\|\cdot\|_{\mathbb{R}^{d_k}},\sum_{i=1}^{n_k}p_i^k\delta_{x^k_i})$, with $X^k=\{x_i^k\}_{i=1}^{n_k}\subset \mathbb{R}^{d_k}$. We denote $C^k=[\|x^k_i-x^{k}_{i'}\|^2]_{i,i'}$ and $\mathrm{p}^k=[p_1^k,\dots,p^k_{n_k}]$. Given positive constants $\lambda_1,\ldots, \lambda_K>0$, the PGW Barycenter is defined by: 
\begin{align}
\min_{C,\gamma_k}\sum_k\xi_k\langle M(C,C^k)\circ \gamma^k,\gamma^k \rangle-2\lambda_k |\gamma^k|^2\label{eq:PGW_bpc_discrete0}
\end{align}
where each $\gamma^k\in \Gamma_\leq(\mathrm p,\mathrm p^k)$.
We refer to Appendix \ref{sec:barycenter} for the solver of \eqref{eq:PGW_bpc_discrete0} and details. 

\textbf{Experiment setup.} We apply the PGW barycenter to the following problem: 
Given two shapes $X=\{x_i\}_{i=1}^n\subset\mathbb{R}^{d_1}$ and  $Y=\{y_i\}_{i=1}^m\subset \mathbb{R}^{d_2}$,  modeled as mm-spaces $\mathbb{X}=(X,\|\cdot\|_{\mathbb{R}^{d_1}},\sum_{i=1}^n\delta_{x_i})$ and $\mathbb{Y}=(Y,\|\cdot\|_{\mathbb{R}^{d_2}},\sum_{i=1}^m\delta_{y_i})$, we wish to find interpolations between them. In addition, we assume $\mathbb{Y}$ is corrupted by noise, i.e., $\mathbb{Y}$ is redefined as 
$\mathbb{Y}=(\tilde Y,\|\cdot\|_{\mathbb{R}^{d_2}},\sum_{i=1}^m\delta_{y_i} +  \sum_{i=1}^{m\eta}\delta_{\tilde{y}_{i}})$ with $\tilde Y= Y\cup\{\tilde y_i\}_{i=1}^m$, 
where $\eta\in [0, 1]$ is the noise level and each $\tilde{y}_i$ is randomly selected from a particular region $\mathcal{R}\subset \mathbb{R}^{d_2}$.

\textbf{Dataset setup.} We adapt the dataset given in \cite{peyre2016gromov}. See Appendix \ref{sec:shape_interpolation_2} for further details on the dataset. In this experiment, we test $\eta=5\%, 10\%$. We visualize the barycenter interpolation from $t=0/7$ to $t=7/7$, where $(1-t),t$ are the weight of the source $\mathbb{X}$ and the target $\mathbb{Y}$, respectively, in the barycenter \eqref{eq:PGW_bpc_discrete0}. The visualization given in Figure \ref{fig:barycenter_sub} is obtained by applying SMACOF MDS (multidimensional scaling) of the minimizer $C$.

\textbf{Performance analysis}. From Figure \ref{fig:barycenter_sub}, we observe that in these two scenarios, the interpolation derived from GW is clearly disturbed by the noise data points. For example, in rows $1,3$, columns $t=1/7, 2/7, 3/7$, we see that the point clouds reconstructed by MDS have significantly different width-height ratios from those of the source and target point clouds. In contrast, PGW is significantly less disturbed, and the interpolation is more natural. The width-height ratio of the point clouds generated by the PGW barycenter is consistent with that of the source/target point clouds.

\begin{figure}[h!]
    \centering
    \includegraphics[width=\textwidth]{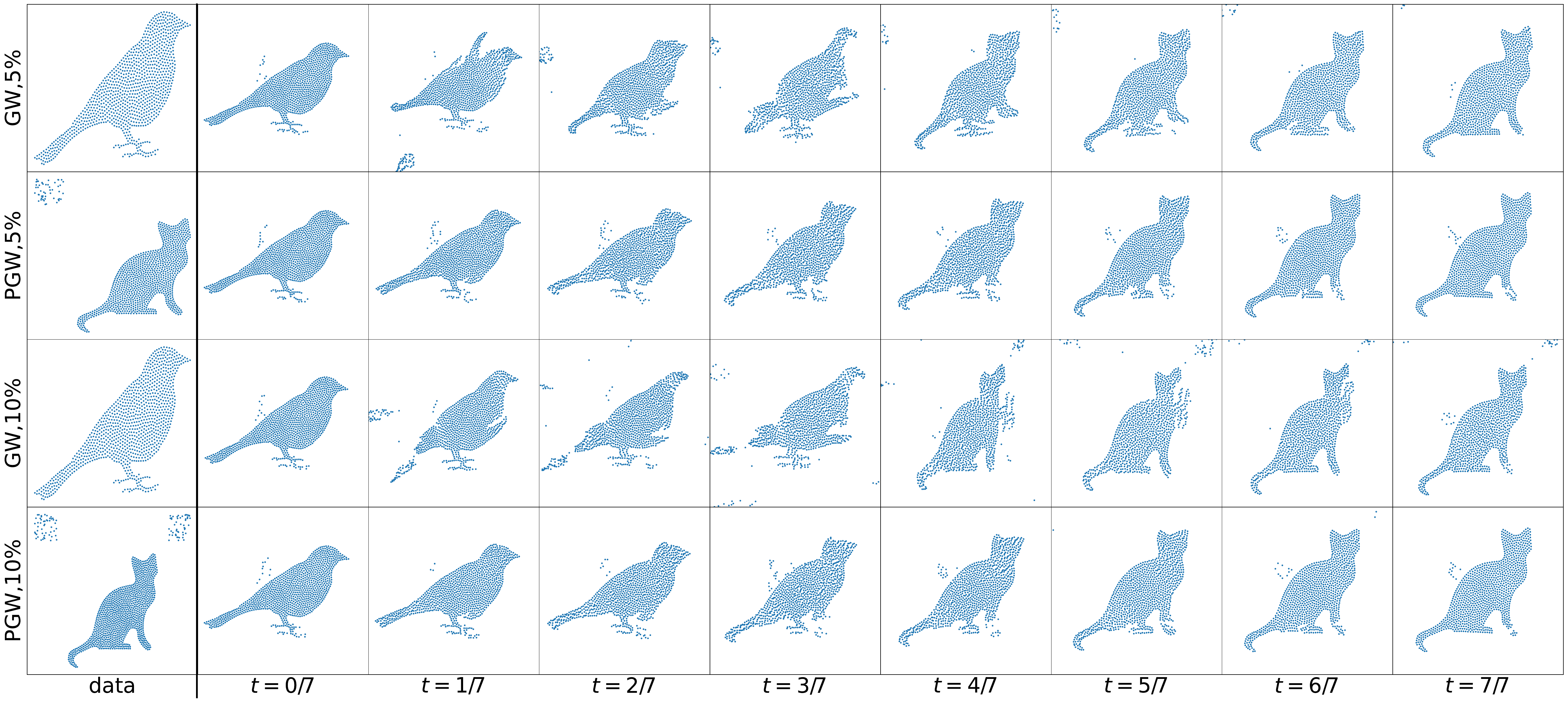}
    \caption{In the first column, the first and second figures are the source and target point clouds in the first experiment ($\eta=5\%$); the third and fourth figures are the source and target point clouds in the second experiment ($\eta=10\%$).}
    \label{fig:barycenter_sub}
\end{figure}

\section{Summary}

In this paper, we propose the Partial Gromov-Wasserstein (PGW) problem and introduce two Frank-Wolfe solvers for it. As a byproduct, we provide pertinent theoretical results, including the relation between PGW and GW, the metric property of PGW, and the PGW barycenter formulation. Furthermore, we demonstrate the efficacy of the PGW solver in solving shape matching, shape retrieval, and shape interpolation tasks. For the shape retrieval experiment, we observe that due to their metric property, PGW and GW have similar accuracy and outperform the other methods evaluated. In the shape matching and shape interpolation experiments, we demonstrate that PGW admits a more robust result when the data are corrupted by outliers/noisy data.

\section*{Acknowledgments}
This research was partially supported by the NSF CAREER Award No. 2339898.

\bibliography{refs}
\bibliographystyle{unsrt}

\newpage 
\appendix

\section{Notation and Abbreviations}
\begin{itemize}
\item OT: Optimal Transport.
\item POT: Partial Optimal Transport.
\item GW: Gromov-Wasserstein.
\item PGW: Partial Gromov-Wasserstein.
\item FW: Frank-Wolfe.
\item MPGW: Mass-Constrained Partial Gromov-Wasserstein.
\item $\|\cdot\|$: Euclidean norm.
\item $X^2=X\times X$.
\item $\mathcal{M}_+(X)$: set of all positive (non-negative) Randon (finite) measures defined on $X$. 
\item $\mathcal{P}_2(X)$: set of all probability measures defined on $X$, whose second moment is finite. 
\item $\mathbb{R}_+$: set of all non-negative real numbers.
\item $\mathbb{R}^{n\times m}$: set of all $n\times m$ matrices with real coefficients.
\item $\mathbb{R}^{n\times m}_+$ (resp. $\mathbb{R}^{n}_+$): set of all $n\times m$ matrices (resp., $n$-vectors) with non-negative coefficients.
\item $\mathbb{R}^{n\times m\times n\times m}$: set of all $n\times m\times n\times m$ tensors with real coefficients.
\item $1_n, 1_{n\times m}, 1_{n\times m\times n\times m}$: vector, matrix, and tensor of all ones.
\item $\mathbbm{1}_{E}$: characteristic function of a measurable set $E$
\begin{equation*}
    \mathbbm{1}_E (z)= \begin{cases}
        1 & \text{ if } z\in E,\\
        0 & \text{ otherwise.}
    \end{cases}
\end{equation*}
\item $\mathbb{X},\mathbb{Y}$: metric measure spaces (mm-spaces): $\mathbb{X}=(X,d_X,\mu), \, \mathbb{Y}=(Y,d_Y,\nu).$
\item $C^X$: given a discrete mm-space $\mathbb{X}=(X,d_X,\mu)$, where $X=\{x_1,\dots,x_n\}$, the symmetric matrix $C^X\in\mathbb{R}^{n \times n}$ is defined as $C^{X}_{i,i'}=d_X^q(x_i,x_i')$.
    \item $\mu^{\otimes 2}$: product measure $\mu\otimes \mu$.

    \item $T_\# \sigma$: $T:X\to Y$ is a measurable function and $\sigma$ is a measure on $X$. $T_\#\sigma$ is the push-forward measure of $\sigma$, i.e., its is the measure on $Y$ such that for all Borel set $A\subset Y$, $T_\# \sigma(A)=\sigma(T^{-1}(A))$.

    \item $\gamma,\gamma_1,\gamma_2$: $\gamma$ is a joint measure defined in a product space  having $\gamma_1,\gamma_2$ as its first and second marginals, respectively. In the discrete setting, they are viewed as matrices and vectors, i.e., $\gamma\in\mathbb{R}_+^{n\times m}$, and $\gamma_1=\gamma 1_m\in\mathbb{R}_+^n$, $\gamma_2=\gamma^\top 1_n\in\mathbb{R}_+^m$.
\item $\pi_1:X\times Y\to X$, canonical projection mapping, with $(x,y)\mapsto x$.
Similarly, $\pi_2: X\times Y\to Y$ is canonical projection mapping, with $(x,y)\mapsto y$. 

\item $\pi_{1,2}: S\times X\times Y\to X\times Y$, canonical projection mapping, with $(s,x,y)\to (x,y)$.
Similarly, $\pi_{0,1}$ maps $(s,x,y)$ to $(s,x)$;
$\pi_{0,2}$ maps $(s,x,y)$ to $(s,y)$.

\item $\Gamma(\mu,\nu)$, where $\mu\in \mathcal{P}_2(X),\nu\in \mathcal{P}_2(Y)$ (where 
$X, Y$ may not necessarily be the same set): it is the set of all the couplings (transportation plans) between $\mu$ and $\nu$, i.e., 
$\Gamma(\mu,\nu):=\{\gamma\in \mathcal{P}_2(X\times Y): \, \gamma_1=\mu,\gamma_2=\nu\}.$
\item $\Gamma(\mathrm{p},\mathrm{q})$: set of all the couplings between the discrete probability measures $\mu=\sum_{i=1}^n p_i^X\delta_{x_i}$ and $\nu=\sum_{j=1}^m q_j^Y\delta_{y_j}$
with weight vectors 
\begin{equation}\label{eq: vectors p q}
 \mathrm{p}=[p_1^X,\dots,p_n^X]^\top \qquad \text{ and  } \qquad \mathrm{q}=[q_1^Y,\dots, q_m^Y]^\top.   
\end{equation}
That is, $\Gamma(\mathrm{p},\mathrm{q})$ coincides with $\Gamma(\mu,\nu)$, but it is viewed as a subset of $n\times m$ matrices defined in \eqref{eq: discrete plans}.
\item $p,q$: real numbers $1\leq p,q<\infty$. 
\item $\mathrm{p},\mathrm{q}$:
vectors of weights as in \eqref{eq: vectors p q}.
\item $\mathrm{p}=[p_1,\dots,p_n]\leq \mathrm{p}'=[p'_1,\dots,p'_n]$ if $p_j\leq p_j'$ for all $1\leq j\leq n$.
\item $|\mathrm{p}|=\sum_{i=1}^np_i$ for $\mathrm{p}=[p_1,\dots,p_n]$. 
\item $c(x,y): X\times Y\to \mathbb{R}_+$ denotes the cost function used for classical and partial optimal transport problems. 
lower-semi continuous function. 
\item $OT(\mu,\nu)$: it is the classical optimal transport (OT) problem between the probability measures $\mu$ and $\nu$ defined in \eqref{eq: OT}.
\item $W_p(\mu,\nu)$: it is the $p$-Wasserstein distance between the probability measures $\mu$ and $\nu$ defined in \eqref{eq: Wp}, for $1\leq p<\infty$.
\item $POT(\mu,\nu;\lambda)$: the Partial Optimal Transport (OPT) problem defined in \eqref{eq: opt}. 
\item $|\mu|$: total variation norm of the positive Randon (finite) measure $\mu$ defined on a measurable space $X$, i.e., $|\mu|=\mu(X)$.
\item $\mu\leq \sigma$: denotes that for all Borel set $B\subseteq X$ we have that the measures $\mu,\sigma\in\mathcal{M}_+(X)$ satisfy $\mu(B)\leq \sigma(B)$.
\item $\Gamma_{\leq}(\mu,\nu)$, where $\mu\in \mathcal{M}_+(X),\nu\in \mathcal{M}_+(Y)$:  set of all ``partial transportation plans'' 
$$\Gamma_{\leq}(\mu,\nu):=\{\gamma\in \mathcal{M}_+(X\times Y): \, \gamma_1\leq\mu,\gamma_2\leq\nu\}.$$
\item $\Gamma_{\leq}(\mathrm{p},\mathrm{q})$: set of all the ``partial transportation plans'' between the discrete probability measures $\mu=\sum_{i=1}^n p_i^X\delta_{x_i}$ and $\nu=\sum_{j=1}^m q_j^Y\delta_{y_j}$
with weight vectors $\mathrm{p}=[p_1^X,\dots,p_n^X]$ and $\mathrm{q}=[q_1^Y,\dots, q_m^Y]$. That is, $\Gamma_{\leq}(\mathrm{p},\mathrm{q})$ coincides with $\Gamma_{\leq}(\mu,\nu)$, but it is viewed as a subset of $n\times m$ matrices defined in \eqref{eq: partial discrete plans}.
\item $\lambda>0$: positive real number.
\item $\hat{\infty}$: auxiliary point.
\item $\hat X= X\cup \{\hat{\infty}\}$.
\item $\hat\mu, \hat\nu$: given in \eqref{eq: mu_hat}.
\item $\hat{\mathrm{p}},\hat{\mathrm{q}}$: given in \eqref{eq: hat p hat q}.
\item $\hat\gamma$: given in \eqref{eq: ot and opt plan}.
\item $\hat{c}(\cdot,\cdot):\hat X\times \hat Y\to \mathbb{R}_+$: cost as in \eqref{eq: OT variant}.
\item $L:\mathbb{R}\times\mathbb{R}\to\mathbb{R}$: cost function for the GW problems.
\item $D:\mathbb{R}\times\mathbb{R}\to\mathbb{R}$: generic distance on $\mathbb{R}$ used for GW problems. 
\item $GW^L(\cdot,\cdot)$: GW optimization problem given in \eqref{eq:gw}.
\item $GW^p(\cdot,\cdot)$: GW optimization problem given in \eqref{eq:gw} when $L(a,b)=|a-b|^p$.
\item $GW^L_q(\cdot,\cdot)$: general GW optimization problem for $g\geq 1$ given in \eqref{eq: prob G^L_q}.
\item $GW^p_q(\cdot,\cdot)$: general GW optimization problem for $q\geq 1$ and $L(a,b)=|a-b|^p$ given in \eqref{eq: prob G^p_q}.
\item $GW^p_{\lambda,q}(\cdot,\cdot)$: generalized GW problem given in \eqref{eq: GW general}.
\item $\widehat{GW}$: GW-variant  problem given in \eqref{eq: GW variant} for the general case, and in \eqref{eq: variant gw discrete} for the discrete setting.
\item $\hat L$: cost given in \eqref{eq:L_tilde} for the GW-variant problem.
\item $d:\hat X\times \hat X\to\mathbb{R}_+\cup\{\infty\}$: ``generalized'' metric given in \eqref{eq: d_hat} for $\hat X$.

\item $\mathbb{X}\sim\mathbb{Y}$:  equivalence relation in for mm-spaces,  $\mathbb{X}\sim\mathbb{Y}$ if and only if 
they have the same total mass and 
$GW_q^p(\mathbb{X},\mathbb{Y})=0$.
\item $PGW_{\lambda,q}^L(\cdot,\cdot)$: partial GW optimization problem given in \eqref{eq:pgw_ori} or, equivalently, in \eqref{eq:pgw}.
\item $PGW_{\lambda,q}^p(\cdot,\cdot)$:
 partial GW optimization problem given in \eqref{eq:pgw} when 
$L(a,b)=|a-b|^p$.
\item $PGW_\lambda(\cdot,\cdot)$: is is the PGW problem $PGW_{\lambda,q}^p(\cdot,\cdot)$ for the case when $p=2=q$.

\item $\mu(\phi)$: given a measure $\mu$ and a function $\phi$, 
$$\mu(\phi):=\int \phi(x)d\mu(x).$$


\item $C(\gamma;\lambda,\mu,\nu)$: the transportation cost induced by transportation plan $\gamma\in\Gamma_\leq(\mu,\nu)$ in the Partial GW problem \ref{eq:pgw}, 
$$C(\gamma;\lambda,\mu,\nu):=\gamma^{\otimes2}(L(d_X^q,d_Y^q))+\lambda(|\mu|^2+|\nu|^2-2|\gamma|^2).$$


\item $\mathcal{L}$: functional for the optimization problem $PGW_\lambda(\cdot,\cdot)$.
\item $M$, $\tilde M$, and  $\hat M$: see \eqref{eq: tensor M}, 
and \eqref{eq: M_hat}. Notice that, $(M-2\lambda)_{i,i',j,j'}:=M_{i,i',j,j'}-2\lambda.$ 

\item $\langle\cdot, \cdot\rangle_F$: Frobenius inner product for matrices, i.e., $\langle A, B\rangle_F=\mathrm{trace}(A^\top B)=\sum_{i,j}^{n,m} A_{i,j}B_{i,j}$ for all $A,B\in\mathbb{R}^{n\times m}$. 
\item $M\circ \gamma$: product between the tensor $M$ and the matrix $\gamma$.
\item $\nabla$: gradient. 
\item $[1:n]=\{1,\dots,n\}$.
\item $\alpha$: step size based on the line search method. 
\item $\gamma^{(1)}$: initialization of the algorithm.
\item $\gamma^{(k)}$, $\gamma^{(k)'}$: previous and new transportation plans before and after step 1 in the $k-$th iteration of version 1 of our proposed FW algorithm. 
\item $\hat{\gamma}^{(k)}$, $\hat{\gamma}^{(k)'}$: previous and new transportation plans before and after step 1 in the $k-$th iteration of version 2 of our proposed FW algorithm. 
\item $G=2\tilde M\circ \gamma$, $\hat G=2\hat M\circ \hat \gamma$: Gradient of the objective function in version 1 and version 2, respectively, of our proposed FW algorithm for solving the discrete version of partial GW problem.
\item $(\delta\gamma,a,b)$ and $(\delta\hat\gamma,a,b)$: given in \eqref{eq:a,b,v1} and \eqref{eq: a,b,v2} for versions 1 and 2 of the algorithm, respectively.
\item $C^1$-function: continuous and with continuous derivatives.
\item $MPGW_\rho(\cdot,\cdot)$: Mass-Constrained Partial Gromov-Wasserstein defined in \eqref{eq:mpgw}.
\item $\Gamma^\rho_\leq(\mu,\nu)$: set transportation plans defined in \eqref{eq:mpgw plans} for the  Mass-Constrained Partial Gromov-Wasserstein problem.
\end{itemize}

    

\section{Proof of Proposition \ref{pro: partial GW version2}}\label{sec: pro: partial GW version2}

The idea of the proof is inspired by the proof of Proposition 1 in \cite{piccoli2014generalized}. 

The goal is to verify that
\begin{align}
P&GW_{\lambda,q}^L(\mathbb{X},\mathbb{Y})\notag \\&:=
\inf_{\gamma \in \mathcal{M}_+(X,Y)}\underbrace{\int_{(X\times Y)^2}L(d_{X}^q(x,x'),d_{Y}^q(y,y'))
d\gamma^{\otimes 2} }_{\text{transport GW cost}}+\underbrace{\lambda\left(|\mu^{\otimes 2}-\gamma_1^{\otimes 2}|+|\nu^{\otimes 2}-\gamma_2^{\otimes2}|\right)}_{\text{mass penalty}}\notag\\
&={\inf_{\gamma \in \Gamma_{\leq}(\mu,\nu)}\int_{(X\times Y)^2}L(d_{X}^q(x,x'),d_{Y}^q(y,y'))
d\gamma^{\otimes 2}} +{\lambda\left(|\mu^{\otimes 2}-\gamma_1^{\otimes 2}|+|\nu^{\otimes 2}-\gamma_2^{\otimes2}|\right)}. \label{eq: equality gamma leq}
\end{align}

Consider $\gamma\in\mathcal{M}_+(X\times Y)$ such that $\gamma_1 \leq \mu$ does not hold.
Then we can write the Lebesgue decomposition of $\gamma_1$ with respect to $\mu$: 
$$\gamma_1=f\mu+\mu^\perp,$$ where 
$f\ge 0$ is the Radon-Nikodym derivative of $\gamma_1$ with respect to $\mu$, and $\mu^\perp,\mu$ are mutually singular, that is, 
there exist measurable sets $A,B$ such that $A\cap B=\emptyset$, $X=A\cup B$ and $\mu^\perp(A)=0,\mu(B)=0$. Without loss of generality, we can assume that the support of $f$ lies on $A$, since
$$\gamma_1(E)=\int_{E\cap A} f(x) \, d\mu(x)+\mu^\perp(E\cap B) \qquad \forall E\subseteq X \text{ measurable}.$$
Define $A_1=\{x\in A: \, f(x)>1\}, A_2=\{x\in A: \, f(x)\leq 1\}$ (both are measurable, since $f$ is measurable), and  define $\bar{\mu}=\min\{f,1\}\mu$. Then,
$$\bar\mu\leq \mu \qquad \text{ and } \qquad \bar\mu\leq f\mu\leq f\mu+\mu^\perp= \gamma_1.$$
There exists a $\bar\gamma\in \mathcal{M}_+(X\times Y)$ such that $\bar \gamma_1=\bar\mu,\bar\gamma\leq \gamma$, and $\bar\gamma_2\leq \gamma_2$. Indeed, we can construct $\bar \gamma$ in the following way: First, let $\{\gamma^x\}_{x\in X}$ be the set of conditional measures (disintegration) such that for every measurable (test) function $\psi: X\times Y\to \mathbb{R}$ we have 
 $$\int \psi(x,y) \, d\gamma(x,y)=\int_X \int_Y\psi(x,y) \, d\gamma^x(y) \, d\gamma_1(x).$$
Then, define $\bar\gamma$ as 
$$\bar{\gamma}(U):=\int_X\int_Y \mathbbm{1}_{U}(x,y) \, d\gamma^x(y) \, d\bar\mu(x) \qquad \forall
 U\subseteq X\times Y \text{ Borel}.$$
Then, $\bar \gamma$ verifies that $\bar\gamma_1=\bar\mu$, and since $\bar\mu\leq \gamma_1$, we also have that 
$\bar \gamma\leq \gamma$, which implies  $\bar\gamma_2\leq \gamma_2$. 

Since $|\gamma_1|=|\gamma_2|$ and $|\bar\gamma_1|=|\bar\gamma_2|$, then  we have $|\gamma_1^{\otimes2}-\bar\gamma_1^{\otimes 2}|=|\gamma_2^{\otimes 2}-\bar\gamma_2^{\otimes2}|$. 

We claim that 
\begin{equation}\label{eq: claim}
    |\mu^{\otimes 2}-\gamma_1^{\otimes 2}|\ge |\mu^{\otimes 2}-\bar\gamma_1^{\otimes 2}|+|\gamma_1^{\otimes 2}-\bar\gamma_1^{\otimes 2}|.
\end{equation}

\begin{itemize}
    \item \textit{Left-hand side of \eqref{eq: claim}}:
    Since $\{A,B\}$ is a partition of $X$, we first split the left-hand side of \eqref{eq: claim} as
    \begin{align*}
  |\mu^{\otimes 2}-\gamma_1^{\otimes 2}| 
&=\underbrace{(\mu^{\otimes 2}-\gamma_1^{\otimes 2})(A\times A)}_{(I)}+\underbrace{(\mu^{\otimes 2}-\gamma_1^{\otimes 2})(A\times B)+(\mu^{\otimes 2}-\gamma_1^{\otimes 2})(B\times A)}_{(II)}\\
&\quad+\underbrace{(\mu^{\otimes2}-\gamma_1^{\otimes2})(B\times B)}_{(III)}.      
    \end{align*}
Then we have
\begin{align*}
&(III)=(\mu^{\otimes 2}-\gamma_1^{\otimes 2})(B\times B)=\mu^\perp\otimes\mu^\perp(B\times B)=|\mu^\perp|^2,\\
    &(II)=(\mu^{\otimes 2}-\gamma_1^{\otimes 2})(A\times B)+(\mu^{\otimes 2}-\gamma_1^{\otimes 2})(B\times A)=2|\mu^\perp|(\mu-\gamma_1)(A).
\end{align*}
Since $\gamma_1=f\mu$ in $A$, then $\bar\gamma_1=\gamma_1$ 
in $A_2$ and $\bar\gamma_1=\mu$ in $A_1$, so we have
\begin{align*}
    (\mu-\gamma_1)(A)&=(\mu-\gamma_1)(A_1)+(\mu-\gamma_1)(A_2)=(\gamma_1-\bar\gamma_1)(A_1)+(\mu-\bar\gamma_1)(A_2)\\
    &=(\gamma_1-\bar\gamma_1)(A)+(\mu-\bar\gamma_1)(A).
\end{align*}
Thus, 
\begin{align*}
(II)&=2|\mu^\perp|((\gamma_1-\bar\gamma_1)(A)+(\mu-\bar\gamma_1)(A)),
\end{align*}
and we also get that
\begin{align*}
(I)&=  (\mu^{\otimes 2}-\gamma_1^{\otimes 2})(A\times A)
  \\
  &=(\mu^{\otimes 2}-\gamma_1^{\otimes 2})({A_1\times A_1})+(\mu^{\otimes 2}-\gamma_1^{\otimes 2})(A_2\times A_2)+
(\mu^{\otimes 2}-\gamma_1^{\otimes 2})(A_1\times A_2)\\
&\qquad +(\mu^{\otimes 2}-\gamma_1^{\otimes 2})(A_2\times A_1)\\
  &=(\gamma_1^{\otimes 2}-\bar\gamma_1^{\otimes 2})({A_1\times A_1})+(\mu^{\otimes 2}-\bar\gamma_1^{\otimes 2})(A_2\times A_2)+\\
  &\qquad +
|\bar\gamma_1\otimes\mu-\gamma_1\otimes\bar\gamma_1|(A_1\times A_2)+|\mu\otimes\bar\gamma_1-\bar\gamma_1\otimes\gamma_1|(A_2\times A_1)\\
  &=(\gamma_1^{\otimes 2}-\bar\gamma_1^{\otimes 2})({A_1\times A_1})+(\mu^{\otimes 2}-\bar\gamma_1^{\otimes 2})(A_2\times A_2)+
2(\bar\gamma_1-\gamma_1)(A_1)(\mu-\bar\gamma_1)(A_2)\\
  &=(\gamma_1^{\otimes 2}-\bar\gamma_1^{\otimes 2})({A\times A})+(\mu^{\otimes 2}-\bar\gamma_1^{\otimes 2})(A\times A)+
\underbrace{2(\bar\gamma_1-\gamma_1)(A_1)(\mu-\bar\gamma_1)(A_2)}_{\geq 0}.
\end{align*}

\item \textit{Right-hand side of \eqref{eq: claim}:}
First notice that
\begin{equation*}
    (\gamma_1-\bar\gamma_1)(B)=(\gamma_1-\bar\gamma_1)(B)\leq\gamma_1(B)=|\mu^\perp|,
\end{equation*}
and since $\bar\gamma_1\leq \mu$ and $\mu(B)=0$, we have
\begin{equation*}
    (\mu-\bar\gamma_1)(B)=0.
\end{equation*}
Then, 
\begin{align*}
&|\mu^{\otimes 2}-\bar\gamma_1^{\otimes 2}|+|\gamma_1^{\otimes 2}-\bar\gamma_1^{\otimes 2}|=\\
&=(\mu^{\otimes 2}-\bar\gamma_1^{\otimes 2})(A\times A)+(\gamma_1^{\otimes 2}-\bar\gamma_1^{\otimes 2})(A\times A)+(\mu^{\otimes 2}-\bar\gamma_1^{\otimes 2})(B\times B)\\
&\quad+(\gamma_1^{\otimes 2}-\bar\gamma_1^{\otimes 2})(B\times B)+{(\mu^{\otimes 2}-\bar\gamma_1^{\otimes 2})(A\times B)}+(\gamma_1^{\otimes 2}-\bar\gamma_1^{\otimes 2})(A\times B)\\
&\quad+{(\mu^{\otimes 2}-\bar\gamma_1^{\otimes 2})(B\times A)}+(\gamma_1^{\otimes 2}-\bar\gamma_1^{\otimes 2})(B\times A)\\
&\leq \underbrace{(\mu^{\otimes 2}-\bar\gamma_1^{\otimes 2})(A\times A)+(\gamma_1^{\otimes 2}-\bar\gamma_1^{\otimes 2})(A\times A)}_{\leq(I)} + \underbrace{|\mu^\perp|^2}_{=(III)}+\underbrace{2|\mu^\perp|(\gamma_1-\bar\gamma_1)(A)}_{= (II)}.
\end{align*}
\end{itemize}
Thus, \eqref{eq: claim} holds.


We finish the proof of the proposition by noting that 
\begin{align}
|\mu^{\otimes 2}-\bar \gamma_1^{\otimes 2}|+|\nu^{\otimes 2}-\bar\gamma_2^{\otimes 2}|&\leq |\mu^{\otimes 2}-\gamma_1^{\otimes 2}|-|\gamma_1^{\otimes 2}-\bar \gamma_1^{\otimes 2}|+|\nu^{\otimes 2}-\bar \gamma_2^{\otimes 2}|\nonumber\\
&=|\mu^{\otimes 2}-\gamma_1^{\otimes 2}|-|\gamma_2^{\otimes 2}-\bar \gamma_2^{\otimes 2}|+|\nu^{\otimes 2}-\bar \gamma_2^{\otimes 2}|\nonumber\\
&\leq |\mu^{\otimes 2}-\gamma_1^{\otimes 2}|+|\nu^{\otimes 2}-\gamma_2^{\otimes 2}| \nonumber 
\end{align}
where the first inequality follows from \eqref{eq: claim}, and the second inequality holds from the fact the total variation norm $|\cdot|$ satisfies triangular inequality. 
Therefore, $\bar \gamma$ induces a smaller transport GW cost than $\gamma$ (since $\bar \gamma\leq \gamma$), and also $\bar\gamma$ decreases the mass penalty in comparison that corresponding to $\gamma$. 
Thus, $\bar\gamma$ is a better GW transportation plan, which satisfies $\bar\gamma_1\leq \mu$. Similarly, we can further construct $\bar\gamma'$ based on $\bar\gamma$ such  that $\bar\gamma'_1\leq \mu, \bar\gamma'_2\leq \nu$. Therefore, we can restrict the minimization in \eqref{eq:pgw_ori} from $\mathcal{M}_+(X\times Y)$ to $\Gamma_{\leq}(\mu,\nu)$.
Thus, the equality \eqref{eq: equality gamma leq} is satisfied.

\begin{proof}[Proof of Remark \ref{remark: rewrite pGW}]
   Given $\gamma\in\Gamma_\leq(\mu,\nu)$, since $\gamma_1\leq \mu$, $\gamma_2\leq \nu$, and $\gamma_1(X)=|\gamma_1|=|\gamma|=|\gamma_2|=\gamma_2(Y)$, we have
\begin{align*}
  |\mu^{\otimes 2}-\gamma_1^{\otimes 2}|+|\nu^{\otimes 2}-\gamma_2^{\otimes2}|&= \mu^{\otimes 2}(X^2)-\gamma_1^{\otimes 2}(X^2)+\nu^{\otimes 2}(Y^2)-\gamma_2^{\otimes2}(Y^2)\\
  &=|\mu|^2+|\nu|^2 -2|\gamma|^2, 
\end{align*}
and so the transportation cost in partial GW problem \eqref{eq:pgw} becomes
\begin{align}
C&(\gamma;\lambda,\mu,\nu)\notag\\
&:=\int_{(X\times Y)^2}L(d_{X}^q(x,x'),d_{Y}^q(y,y')) \,
d\gamma(x,y)d\gamma(x',y') +\lambda\left(|\mu^{\otimes 2}-\gamma_1^{\otimes 2}|+|\nu^{\otimes 2}-\gamma_2^{\otimes2}|\right)\notag \\
&=\int_{(X\times Y)^2}L(d_{X}^q(x,x'),d_{Y}^q(y,y')) \, 
d\gamma(x,y)d\gamma(x',y') +\lambda\left(|\mu|^2+|\nu|^2-2|\gamma|^2\right) \notag\\
&=\int_{(X\times Y)^2}\left(L(d_{X}^q(x,x'),d_{Y}^q(y,y')- 2\lambda\right)\, 
d\gamma(x,y)d\gamma(x',y') +\underbrace{\lambda\left(|\mu|^2+|\nu|^2\right)}_{\text{does not depend on $\gamma$}}.\label{eq:pgw_cost}
\end{align}
\end{proof}

\section{Proof of Proposition \ref{pro: pgw minimizer}}\label{sec: pro: pgw minimizer}
In this section, we discuss the minimizer of the Partial GW problem \eqref{eq:pgw_ori}. 
Trivially,  $\Gamma_\leq(\mu,\nu)\subseteq\mathcal{M}_+(X\times Y)$ and 
by using Proposition \ref{pro: partial GW version2} it is enough to show that a minimizer for problem \eqref{eq:pgw} exists. 

We refer the reader to \cite[Chapters 5 and 10]{memoli2011gromov} for similar ideas. 

\subsection{Formal Statement of Proposition \ref{pro: pgw minimizer}}
{\itshape 
Suppose $X,Y$ are compact sets, then exists compact set $[0,\beta]\subset \mathbb{R}$,  such that 
$$d(x,x'), \, d(y,y')\in [0,\beta], \qquad \forall x,x'\in X,y, \, y'\in Y$$
Let $A=[0,\beta^q]$. 
Let $L_{A^2}$ denote the restriction of $L$ on $A^2$, i.e. 
$L_{A^2}:A^2\to \mathbb{R}$ with 
$L_{A^2}(r_1,r_2)=L(r_1,r_2)$, $\forall r_1,r_2\in A$. 
Suppose $L$ satisfies the following: 
there exists $0<K<\infty$ such that for every $r_1,r_1',r_2,r_2'\in A$, 
\begin{align}
|L_{A^2}(r_1,r_2)-L_{A^2}(r_1',r_2)|\leq K|r_1-r_1'|, \, \, \, 
|L_{A^2}(r_1,r_2)-L_{A^2}(r_1,r_2')|\leq K|r_2-r_2'| \label{pf:L_cond} 
\end{align}
(i.e., $L_{A^2}$ is Lipschitz on each variable).
Then $PGW^L_\lambda(\cdot,\cdot)$ admits a minimizer. 
}

Note, the condition \eqref{pf:L_cond} contains the case $L(r_1,r_2)=|r_1-r_2|^p$ as a special case: 





\begin{lemma}
If $L(r_1,r_2)=|r_1-r_2|^p$, for $1\leq p<\infty$, then $L$ satisfies the condition \eqref{pf:L_cond}. 
\end{lemma}
\begin{proof}
Assume that $L$ is defined on an interval of the form $[0,M]$, for some  $M>0$. 
Consider $r_1,r_1',r_2,r_2'\in [0,M]$. 
If $p=1$, by triangle inequality, we have
\begin{align}
|L(r_1,r_2)-L(r_1',r_2)|&=||r_1-r_2|-|r_1'-r_2||\leq |r_1-r_1'| \nonumber 
\end{align}
and similarly, 
\begin{align}
|L(r_1,r_2)-L(r_1,r_2')|&
\leq |r_2-r_2'|.\nonumber
\end{align}
From \cite[page 473]{memoli2011gromov}, since for $1\leq p<\infty$, the function $t\mapsto t^p$, for $t\in[0,M]$, is Lipschitz with constant bounded
by $pM^{p-1}$,  we have  
\begin{align}
|L(r_1,r_2)-L(r_1',r_2)|
&\leq 
pM^{p-1}|r_1-r_1'|.\nonumber 
\end{align}
and similarly, 
\begin{align}
|L(r_1,r_2)-L(r_1,r_2')|
&\leq pM^{p-1}|r_2-r_2'|. \nonumber     
\end{align}
\end{proof}
\begin{lemma}\label{lem:lip_func_q}
Given $q\ge 1$, consider $\beta>0$. Then $[0,\beta]\ni c\mapsto c^q \in [0,\beta^q]$ is a Lipschitz function. 
\end{lemma}
\begin{proof}
Given $c_1,c_2\in [0,\beta]$, we have 
\begin{align}
|c_1^q-c_2^q|\leq q\beta^{q-1}|c_1-c_2|
\end{align}
Thus, $c\mapsto c^q$ is a Lipschitz function. 
\end{proof}

\subsection{Convergence Auxiliary Result}

If a sequence $\{\gamma^n\}$ converges weakly to $\gamma$, we write $\gamma^n\overset{w}{\rightharpoonup}\gamma$. In this setting, if $\gamma^n\overset{w}{\rightharpoonup}\gamma$, it does not imply that  $(\gamma^n)^{\otimes2}\overset{w}{\rightharpoonup}\gamma^{\otimes2}$. Thus, the technique used in classical OT for proving the existence of a minimizer for the optimal transport optimization problem as a consequence of the Stone-Weierstrass theorem does not apply directly in the Gromov-Wasserstein context.

Inspired by \cite{memoli2011gromov}, we introduce the following lemma. 
\begin{lemma}\label{lem:pgw_minima}
Given metric space $(Z,d_Z)$, suppose $\phi:Z^2\to \mathbb{R}$ is a Lipschitz continuous function with respect to $(Z^2,d_Z^{+})$, where 
$$d_Z^{+}((z_1,z_2),(z_1',z_2')):=d_Z(z_1,z_1')+d_Z(z_2,z_2'), \qquad \forall (z_1,z_2),(z_1',z_2')\in Z^2.$$
Given  $\gamma\in\mathcal{M}_+(Z)$, and a sequence $\{\gamma^n\}_{n\ge 1}\in\mathcal{M}_+(Z)$ such that converges weakly to $\gamma$, $$\gamma^n \overset{w}{\rightharpoonup}\gamma \qquad ({n\to\infty}).$$ 
Finally, consider the mapping 
$$Z\ni z\mapsto \gamma(\phi(z,\cdot)):=\int_Z \phi(z,z')d\gamma(z')\in \mathbb{R}.$$
Then we have the following results: 
\begin{enumerate}[(1)]
\item   $\gamma^n(\phi(z,\cdot))\to \gamma(\phi(z,\cdot))$ uniformly (when $n\to\infty$). 
\item  
$(\gamma^n)^{\otimes2}(\phi(\cdot,\cdot))\to \gamma^{\otimes2}(\phi(\cdot,\cdot))$ (when $n\to\infty$). 
\item If $\mathcal{M}\subset \mathcal{M}_+(Z)$ is compact for the weak convergence, then 
$\inf_{\gamma\in \mathcal{M}}\gamma^{\otimes2}(\phi(\cdot,\cdot))$ admits a minimizer.  
\end{enumerate}
\end{lemma}
\begin{proof}
The main idea of the proof is similar to \cite[Lemma 10.3]{memoli2011gromov}:  we extend it from $\mathcal{P}_+(Z)$ to $\mathcal{M}_+(Z)$. 
\begin{enumerate}[(1)]
    \item  
    Since $\gamma^n \overset{w}{\rightharpoonup}\gamma$, and $Z$ is compact, we have $|\gamma^n|\to |\gamma|$.  Then, given $\epsilon>0$, for $n$ sufficiently large we have $|\gamma^n|\leq |\gamma|+\epsilon$. 
    
Let us denote by $\|\phi\|_{Lip}$ the Lipschitz constant of $\phi$. For any $z_1,z_2\in Z$, we have: 
\begin{align}
|\gamma^n(\phi(z_1,\cdot))-\gamma^n(\phi(z_2,\cdot))|
&\leq \int_Z |\phi(z_1,z)-\phi(z_2,z)| \gamma^n(z)\nonumber\\
&\leq \max_{z\in Z} |\phi(z_1,z)-\phi(z_2,z)|(|\gamma|+\epsilon)\nonumber\\
&\leq (|\gamma|+\epsilon)\|\phi\|_{Lip} \, d_Z(z_1,z_2)
= K d_Z(z_1,z_2) \nonumber, 
\end{align}
where $K=(|\gamma|+\epsilon)\|\phi\|_{Lip}$ is a finite positive value.
Note that the above inequality also holds if we replace $\gamma^n$ with $\gamma$. 

Since $(Z,d_Z)$ is compact, $Z=\bigcup_{i=1}^N B(z_i,\epsilon/K)$ for some $z_1,\ldots, z_N\in Z$, where $B(z_i,\epsilon/3K)=\{z\in Z:\, d_Z(z,z_i)\leq \epsilon/3K\}$ is the closed ball centered at $z_i$, with radius $\epsilon/K$. By definition of weak convergence, when $n$ is sufficiently large, 
$$|\gamma^n(\phi(z_i,\cdot))-\gamma(\phi(z_i,\cdot))|<\epsilon/3,\qquad \text{for each } i\in[1:N].$$
Given $z\in Z$, then $z\in B(z_i)$ for some $z_i$. For sufficiently large $n$, we have: 
\begin{align}
&|\gamma^n(\phi(z,\cdot))-\gamma(\phi(z,\cdot))|\nonumber\\
&\leq |\gamma^n(\phi(z,\cdot))-\gamma^n(\phi(z_i,\cdot))|+|\gamma^n(\phi(z_i,\cdot))-\gamma(\phi(z_i,\cdot))|+|\gamma(\phi(z_i,\cdot))-\gamma(\phi(z,\cdot))| \nonumber\nonumber\\
&\leq K d(z,z_i)+\epsilon/3 +K d(z,z_i)=\epsilon/3+\epsilon/3+\epsilon/3=\epsilon. 
\end{align}
Thus, we prove the first statement. 

\item We recall that we do not have $(\gamma^n)^{\otimes2}\overset{w}{\rightharpoonup}\gamma^{\otimes2}$.

Consider an arbitrary  $\epsilon>0$. We have, 
\begin{align}
0&\leq \limsup_{n\to \infty}|(\gamma^n)^{\otimes2}(\phi)-(\gamma)^{\otimes2}(\phi)| \label{pf:weak_converge_A_B}\\
&\leq \limsup_{n\to \infty}\underbrace{|(\gamma^n\otimes\gamma^n)(\phi)-(\gamma\otimes \gamma^n)(\phi)|}_{A_n}+\limsup_{n\to\infty}
\underbrace{|(\gamma^n\otimes\gamma)(\phi)-(\gamma\otimes \gamma)(\phi)|}_{B_n}.\nonumber
\end{align}
For the first term, when $n$ is sufficiently large, by statement (1), we have: 
\begin{align}
A_n&=\int\left(\gamma^n(\phi(z,\cdot))-\gamma(\phi(z,\cdot)\right)d\gamma^n(z)\nonumber\\
&\leq \max_z |\gamma^n(\phi(z,\cdot))-\gamma(\phi(z,\cdot)| |\gamma^n|\nonumber\\
&\leq \epsilon (|\gamma|+\epsilon)\label{pf:weak_converge_A_B1}
\end{align}
Thus, $\limsup_n A=\lim _n A=0$. 

Similarly, for the second term, when $n$ is sufficiently large, we have 
\begin{align}
B_n&:=\int (\gamma^n(\phi(z,\cdot))-\gamma(\phi(z,\cdot)))d\gamma(z)\leq \epsilon |\gamma|. \label{pf:weak_converge_A_B2}
\end{align}
Thus, $\limsup_n B_n=\lim_n B_n=0$.

Therefore, from \eqref{pf:weak_converge_A_B}, \eqref{pf:weak_converge_A_B1} and \eqref{pf:weak_converge_A_B2}, we obtain 
\begin{align}
\limsup_{n\to \infty}|(\gamma^n)^{\otimes2}(\phi)-(\gamma)^{\otimes2}(\phi)|=\lim_{n\to \infty}|(\gamma^n)^{\otimes2}(\phi)-(\gamma)^{\otimes2}(\phi)|=0.
\end{align}

\item Let $\gamma^n\in \mathcal{M}$ be a sequence such that $(\gamma^n)^{\otimes2}(\phi)$   (weakly) converges to $\inf_{\gamma\in \mathcal{M}}\gamma^{\otimes2}(\phi)$. 
Since $\mathcal{M}$ is compact, there exists a sub-sequence $\gamma^{n_k}\overset{w}{\rightharpoonup}\gamma$ for some $\gamma\in \mathcal{M}$. 
Then, by statement (2), we have: 
$$\gamma^{\otimes2}(\phi)=\lim_k (\gamma^{n_k})^{\otimes2}(\phi)=\inf_{\gamma\in\mathcal{M}}\gamma^{\otimes2}(\phi),$$
and we complete the proof. 
\end{enumerate}
\end{proof}

\subsection{Proof of the Formal Statement for Proposition \ref{pro: pgw minimizer}}

The proof follows the ideas of \cite[Corollary 10.1]{memoli2011gromov}.

Define $(Z,d_Z)$ as 
$Z:=X\times Y$, with  $d_Z((x,y),(x',y')):=d_X(x,x')+d_Y(y,y')$. 

We claim that the following mapping
\begin{align}
(X\times Y)^2=Z^2&\to \mathbb{R}\nonumber\\
((x,y),(x',y'))&\mapsto \phi((x,y),(x',y')):=L(d_X^q(x,x'),d_Y^q(y,y'))-2\lambda\nonumber 
\end{align}
is a Lipschitz function with respect to $d_Z^+$, where $L$ satisfies \eqref{pf:L_cond}. 
Indeed, given $((x_1,y_1),(x'_1,y'_1)),((x_2,y_2),(x'_2,y'_2))\in Z^2$, we have: 
\begin{align}
&|\phi((x_1,y_1),(x'_1,y'_1))-\phi((x_2,y_2),(x'_2,y'_2))|\nonumber\\
&=|L(d_X(x_1,x_1'),d_Y(y_1,y_1'))-L(d_X(x_2,x_2'),d_Y(y_2,y_2'))|\nonumber\\
&\leq |L(d_X(x_1,x_1'),d_Y(y_1,y_1'))-L(d_X(x_2,x_2'),d_Y(y_1,y_1'))|\nonumber\\
&\quad +|L(d_X(x_2,x_2'),d_Y(y_1,y_1'))-L(d_X(x_2,x_2'),d_Y(y_2,y_2'))| \nonumber\\
&\leq K|d_X^q(x_1,x'_1)-d_X^q(x_2,x'_2)|+K|d_Y^q(y_1,y'_1)-d_Y^q(y_2,y'_2)|\nonumber\\
&\leq K' |d_X(x_1,x'_1)-d_X(x_2,x'_2)|+K'|d_Y(y_1,y'_1)-d_Y(y_2,y'_2)| \label{pf:lip_x^q} \\
&\leq 
K'(d_X(x_1,x'_2)+d_X(x_1',x'_2))+K'(d_Y(y_1,y_2)+d_Y(y_1',y_2'))\label{pf:lipschitz_ineq}\\
&=K'\left[\left((d_X(x_1,x_2)+d_Y(y_1,y_2)\right)+\left((d_X(x_1',x_2')+d_Y(y_1',y_2')\right)\right]\nonumber\\
&=K'\left[d_Z((x_1,y_1),(x_2,y_2))+d_Z((x_1',y_1'),(x_2',y_2'))\right]\nonumber\\
&=K'd_Z^+(((x_1,y_1),(x_2,y_2)),((x_1,y_1),(x_2,y_2)))\nonumber 
\end{align}
where in \eqref{pf:lip_x^q}, $K'=q\beta^{q-1}K$; the inequality holds by lemma \ref{lem:lip_func_q}; The inequality  \eqref{pf:lipschitz_ineq} follows from the triangle inequality: 
\begin{align}
d_X(x_1,x'_1)-d_X(x_2,x_2')
&\leq d_X(x_1,x_2)+d_X(x_2,x_2')+d_X(x_2',x_1')-d_X(x_2,x_2')\nonumber\\
&=d_X(x_1,x_2)+d_X(x_1',x_2'), \nonumber 
\end{align}
and similarly, 
$$d_X(x_2,x'_2)-d_X(x_1,x_1')\leq d_X(x_1,x_2)+d_X(x_1',x_2').$$

Let $\mathcal{M}=\Gamma_\leq(\mu,\nu)$.
From \cite[Proposition B.1]{liu2023ptlp}, we have that $\Gamma_\leq(\mu,\nu)$ is a compact set with respect to the weak convergence topology.  

By Lemma \eqref{lem:pgw_minima} part (3), we have the PGW problem, which can be written as 
\begin{align}
\inf_{\gamma\in\Gamma_\leq(\mu,\nu)}\gamma^{\otimes2}(\phi)+\lambda(|\mu|^2+|\nu|^2) \nonumber
\end{align}
admits a solution, i.e., a minimizer $\gamma\in \Gamma_{\leq}(\mu,\nu)$. Therefore, we end the proof of Proposition \ref{pro: pgw minimizer}.




\section{Proof of Proposition \ref{pro: pgw is metric}: Metric Property of Partial GW}\label{sec: metric property}

Let $L(r_1,r_2)=D^p(r_1,r_2)$ for a metric $D$ on $\mathbb{R}$, and since all the metrics in $\mathbb{R}$ are equivalent, for simplicity, consider $D(r_1,r_2)=|r_1-r_2|$. (Notice that this satisfies the hypothesis of Proposition \ref{pro: tensor product gw} used in the experiments). 

Consider the GW problem, for $q\geq 1$,
\begin{equation}\label{eq: prob G^L_q}
GW^L_q(\mathbb{X},\mathbb{Y}):=\inf_{\gamma\in\Gamma(\mu,\nu)}\int_{(X\times Y)^2}L(d_X^q(x,x'),d_Y^q(y,y')) 
 \, d\gamma^{\otimes 2},
\end{equation}
or, in particular,
\begin{equation}\label{eq: prob G^p_q}
GW^p_q(\mathbb{X},\mathbb{Y}):=\inf_{\gamma\in\Gamma(\mu,\nu)}\int_{(X\times Y)^2}|d_X^q(x,x')-d_Y^q(y,y')|^p 
 \, d\gamma^{\otimes 2}.
\end{equation}
For probability mm-spaces we have the equivalence relation $\mathbb{X}\sim\mathbb{Y}$ if and only if $GW^p_q(\mathbb{X},\mathbb{Y})=0$. 


%

By \cite[Chapter 5]{memoli2011gromov}, $\mathbb{X}\sim \mathbb{Y}$ is equivalent to the following: there exists a bijective isometry mapping $\phi: X\to Y$, such that 
\begin{align}
 &d_X(x,x')-d_Y(\phi(x),\phi(x'))=0,\quad \mu^{\otimes2}-a.s. \nonumber\\
 &\phi_\#\mu=\nu \nonumber .
\end{align}

\begin{remark}\label{remark: equivalence class}
In the literature, the case where $q=1$ is the most frequently considered problem. In particular, in  \cite{memoli2011gromov} it is stated the equivalence relation 
$\mathbb{X}\sim\mathbb{Y}$ if and only if there exists $\phi: X\to Y$ such that $\phi_\#\mu=\nu$ and 
 $d_X(x,x')=d_Y(\phi(x),\phi(x'))$ $\mu^{\otimes2}-a.s.$ if and only if $GW^p_1(\mathbb{X},\mathbb{Y})=0$. Thus,  $\mathbb{X}\sim\mathbb{Y}$ is also equivalent to have $\phi: X\to Y$ such that $\phi_\#\mu=\nu$ and 
 $d_X(x,x')=d_Y(y,y')$ $\gamma^{\otimes2}-a.s.$ where $\gamma$ is a minimizer for $GW^p_1(\mathbb{X},\mathbb{Y})$. So, in this situation we also  have $d_X^q(x,x')=d_Y^q(y,y')$ $\gamma^{\otimes2}-a.s.$ for any given $q\geq 1$. Therefore, $\mathbb{X}\sim\mathbb{Y}$ if and only if  $GW^p_q(\mathbb{X},\mathbb{Y})=0$.
\end{remark}

\subsection{Formal Statement of Proposition \ref{pro: pgw is metric}}

We first introduce the formal statement of Proposition \ref{pro: pgw is metric}. To do so, we extend the equivalence relation $\sim$ to all mm-spaces (not only probability mm-spaces):
Given arbitrary mm-spaces $\mathbb{X}=(X,d_X,\mu)$, $\mathbb{Y}=(Y,d_Y,\nu)$, where $X,Y$ are compact and $\mu\in\mathcal{M}_+(X)$, $\nu\in\mathcal{M}_+(Y)$,  we write $\mathbb{X}\sim \mathbb{Y}$ if and only if they have the same total mass (i.e., $|\mu|=\mu(X)=\nu(Y)=|\nu|$) and $GW^p_q(\mathbb{X},\mathbb{Y})=0$.

{\itshape 

\textbf{Formal statement of Proposition \ref{pro: pgw is metric}:} Given $\lambda>0$, $1\leq p,q<\infty$, then $(PGW_{\lambda,q}^p(\cdot,\cdot))^{1/p}$ defines a metric among mm-spaces under taking quotient with respect to the equivalence relation $\sim$. 
}

Next, we discuss its proof.
\subsection{Non-Negativity and Symmetry Properties}

It is straightforward to verify $PGW_{\lambda,q}^p(\mathbb{X},\mathbb{Y})\ge 0$, and that $PGW_{\lambda,q}^p(\mathbb{X},\mathbb{Y})=PGW_{\lambda,q}^p(\mathbb{Y},\mathbb{X})$. In what follows, we will concentrate on proving $PGW_{\lambda,q}^p(\mathbb X,\mathbb Y)=0$ if and only if $\mathbb X\sim \mathbb Y$:

If $\mathbb{X}\sim \mathbb{Y}$, then $|\mu|=|\nu|$, and  we have 
$$0\leq PGW_{\lambda,q}^p(\mathbb{X},\mathbb{Y})\leq GW_q^p(\mathbb{X},\mathbb{Y})=0,$$
where the inequality follows from the fact 
$\Gamma(\mu,\nu)\subseteq \Gamma_\leq(\mu,\nu)$. 
Thus, $PGW_{\lambda,q}^p(\mathbb{X},\mathbb{Y})=0$.

For the other direction, suppose that $PGW_{\lambda,q}^p(\mathbb{X},\mathbb{Y})=0$.
We claim that $|\mu|=|\nu|$ and that there exist an optimal plan $\gamma$ for $PGW_{\lambda,q}^p(\mathbb{X},\mathbb{Y})$ such that  $|\mu|=|\gamma|=|\nu|$. Let us prove this by contradiction. 
Assume $|\mu|< |\nu|$. For convenience, suppose $|\mu|^2\leq |\nu|^2-\epsilon$, for some $\epsilon>0$. Then, for each $\gamma\in\Gamma_\leq(\mu,\nu)$, we have $|\gamma^{\otimes 2}|\leq |\mu|^2\leq |\nu|^2-\epsilon$, and so 
\begin{align*}
PGW_{\lambda,q}^p(\mathbb{X},\mathbb{Y}) \ge \lambda(|\mu|^2+|\nu|^2-2|\gamma|^2)\ge \lambda(|\nu^2|-|\gamma|^2)\ge \lambda \epsilon >0.
\end{align*}
Thus, $PGW_{\lambda,q}^p(\mathbb{X},\mathbb{Y})>0$, which is a contradiction. So, $|\mu|=|\nu|$. In addition, if $\gamma\in\Gamma_\leq(\mu,\nu)$ is optimal for $PGW_{\lambda,q}^p(\mathbb{X},\mathbb{Y})$, we have $|\gamma|=|\mu|=|\nu|$, thus $\gamma\in\Gamma(\mu,\nu)$. 
Therefore, since $PGW_{\lambda,q}^p(\mathbb{X},\mathbb{Y})=0$, 
and for such optimal $\gamma$ we have
$|\gamma|=|\mu|=|\nu|$, we obtain
\begin{align}
\int_{(X\times Y)^2}|d_X^q(x,x')-d_Y^q(y,y')|^p d\gamma^{\otimes2}=0.\nonumber
\end{align}
As a result,  $
d_X^q(x,x')=d_Y^q(y,y')$ $\gamma^{\otimes2}-a.s.$, which implies that 
$GW_q^p(\mathbb{X},\mathbb{Y})=0$, and so $\mathbb{X}\sim\mathbb{Y}$. 

\subsection{Triangle Inequality -- Strategy: Convert the PGW Problem into a GW Problem}\label{app: main technique of main proof}
Consider three arbitrary mm-spaces $\mathbb{S}=(S,d_S,\sigma)$, $\mathbb{X}=(X, d_{X},  \mu)$, $\mathbb{Y}=(Y, d_{Y},  \nu)$. 
We define
$\hat{\mathbb{S}}=(\hat S,d_{\hat S},\hat \sigma)$, $\hat{\mathbb{X}}=(\hat X, d_{\hat X}, \hat \mu)$, $\hat{\mathbb{Y}}=(\hat Y,d_{\hat Y}, \hat \nu)$ in a similar way to that of Proposition \ref{pro: pgw and gw} but now aiming to have new spaces with equal total mass:

First, introduce auxiliary points $\hat\infty_0,\hat\infty_1,\hat\infty_2$ and set 
\begin{align}
\begin{cases}
\hat{S}&=S\cup 
\{\hat\infty_0,\hat\infty_1,\hat\infty_2\},\nonumber\\
\hat{X}&=X\cup \{\hat\infty_0,\hat\infty_1,\hat\infty_2\},\nonumber \\
\hat{Y}&=Y\cup \{\hat\infty_0,\hat\infty_1,\hat\infty_2\}.\nonumber
\end{cases}
\end{align}
Define $\hat\sigma, \hat\mu,\hat\nu$ as follows: 
\begin{equation}
\begin{cases} 
\hat\sigma&=\sigma+|\mu|\delta_{\hat\infty_1}+|\nu|\delta_{\hat\infty_2},\\
\hat\mu&=\mu+|\sigma|\delta_{\hat\infty_0}+|\nu|\delta_{\hat\infty_2},\\
\hat\nu&=\nu+|\sigma|\delta_{\hat\infty_0}+|\mu|\delta_{\hat\infty_1}.
\end{cases}\label{pf:sigma_hat,mu_hat,nu_hat}    
\end{equation}
Note that $\hat\sigma$ is not supported on point $\hat\infty_0$, similarly, $\hat\mu$ is not supported on $\hat\infty_1$, $\hat\nu$ is not supported on $\hat\infty_2$.
In addition, we have $|\hat\mu|=|\hat\nu|=|\hat\sigma|=|\mu|+|\nu|+|\sigma|$. 
(For a similar idea in classical unbalanced optimal transport, see, for example, \cite{heinemann2023kantorovich}.) 

Finally, define $d_{\hat{S}}:\hat{S}^2\to \mathbb{R}\cup\{\infty\}$ as follows: 
\begin{align}
d_{\hat{S}}(s,s')=\begin{cases}
d_S(s,s') &\text{if }(s,s')\in S^2,\\
\infty &\text{elsewhere. }
\end{cases} \label{pf:d_S_hat}
\end{align}
Note, $d_{\hat{S}}(\cdot,\cdot)$ is not a rigorous metric in $\hat{S}$ since we allow $d_{\hat{S}}=\infty$. 
Similarly, define 
$d_{\hat{X}},d_{\hat{Y}}$. 
As a result, we have constructed new spaces
\begin{align}
\hat{\mathbb{S}}=(\hat{S},d_{\hat{S}},\hat\sigma), \quad \hat{\mathbb{X}}=(\hat{X},d_{\hat{X}},\hat\mu), \quad  \hat{\mathbb{Y}}=(\hat{Y},d_{\hat{Y}},\hat\nu) \label{eq:space_hat_pgw_gw}.
\end{align}

We define the following mapping $D_\lambda: (\mathbb{R}\cup\{\infty\})\times (\mathbb{R}\cup\{\infty\})\to \mathbb{R}_+$: 
\begin{align}\label{eq:D_lambda}
    D^p_{\lambda}(r_1,r_2)=
    \begin{cases}
    |r_1-r_2|^p &\text{if }r_1,r_2<\infty, \\ 
    \lambda  &\text{if }r_1=\infty, r_2<\infty \text{ or vice versa},\\
    0 &\text{if }r_1=r_2=\infty.
    \end{cases}
\end{align}
Note that $D_\lambda$ is not a rigorous metric since it may sometimes violate triangle inequality. See the following lemma for a detailed and precise explanation. 

\begin{lemma}\label{lem:D_lambda_metric}
Let $D_\lambda(\cdot,\cdot)$ denote the function defined in \eqref{eq:D_lambda}. 
For any 
$r_0,r_1,r_2\in \mathbb{R}\cup\{\infty\}$, we have the following: 
\begin{itemize}
    \item $D_\lambda(r_1,r_2)\ge0$. $D_\lambda(r_1,r_2)=0$ if and only if  $r_1=r_2$, where $r_1=r_2$ denotes that 
    $r_1=r_2\in \mathbb{R}$ or $r_1=r_2=\infty$. 
    
    \item Except the case $r_1,r_2\in \mathbb{R},r_0=\infty$, for all other cases, we have $$D_\lambda(r_1,r_2)\leq D_\lambda(r_1,r_0)+D_\lambda(r_2,r_0).$$ 
\end{itemize}

\end{lemma}
\begin{proof}[Proof of Lemma \ref{lem:D_lambda_metric}]
It is straightforward to verify $D_\lambda(\cdot,\cdot)\ge 0$.

Now, consider $r_0,r_1,r_2\in \mathbb{R}\cup\{\infty\}$.
If $r_1=r_2\in \mathbb{R}$ or $r_1=r_2=\infty$, 
we have 
$D_\lambda(r_1,r_2)=0$. 
Otherwise, $D_\lambda(r_1,r_2)>0$. 
So, $D_\lambda(r_1,r_2)=0$ if and only if $r_1=r_2$.

For the second item, we have the following cases: 

Case 1:  $r_1,r_2,r_0\in \mathbb{R}$, 
\begin{align}
D_\lambda(r_1,r_2) 
&=|r_1-r_2| \nonumber\\
&\leq |r_1-r_2|+|r_2-r_0| \nonumber\\
&=D_\lambda(r_0,r_1)+D_\lambda(r_0,r_2) \nonumber
\end{align}

Case 2: $r_1,r_2\in \mathbb{R},r_0=\infty$. We do not need to verify the inequality in this case.

Case 3: $r_1\in \mathbb{R},r_2,r_0=\infty$, or $r_1=\infty, r_2\in \mathbb{R}, r_0=\infty$. 
In this case, we have 
\begin{align}
D_\lambda(r_1,r_2)=D_\lambda(r_1,r_0)=\sqrt{\lambda}, D_\lambda(r_2,r_0)=0\nonumber
\end{align}
and it is straightforward to verify the inequality. 

Case 4: $r_1,r_2=\infty, r_3\in \mathbb{R}$. In this case, we have 
$D_\lambda(r_1,r_2)=0\leq D_\lambda(r_0,r_1)+D_\lambda(r_0,r_2)$. 

Case 5: $r_1,r_2,r_0=\infty$. 
In this case, we have 
\begin{align}
D_\lambda(r_1,r_2)=D_\lambda(r_1,r_0)=D_\lambda(r_2,r_0)=0\nonumber
\end{align}
and it is straightforward to verify the inequality. 
\end{proof}

We construct the following \textit{generalized GW problem}: 
\begin{align}
GW^p_{\lambda,q}(\hat{\mathbb{X}},\hat{\mathbb{Y}})&:=\inf_{\hat\gamma\in\Gamma(\hat\mu,\hat\nu)}\underbrace{\int_{(\hat X\times \hat Y)^2}D_\lambda^p(d^q_{\hat{X}}(x,x'),d^q_{\hat{Y}}(y,y')) \, d\hat \gamma^{\otimes 2}}_{\hat C(\hat\gamma;\lambda,\hat\mu,\hat\nu)} \label{eq: GW general}.
\end{align}
Similarly, we define $GW^p_{\lambda,q}(\hat{\mathbb{X}},\hat{\mathbb{S}})$, and  $GW^p_{\lambda,q}(\hat{\mathbb{S}},\hat{\mathbb{Y}})$.

The mapping \eqref{eq: ot and opt plan} is modified as: 
\begin{align}
&\Gamma_\leq(\sigma,\mu)\ni \gamma^{01}\mapsto\hat\gamma^{01} \in \Gamma(\hat\sigma,\hat\mu), \nonumber\\
&\hat\gamma^{01}:=\gamma^{01}+(\sigma-\gamma_1^{01})\otimes\delta_{\hat\infty_0}+\delta_{\hat\infty_1}\otimes(\mu-\gamma_2^{01})+|\gamma|\delta_{\hat\infty_1,\hat\infty_0}+|\nu|\delta_{\hat\infty_2,\hat\infty_2};\nonumber \\
&\Gamma_\leq(\sigma,\nu)\ni \gamma^{02}\mapsto\hat\gamma^{02} \in \Gamma(\hat\sigma,\hat\nu), \nonumber\\
&\hat\gamma^{02}:=\gamma^{02}+(\sigma-\gamma_1^{02})\otimes\delta_{\hat\infty_0}+\delta_{\hat\infty_2}\otimes(\nu-\gamma_2^{02})+|\gamma|\delta_{\hat\infty_2,\hat\infty_0}+|\mu|\delta_{\hat\infty_1,\hat\infty_1}; \nonumber \\
&\Gamma_\leq(\mu,\nu)\ni \gamma^{12}\mapsto\hat\gamma^{12}\in\Gamma(\hat\mu,\hat\nu), \nonumber\\
&\hat\gamma^{12}:=\gamma^{12}+(\mu-\gamma_1^{12})\otimes\delta_{\hat\infty_1}+\delta_{\hat\infty_2}\otimes(\nu-\gamma_2^{12})+|\gamma|\delta_{\hat\infty_2,\hat\infty_1}+|\mu|\delta_{\hat\infty_0,\hat\infty_0} \label{eq:gamma_pgw_gw}.
\end{align}
\begin{remark}
It is straightforward to verify the above mappings are well-defined. 
In addition, we can observe that, for each $\gamma^{01}\in\Gamma_\leq(\sigma,\mu),\gamma^{02}\in\Gamma_\leq(\sigma,\nu),\gamma^{12}\in\Gamma_\leq(\mu,\nu)$,
\begin{align}
&\hat{\gamma}^{01}(\{\hat\infty_2\}\times X)=\hat{\gamma}^{01}(S\times \{\hat\infty_2\})=0,  \label{pf:gamma01_cond} \\
&\hat{\gamma}^{02}(\{\hat\infty_1\}\times Y)=\hat{\gamma}^{02}(S\times \{\hat\infty_1\})=0,  \label{pf:gamma02_cond} \\
&\hat{\gamma}^{12}(\{\hat\infty_0\}\times Y)=\hat{\gamma}^{12}(X\times \{\hat\infty_0\})=0.  \nonumber 
\end{align}
\end{remark}

\begin{proposition}\label{pro:pgw_gw_metric}
If $\gamma^{12}\in\Gamma_\leq(\mu,\nu)$ is optimal in PGW problem $PGW_{\lambda,q}^p(\mathbb{X},\mathbb{Y})$, then $\hat\gamma^{12}$ defined in \eqref{eq:gamma_pgw_gw} is optimal in generalized GW problem $GW_{\lambda,q}^p(\hat{\mathbb{X}},\hat{\mathbb{Y}})$.
Furthermore, 
$\hat{C}(\hat\gamma^{12};\lambda,\hat{\mu},\hat{\nu})=C(\gamma^{12};\lambda,\mu,\nu),$
and thus, $$PGW_{\lambda,q}^p(\mathbb{X},\mathbb{Y})=
GW_{\lambda,q}^p(\hat{\mathbb{X}},\hat{\mathbb{Y}}).$$
\end{proposition}
\begin{proof}[Proof of Proposition \ref{pro:pgw_gw_metric}]
For each $\gamma\in\Gamma_{\leq}(\mu,\nu)$, define $\hat\gamma$ by \eqref{eq:gamma_pgw_gw}. 

Note that if we merge the points $\hat\infty_1,\hat\infty_2,\hat\infty_3$ as $\hat\infty$, i.e. $$\hat\infty=\hat\infty_1=\hat\infty_2=\hat\infty_3,$$
the value $\hat{C}(\hat\gamma;\lambda,\hat\mu,\hat\nu)$ will not change. Thus, we merge these three auxiliary points. 

We have: 
\begin{align}
& \hat{C}(\hat\gamma;\lambda,\hat\mu,\hat\nu)=\int_{(\hat X\times \hat Y)^2}D_\lambda^p(d^q_{\hat{X}}(x,x'),d^q_{\hat{Y}}(x,x'))d\hat \gamma^{\otimes 2}\nonumber \\
 &=\int_{(X\times Y)^2}|d^q_X(x,x')-d^q_Y(y,y')|^p d\hat\gamma^{\otimes 2}+\int_{(\{\hat\infty\}\times Y)^2 }\lambda d\hat\gamma^{\otimes 2}+\int_{(X\times \{\hat \infty\})^2}\lambda\hat\gamma^{\otimes 2} \notag\\
 & \quad +2\int_{(\{\hat\infty\}\times Y)\times (X\times Y)}\lambda d\hat\gamma^{\otimes 2}+2\int_{(X\times \{\hat \infty\})\times (X\times Y)}\lambda d\hat{\gamma}^{\otimes 2}+\int_{(\{\hat \infty\}\times\{\hat \infty\})^2}D_\lambda^p(\infty,\infty)d\hat\gamma^{\otimes 2}\nonumber\\
 &\quad  +2\int_{(\{\hat \infty\}\times Y)\times (X\times \{\hat \infty \})}D_\lambda^p(\infty,\infty)d\hat\gamma^{\otimes 2}+2\int_{(\{\hat \infty\}\times \{\hat \infty\})\times(X\times Y)}D_\lambda^p(\infty,\infty)d\hat\gamma^{\otimes 2}\notag\\
 &\quad +2\int_{(\{\hat \infty\}\times\{Y\})\times \{\hat\infty\}^2}D_\lambda^p(\infty,\infty)d\hat\gamma^{\otimes 2}+2\int_{(X\times\{\hat \infty\})\times \{\hat\infty\}^2}D_\lambda^p(\infty,\infty)d\hat\gamma^{\otimes 2}\notag\\
 &=\int_{(X\times Y)^2}|d^q_X(x,x')-d^q_Y(y,y')|^p d\gamma^{\otimes 2}\notag\\
 &\quad +2\lambda (|\nu|-|\gamma|)|\gamma|+\lambda (|\nu|-|\gamma|)^2+2\lambda (|\mu|-|\gamma|)|\gamma|+\lambda(|\mu|-|\gamma|)^2\nonumber \\ 
 &=\int_{(X\times Y)^2}|d^q_X(x,y')-d^q_Y(y,y')|^p  \, d\gamma^{\otimes 2})+\lambda(|\nu^2|+|\mu|^2-2|\gamma|^2)=C(\gamma;\lambda,\mu,\nu).\notag
\end{align}
As we merged the points $\hat\infty_1,\hat\infty_2,\hat\infty_3$, by  \cite[Proposition B.1.]{bai2022sliced}, 
the mapping $\gamma\mapsto \hat\gamma$ defined in \eqref{eq:gamma_pgw_gw} is a bijection. 
Then, if $\gamma\in\Gamma_\leq(\mu,\nu)$ is optimal for the PGW problem $PGW_{\lambda,q}^p(\mathbb{X},\mathbb{Y})$ (defined in \eqref{eq:pgw}),  $\hat\gamma\in\Gamma(\hat\mu,\hat\nu)$ is optimal for generalized GW problem $GW_{\lambda,q}^p(\hat{\mathbb{X}},\hat{\mathbb{Y}})$ (defined in 
\eqref{eq: GW general}).  
Therefore, $$GW_{\lambda,q}^p(\hat{\mathbb{X}},\hat{\mathbb{Y}})=PGW_{\lambda,q}^p(\mathbb{X},\mathbb{Y}).$$
\end{proof}

\begin{proposition}[Triangle inequality for $GW_{\lambda,q}^p(\cdot,\cdot)$]\label{pro:ggw_tri}    
Consider the generalized GW problem \eqref{eq: GW general}. 
Then, for any $p\in[1,\infty)$, we have 
$$GW_{\lambda,q}^p(\hat{\mathbb{X}},\hat{\mathbb{Y}})\leq GW_{\lambda,q}^p(\hat{\mathbb{S}},\hat{\mathbb{X}})+GW_{\lambda,q}^p(\hat{\mathbb{S}},\hat{\mathbb{Y}}).$$
\end{proposition}

\begin{proof}[Proof of Proposition \ref{pro:ggw_tri}]
We prove the case $p=2$. For general $p\ge1$, it can be proved similarly. 

Choose an optimal $\gamma^{12}\in\Gamma_\leq(\mu,\nu)$ for $PGW_{\lambda,q}^2(\mathbb{X},\mathbb{Y})$, an optimal $\gamma^{01}\in \Gamma_\leq(\sigma,\mu)$ for $PGW_{\lambda,q}^2(\mathbb{S},\mathbb{X})$, and an optimal $\gamma^{02}\in\Gamma_\leq(\sigma,\nu)$ for 
$PGW_{\lambda,q}^2(\mathbb{S},\mathbb{Y})$. Construct $\hat\gamma^{12},\hat\gamma^{01},\hat\gamma^{02}$ by \eqref{eq:gamma_pgw_gw}. 

By Proposition \ref{pro:pgw_gw_metric}, we have that
$\hat{\gamma}^{12}$, $\hat{\gamma}^{01}$, $\hat{\gamma}^{02}$ are optimal for $GW^2_{\lambda,q}(\hat{\mathbb{X}},\hat{\mathbb{Y}})$, $GW^2_{\lambda,q}(\hat{\mathbb{S}},\hat{\mathbb{X}})$, $GW^2_{\lambda,q}(\hat{\mathbb{S}},\hat{\mathbb{Y}})$, respectively. 

Define canonical projection mapping 
\begin{align}
\pi_{0,1}: &(\hat S\times \hat X\times \hat Y)\to(\hat S\times \hat X) \nonumber \\
&(s,x,y)\mapsto (s,x)\nonumber. 
\end{align}
Similarly, we define $\pi_{0,2},\pi_{1,2}$. 

By \textit{gluing lemma} (see Lemma 5.5 \cite{santambrogio2015optimal}), there exists $\hat\gamma\in \mathcal{M}_+(\hat{S}\times \hat{X}\times \hat{Y})$, such that $(\pi_{0,1})_\#\hat\gamma=\hat\gamma^{01},(\pi_{0,2})_\#\hat\gamma=\hat\gamma^{02}$.  
Thus, $(\pi_{1,2})_\#\hat\gamma$ is a coupling between $\hat\mu,\hat\nu$. We have 
\begin{align}
GW^2_{\lambda,q}(\mathbb{X},\mathbb{Y})&=\int_{(\hat X\times \hat Y)^2}D^2_\lambda(d^q_{\hat{X}}(x,x'),d^q_{\hat{Y}}(y,y'))d(\hat\gamma^{12})^{\otimes2}\nonumber\\
&\leq \int_{(\hat{S}\times\hat X\times \hat Y)^2}D^2_\lambda(d^q_{\hat{X}}(x,x'),d^q_{\hat{Y}}(y,y'))d\hat{\gamma}^{\otimes2}.\label{pf:pgw_tri_1}
\end{align}
The inequality holds since $(\pi_{1,2})_\#\hat\gamma,\hat\gamma^{12}\in\Gamma(\hat\mu,\hat\nu)$, and $\hat\gamma^{12}$ is optimal. 

Next, we will show that 
\begin{align}
&\int_{(\hat{S}\times\hat X\times \hat Y)^2}D^2_\lambda(d^q_{\hat{X}}(x,x'),d^q_{\hat{Y}}(y,y'))d\hat{\gamma}^{\otimes2}\nonumber\\
&\leq \int_{(\hat{S}\times\hat X\times \hat Y)^2}(D_\lambda(d^q_{\hat{S}}(s,s'),d^q_{\hat{X}}(x,x'))+D_\lambda(d^q_{\hat{S}}(s,s'),d^q_{\hat{Y}}(y,y')))^2d\hat{\gamma}^{\otimes2}.\nonumber
\end{align}
Let $((s,x,y),(s',x',y'))\in (\hat{S},\hat{X},\hat{Y})^2$, and assume that 
\begin{align}
D_\lambda(d_{\hat X}^2(x,x'),d_{\hat Y}^2(y,y'))>D_\lambda(d_{\hat{S}}^2(s,s'),d_{\hat X}^2(x,x'))+D_\lambda(d_{\hat{S}}^2(s,s'),d_{\hat Y}^2(y,y')). \label{pf:tri_inq_violate}
\end{align}

By Lemma \ref{lem:D_lambda_metric},  \eqref{pf:tri_inq_violate} implies  $d_{\hat X}(x,x'),d_{\hat Y}(y,y')\in \mathbb{R},d_{\hat S}(s,s')=\infty$. 
Thus, 
by definition \eqref{pf:d_S_hat}, it also implies \begin{equation}
 (x,x')\in X^2, (y,y')\in Y^2, (s,s')\in \hat{S}^2\setminus S^2. \label{pf:pgw_tri_cond}   
\end{equation}
Define the following sets: 
\begin{align}
A_{\alpha}&=\hat{S}\times X\times Y,\nonumber\\
A_0&=\{\hat\infty_0\}\times X\times Y, \nonumber\\
A_1&=\{\hat\infty_1\}\times X\times Y, \nonumber\\
A_2&=\{\hat\infty_2\}\times X\times Y.\nonumber 
\end{align}
Notice that, \eqref{pf:tri_inq_violate} 
 $\Longrightarrow$ \eqref{pf:pgw_tri_cond} is equivalent to 
\begin{equation}
\eqref{pf:tri_inq_violate} \Longrightarrow
((s,x,y),(s,x',y'))\in A:=\bigcup_{i=0}^2 (A_i\times A_\alpha)\cup \bigcup_{i=0}^2(A_\alpha\times A_i). \label{pf:set_A}    
\end{equation}
Next, we will show $\hat \gamma^{\otimes2}(A)=0$. Indeed,
\begin{align}
&\hat\gamma(A_0)\leq \hat\gamma(\{\infty_0\}\times \hat X\times \hat{Y})=\hat\sigma(\{\infty_0\})=0\nonumber \qquad\text{by definition \eqref{pf:sigma_hat,mu_hat,nu_hat} of $\hat\sigma$ , }\nonumber\\
&\hat\gamma(A_1)\leq \hat\gamma(\{\infty_1\}\times  \hat X\times Y)=\hat\gamma^{02}(\{\hat\infty_1\times Y\})=0 \qquad\text{by  }\eqref{pf:gamma02_cond}, \nonumber \\
&\hat\gamma(A_2)\leq \hat\gamma(\{\infty_2\}\times   X\times \hat Y)=\hat\gamma^{01}(\{\hat\infty_2\times X\})=0 \qquad\text{by }\eqref{pf:gamma01_cond} \nonumber. 
\end{align}
Thus, $\hat{\gamma}^{\otimes2}(A)=0$. By considering 
$B=(\hat{S}\times\hat{X}\times {Y})^2\setminus A,$
we obtain
\begin{align}
&\int_{(\hat{S}\times \hat{X}\times \hat{Y})^2}D_\lambda^2(d^q_{\hat{X}}(x,x'),d^q_{\hat{Y}}(y,y'))d\gamma^{\otimes2}\nonumber\\
&=\int_{B}D_\lambda^2(d^q_{\hat{X}}(x,x'),d^q_{\hat{Y}}(y,y'))d\gamma^{\otimes2}\qquad\text{since }\gamma^{\otimes2}(A)=0\nonumber \\
&\leq \int_{B}\left(D_\lambda(d^q_{\hat{S}}(s,s'),d^q_{\hat{X}}(x,x')+D_\lambda(d^q_{\hat{S}}(s,s'),d^q_{\hat{Y}}(y,y'))\right)^2d\gamma^{\otimes2}\qquad\text{by }\eqref{pf:set_A}\nonumber\\
&\leq \int_{(\hat{S}\times\hat{X}\times\hat{Y})^2}\left(D_\lambda(d^q_{\hat{S}}(s,s'),d^q_{\hat{X}}(x,x')+D_\lambda(d^q_{\hat{S}}(s,s'),d^q_{\hat{Y}}(y,y'))\right)^2d\gamma^{\otimes2}.\label{pf:pgw_tri_2}
\end{align}

Following \eqref{pf:pgw_tri_1} and \eqref{pf:pgw_tri_2}, we have 
\begin{align}
&GW^2_{\lambda,q}(\hat{\mathbb{X}},\hat{\mathbb{Y}})\leq\left(\int_{(\hat{S}\times\hat X\times \hat Y)^2}D^2_\lambda(d^q_{\hat{X}}(x,x'),d^q_{\hat{Y}}(y,y'))d\hat{\gamma}^{\otimes2}\right)^{1/2}\nonumber\\
&\leq \left(\int_{(\hat{S}\times\hat{X}\times\hat{Y})^2}\left(D_\lambda(d^q_{\hat{S}}(s,s'),d^q_{\hat{X}}(x,x'))+D_\lambda(d^q_{\hat{S}}(s,s'),d^q_{\hat{Y}}(y,y'))\right)^2d\gamma^{\otimes2} \right)^{1/2}\nonumber\\
&\leq \left(\int_{(\hat{S}\times\hat{X}\times\hat{Y})^2}D^2_\lambda(d^q_{\hat{S}}(s,s'),d^q_{\hat{X}}(x,x'))d\gamma^{\otimes2}\right)^{1/2}\nonumber\\
&\qquad+\left(\int_{(\hat{S}\times\hat{X}\times\hat{Y})^2}D^2_\lambda(d^q_{\hat{S}}(s,s'),d^q_{\hat{Y}}(y,y'))d\gamma^{\otimes2}\right)^{1/2}\label{pf:Minkowski_inq}\\
&=\left(\int_{(\hat{S}\times\hat{X}\times\hat{Y})^2}D^2_\lambda(d^q_{\hat{S}}(s,s'),d^q_{\hat{X}}(x,x'))d(\gamma^{01})^{\otimes2}\right)^{1/2}\nonumber\\
&\qquad +\left(\int_{(\hat{S}\times\hat{X}\times\hat{Y})^2}D^2_\lambda(d^q_{\hat{S}}(s,s'),d^q_{\hat{Y}}(y,y'))d(\gamma^{02})^{\otimes2}\right)^{1/2}\nonumber\\
&=GW^2_{\lambda,q}(\hat{\mathbb{S}},\hat{\mathbb{X}})+GW^2_{\lambda,q}(\hat{\mathbb{S}},\hat{\mathbb{Y}}), \nonumber
\end{align}
where in the third inequality \eqref{pf:Minkowski_inq} we used the Minkowski inequality in $L^2((\hat S\times \hat X\times \hat Y)^2,\hat\gamma^{\otimes 2})$.
\end{proof}

Now, we can complete the proof of Proposition \ref{pro: pgw is metric}: 
By the Propositions \ref{pro:pgw_gw_metric}, we have $$PGW^p_{\lambda,q}(\mathbb{X},\mathbb{Y})=GW^p_{\lambda,q}(\hat{\mathbb{X}},\hat{\mathbb{Y}})$$ 
and similarly for $PGW^p_{\lambda,q}$ and $(\mathbb{S},\mathbb{X}),PGW^p_{\lambda,q}(\mathbb{S},\mathbb{Y})$. 
By the Proposition \ref{pro:ggw_tri}, $GW^p_{\lambda,q}(\cdot,\cdot)$ satisfies the triangle inequality, thus
We complete the proof: 
\begin{align}
PGW^p_{\lambda,q}(\mathbb{X},\mathbb{Y})&=GW^p_{\lambda,q}(\hat{\mathbb{X}},\hat{\mathbb{Y}})\nonumber\\
&\leq GW^p_{\lambda,q}(\hat{\mathbb{S}},\hat{\mathbb{X}})+GW^p_{\lambda,q}(\hat{\mathbb{S}},\hat{\mathbb{Y}})\nonumber\\
&=PGW^p_{\lambda,q}(\mathbb{S},\mathbb{X})+PGW^p_{\lambda,q}(\mathbb{S},\mathbb{Y}).\nonumber
\end{align}





\section{Proof of Proposition \ref{prop: limit for lambda}: PGW converges to GW as $\lambda\to\infty$.}\label{sec: limit lambda}
In the main text, we set $\lambda\in \mathbb{R}$. In this section, we discuss the limit case when $\lambda\to \infty$. 

\begin{lemma}\label{lem:gamma_mass}
Suppose $|\mu|\leq |\nu|$, for each $\gamma\in \Gamma_\leq(\mu,\nu)$, there exists $\gamma'\in \Gamma_\leq(\mu,\nu)$ such that $\gamma\leq \gamma'$ and $(\pi_1)_\#\gamma'=\mu$. 
\end{lemma}
\begin{proof}
Let $\gamma\in\Gamma_\leq(\mu,\nu)$. 

If $|\gamma|=|\mu|$, then we have $(\pi_1)_\#\gamma=\mu$.

If $|\gamma|<|\mu|$, let $\mu^r=\mu-(\pi_1)_\#\gamma,\nu^r=\nu-(\pi_2)_\#\gamma$. We have that $\mu^r,\nu^r$ are non-negative measures, with $|\mu^r|=|\mu|-|\gamma|>0$. 
If we define 
\begin{equation}
  \gamma':=\gamma+\frac{1}{|\nu|-|\gamma|}\mu^r\otimes \nu^r,\nonumber
\end{equation}
we obtain $\gamma\leq \gamma'$. 
In addition, we have: 
\begin{align}
(\pi_1)_\#\gamma'&=(\pi_1)_\#\gamma+\mu^r \frac{|\nu^r|}{|\nu|-|\gamma|}=(\pi_1)_\#\gamma+\mu^r=\mu,\nonumber\\
(\pi_2)_\#\gamma'&=(\pi_2)_\#\gamma+\nu^r \frac{|\mu^r|}{|\nu|-|\gamma|}\leq (\pi_2)_\#\gamma+\nu^r \frac{|\nu^r|}{|\nu|-|\gamma|}=\nu.\nonumber
\end{align}
Thus, $\gamma'\in\Gamma_\leq(\mu,\nu)$ and $(\pi_1)_\#\gamma'= \mu$. 
\end{proof}

\begin{lemma}\label{lem:large_lambda}
Given general mm-spaces $\mathbb{X}=(X,d_X,\mu)$, $\mathbb{Y}=(Y,d_Y,\nu)$, where $\mu,\nu$ are supported on bounded sets (in general, it is assumed that $X$ and $Y$ are compact, and that $\text{supp}(\mu)=X$, $\text{supp}(\nu)=Y$), consider the problem the problem $PGW^L_{\lambda,q}(\mathbb{X},\mathbb{Y})$
with $L(r_1,r_2)$ a continuous functions. If $\lambda$ is sufficiently large, in particular:
$$\lambda\ge \max_{\substack{x,x'\in \text{supp}(\mu)\\ 
y,y'\in \text{supp}(\nu)}}L(d_X(x,x'),d_Y(y,y')),$$
then there exists optimal $\gamma$ for $PGW_\lambda(\mathbb{X},\mathbb{Y})$ such that $|\gamma|=\min(|\mu|,|\nu|)$.

Furthermore, when $$\lambda> \max_{\substack{x,x'\in \text{supp}(\mu)\\ 
y,y'\in \text{supp}(\nu)}}L(d_X(x,x'),d_Y(y,y')),$$
the for all optimal $\gamma\in\Gamma_\leq(\mu,\nu)$, we have $|\gamma|=\min(|\mu|,|\nu|)$. 
\end{lemma}
\begin{proof}
We prove it for $q=1$. For a general $q\ge 1$, it can be proved similarly. 

Without loss of generality, suppose $|\mu|\leq |\nu|$. 

Since $\mu,\nu$ are supported on bounded sets, there exists $A=[0,M]$ such that  $d_X(x,x'),d_Y(y,y')\in A$ for all $x,x'\in \text{supp}(\mu),y,y'\in \text{supp}(\nu)$.

Thus, the restriction of $L$ on $A^2$, denoted as $L_{A^2}$, is continuous on $A^2$, and thus it is bounded. 
So, consider $$\mathrm{m}:=\max_{r_1,r_2\in A}(L(r_1,r_2))\ge L(d_X(x,x'),d_Y(y,y')),\quad \forall x,x'\in\text{supp}(\mu),y,y'\in\text{supp}(\nu).$$

Suppose $2\lambda\ge \mathrm{m}+1$, and assume that there exists a optimal $\gamma\in\Gamma_\leq (\mu,\nu)$ such that $|\gamma|<|\mu|$. By Lemma \ref{lem:gamma_mass}, there exists $\gamma'$ such that $\gamma\leq \gamma',(\pi_1)_\#\gamma'=\mu$. Thus, we have
\begin{align}
C(\gamma';\lambda,\mu,\nu)-C(\gamma;\lambda,\mu,\nu)\nonumber&=\int_{(X\times Y)} L(d_X(x,x'),d_Y(y,y'))-2\lambda  \, d((\gamma')^{\otimes2}-(\gamma)^{\otimes2})\nonumber\\
&\leq \int_{(X\times Y)}\mathrm{m}-2\lambda  \, d((\gamma')^{\otimes2}-(\gamma)^{\otimes2})\nonumber\\
&=- (|\gamma'|^{2}-|\gamma|^{2})=-(|\mu|^2-|\gamma|^2)< 0,\nonumber
\end{align}
which is contradiction since $\gamma$ is optimal, and so we have completed the proof. 
\end{proof}

\begin{lemma}\label{lem:pgw_leq_gw} Consider probability mm-spaces $\mathbb{X}=(X,d_X,\mu)$, $\mathbb{Y}=(Y,d_Y,\nu)$, that is, with $|\mu|=|\nu|=1$. Then, for each $\lambda>0$, we have 
\begin{align}
PGW_{\lambda,q}^L(\mathbb{X},\mathbb{Y})\leq GW^L_q(\mathbb{X},\mathbb{Y}).\nonumber
\end{align}
\end{lemma}
\begin{proof}
In this setting, we have $\Gamma(\mu,\nu)\subset \Gamma_\leq(\mu,\nu)$, and thus
\begin{align}
&PGW^L_{\lambda,q}(\mathbb{X},\mathbb{Y})\nonumber\\
&=\inf_{\Gamma\in\Gamma_\leq(\mu,\nu)}\int_{(X\times Y)^2}L(d_X^q(x,x'),d_Y^q(y,y'))d\gamma^{\otimes 2}+\lambda(|\mu|^2+|\nu|^2-2|\gamma|^2)\nonumber\\
&\leq \inf_{\gamma\in\Gamma(\mu,\nu)}\int_{(X\times Y)^2}L(d_X^q(x,x'),d_Y^q(y,y'))+\lambda(|\mu|^2+|\nu|^2-2|\gamma|^2)d\gamma^{\otimes 2}\nonumber\\
&=\inf_{\gamma\in\Gamma(\mu,\nu)}\int_{(X\times Y)^2}L(d_X^q(x,x'),d_Y^q(y,y'))d\gamma^{\otimes 2}\nonumber\\
&=GW^L_q(\mathbb{X},\mathbb{Y})\nonumber. 
\end{align}
\end{proof}

Based on the above properties, we can now prove Proposition \ref{prop: limit for lambda}: 
\begin{proposition}[Generalization of Proposition \ref{prop: limit for lambda}]
Consider general probability mm-spaces $\mathbb{X}=(X,d_X,\mu)$, $\mathbb{Y}=(Y,d_Y,\nu)$, that is, with $|\mu|=|\nu|=1$,   where $X,Y$ are bounded. Assume that $L$ is continuous.
Then 
$$\lim_{\lambda\to \infty}PGW_{\lambda,q}^L(\mathbb{X},\mathbb{Y})=GW^L_q(\mathbb{X},\mathbb{Y}).$$
\end{proposition}
\begin{proof}
When $\lambda$ is sufficiently large, by Lemma \ref{lem:large_lambda}, for each optimal  $\gamma_\lambda\in\Gamma_\leq(\mu,\nu)$ of the minimization problem $PGW^L_{\lambda,q}(\mathbb{X},\mathbb{Y})$, we have $|\gamma_\lambda|=\min(|\mu|,|\nu|)=1$. That is, $\gamma_\lambda\in\Gamma(\mu,\nu)$. Plugging $\gamma_\lambda$ into $C(\gamma_\lambda;\lambda,\mu,\nu)$, we obtain: 
\begin{align}
PGW^L_{\lambda,q}(\mathbb{X},\mathbb{Y})&=\int_{(X\times Y)^2}L(d_X^q(x,x'),d_Y^q(y,y'))d\gamma_\lambda^{\otimes2}+\lambda(1^2+1^2-2\cdot 1^2)\nonumber\\
&=\int_{(X\times Y)^2}L(d_X^q(x,x'),d_Y^q(y,y'))d\gamma_\lambda^{\otimes2}\ge GW(\mathbb{X},\mathbb{Y}).\nonumber
\end{align}
 By Lemma \ref{lem:pgw_leq_gw}, we also have
$PGW^L_{\lambda,q}(\mathbb{X},\mathbb{Y})\leq GW^L_q(\mathbb{X},\mathbb{Y})$ and we complete the proof. 
\end{proof}

\section{Tensor Product Computation}\label{sec: tensor prod}
\begin{lemma}\label{lem: tensor prod}
Given a tensor $M\in \mathbb{R}^{n\times m\times n\times n}$ and $\gamma,\gamma'\in \mathbb{R}^{n\times m}$, the tensor product operator 
$M\circ \gamma$ satisfies the following: 
\begin{enumerate}
    \item[(i)] The mapping $\gamma\mapsto M\circ \gamma$ is linear with respect to $\gamma$.
    \item[(ii)] If $M$ is symmetric, in particular, $M_{i,j,i',j'}=M_{i',j',i,j},\forall i,i'\in[1:n],j,j'\in[1:m]$, then 
    $$\langle M\circ \gamma,\gamma' \rangle_F=\langle M\circ \gamma',\gamma\rangle_F.$$
\end{enumerate}
\end{lemma}
\begin{proof}$ $

\begin{enumerate}
    \item[(i)] For the first part, consider $\gamma,\gamma'\in \mathbb{R}^{n\times m}$ and $k\in\mathbb{R}$.
For each $i,j\in[1:n]\times [1:m]$, we have 
we have 
\begin{align}
(M\circ (\gamma+\gamma'))_{ij}
&=\sum_{i',j'}M_{i,j,i',j'}(\gamma+\gamma')_{i'j'}\nonumber\\
&=\sum_{i',j'}M_{i,j,i',j'}\gamma_{i'j'}+\sum_{i',j'}M_{i,j,i',j'}\gamma'_{i'j'}\nonumber\\
&=(M\circ \gamma)_{ij}+(M\circ \gamma)_{i'j'},\nonumber\\
(M\circ (k\gamma))_{ij}
&=\sum_{i',j'}M_{i,j,i',j'}(k\gamma)_{ij}\nonumber\\
&=k\sum_{i',j'}M_{i,j,i',j'}\gamma_{ij}\nonumber\\
&=k(M\circ \gamma)_{ij}.\nonumber
\end{align}
Thus, $M\circ (\gamma+\gamma')=M\circ \gamma+M\circ \gamma'$ and $M\circ (k\gamma)=k M\circ \gamma$. Therefore, $\gamma\mapsto M\circ \gamma$ is linear. 

\item[(ii)] 
For the second part, we have 
\begin{align}
\langle M\circ \gamma,\gamma' \rangle_F&=\sum_{iji'j'}M_{i,j,i',j',}\gamma_{ij}\gamma'_{i'j'}\nonumber\\
&=\sum_{i,j,i',j'}M_{i',j',i,j}\gamma_{i',j'}\gamma_{i,j}\label{pf: tensor prod ex1}\\
&=\langle M\gamma',\gamma \rangle\nonumber  
\end{align}
where \eqref{pf: tensor prod ex1} follows from the fact that $M$ is symmetric.
\end{enumerate} 
\end{proof}

\section{Another Algorithm for Computing PGW Distance -- Solver 2}\label{sec:solver2}

Our Algorithm 2 for solving the proposed PGW problem is based on a theoretical result that relates GW and PGW. The details of our computational method, as well as the proof of Proposition \ref{pro: pgw and gw} stated below,  are provided in Appendix \ref{subsec: alg 2}. Based on such a proposition, we
 extend the PGW problem to a discrete \textit{GW-variant} problem \eqref{eq: variant gw discrete}, leading to a solution for the original PGW problem by truncating the GW-variant solution. 

\begin{proposition}\label{pro: pgw and gw}
Let $\mathbb{X}=(X,d_X,\mu)$ be a mm-space. 
Consider an auxiliary point $\hat\infty$ and let $\hat{\mathbb{X}}=(\hat{X},d_{\hat X}, \hat\mu)$, where $\hat{X}=X\cup \{\hat\infty\}$,  $\hat\mu$ is constructed by \eqref{eq: mu_hat},  and considering $\infty$ as an auxiliary point to $\mathbb{R}$ such that $x\leq \infty$ for every $x\in \mathbb{R}$, we extend $d_X$ into $d_{\hat{X}}: \hat{X}^2\to \mathbb{R}\cup\{\infty\}$ and define $L_{\lambda}:\mathbb{R}\cup\{\infty\}\to \mathbb{R}$ as follows: 
\begin{align}\label{eq: d_hat}
  d_{\hat{X}}(x,x')=\begin{cases}
 d_X(x,x') &\text{ if }x,x'\in X\\ 
    \infty & \text{ otherwise}
\end{cases},  
L_{\lambda}(r_1,r_2):=\begin{cases}
    L(r_1,r_2)-2\lambda  &\text{ if }r_1,r_2\in\mathbb{R}\\ 
    0 &\text{ elsewhere}
\end{cases}.   
\end{align}
Consider the following GW-variant\footnote{$\widehat{GW}^{L_\lambda}(\hat{\mathbb{X}},\hat{\mathbb{Y}})$ is not a rigorous GW problem since $d_{\hat{X}}=\infty$ is possible, thus it is not a metric. Also, $\mathbb{X}$, $\mathbb{Y}$ are not necessarily \textit{probability} mm-spaces} problem: 
\begin{align} 
\widehat{GW}^{L_\lambda}(\hat{\mathbb{X}},\hat{\mathbb{Y}})= \inf_{\hat\gamma\in\Gamma(\hat \mu,\hat\nu)}&\hat\gamma^{\otimes2}(L_{\lambda}(d_{\hat X}^q,d_{\hat Y}^q))\label{eq: GW variant}  
\end{align}
Then, when considering the bijection  $\gamma\mapsto \hat\gamma$ defined in \eqref{eq: ot and opt plan}, we have that $\gamma$ is optimal for PGW problem \eqref{eq:pgw} if and only if $\hat\gamma$ is optimal for the GW-variant problem \eqref{eq: GW variant}.
\end{proposition}
\begin{remark}
Intuitively, the above proposition states that, by introducing auxiliary points $\hat\infty$, we can build an equivalent relation between PGW and GW problem. This idea is firstly discussed in \cite{cagniart2010free} in a classical optimal partial transport setting. In this paper, we extend the technique to the partial GW setting. However, this technique cannot be extended to the MPGW setting; we refer to \cite[Appendix A.3]{chapel2020partial} for details. 
\end{remark}

\begin{proof}
    The mapping $F$ defined by \eqref{eq: ot and opt plan} well-defined bijection, as shown in\cite{bai2022sliced,caffarelli2010free}.

Given $\gamma\in\Gamma_\leq(\mu,\nu)$, we have $\hat\gamma=F(\gamma)\in\Gamma(\hat\mu,\hat\nu)$. 
Let $\hat{C}(\hat{\gamma};\mu,\nu)$ denote the transportation cost in the GW-variant problem \eqref{eq: GW variant}, that is, 
\begin{align*}
\hat{C}(\hat{\gamma};\mu,\nu):=
\int_{(\hat X\times \hat Y)^2}L_{\lambda}(d^q_{\hat X}(x,x'),d^q_{\hat Y}(y,y')) \, d\hat{\gamma}(x,y)d\hat{\gamma}(x',y')  
\end{align*}
Then, we have 
\begin{align}
&C(\gamma;\lambda,\mu,\nu)\nonumber\\
&=\int_{(X\times Y)^2}(L(d_{X}^q(x,x'),d_{Y}^q(y,y'))-2\lambda)
 \, d\gamma^{\otimes 2} +\underbrace{\lambda(|\mu|+|\nu|)}_{\text{does not depend on $\gamma$}}\nonumber   \\
 &=\int_{(X\times Y)^2}(L(d_{X}^q(x,x'),d_{Y}^q(y,y'))-2\lambda)
 \, d\hat\gamma^{\otimes 2} +{\lambda(|\mu|+|\nu|)}\nonumber \qquad (\text{since } \hat\gamma|_{X\times Y}=\gamma)\\
&=\int_{(X\times Y)^2}(L(d_{\hat X}^q(x,x'),d_{\hat Y}^q(y,y'))-2\lambda) \, d\hat\gamma^{\otimes 2} +\lambda(|\mu|+|\nu|)\nonumber \quad (\text{as } d_{\hat X}|_{X\times X}=d_X, \, d_{\hat Y}|_{Y\times Y}=d_Y)\\
&=\int_{(X\times Y)^2} L_{\lambda}(d_{\hat X}^q(x,x'),d_{\hat Y}^q(y,y')) \, d\hat\gamma^{\otimes 2} +\lambda(|\mu|+|\nu|)\nonumber \quad (\text{since } \hat{L}|_{\mathbb R \times \mathbb R}(\cdot,\cdot)=(L(\cdot,\cdot)-2\lambda))\\
&=\int_{(\hat X\times \hat Y)^2} L_{\lambda}(d_{\hat X}^q(x,x'),d_{\hat Y}^q(y,y')) \, d\hat\gamma^{\otimes 2} +\underbrace{\lambda(|\mu|+|\nu|)}_{\text{does not depend on } \hat \gamma}.\nonumber \quad (\text{since } \hat L \text{ assigns } 0 \text{ to } \hat\infty)
\end{align}
Combining this with the fact that $F: \gamma\mapsto \hat\gamma$ is a bijection, we have that $\gamma$ is optimal for \eqref{eq:pgw} if and only if $\hat\gamma$ is optimal for \eqref{eq: GW variant}.
Under the assumptions of Proposition \ref{pro: pgw minimizer}, there exists an optimal $\gamma\in\Gamma_{\leq}(\mu,\nu)$ for the PGW problem exists, and so
we have: 
\begin{equation}\label{eq: cost PGW and GW}
  \arg\min_{\hat\gamma\in \Gamma(\hat \mu,\hat\nu)}\hat{C}(\hat{\gamma};\mu,\nu)=\arg\min_{\gamma\in\Gamma_{\leq}(\mu,\nu)} C(\gamma;\lambda,\mu,\nu).  
\end{equation}
\end{proof}



\begin{remark}
Both algorithms (Algorithm 1 and 2) are mathematically and computationally equivalent, owing to the equivalence between the POT problem in Solver 1 and the OT problem in Solver 2.
\end{remark}

\subsection{Frank-Wolfe for the PGW Problem -- Solver 2}\label{subsec: alg 2}
Similarly to the discrete PGW problem \eqref{eq:pgw_discrete}, consider the discrete version of \eqref{eq: mu_hat}: \begin{equation}\label{eq: hat p hat q}
  \hat{\mathrm{p}}=[\mathrm{p};|\mathrm{q}|]\in \mathbb{R}^{n+1}, \quad \hat{\mathrm{q}}=[\mathrm{q};|\mathrm{p}|]\in\mathbb{R}^{m+1},
 \end{equation}
and, in a similar fashion, we define $\hat M\in\mathbb{R}^{(n+1)\times (m+1)\times(n+1)\times (m+1)}$ as
\begin{align}
 \hat M_{i,j,i',j'}=\begin{cases}
     \tilde{M}_{i,j,i',j'} &\text{if }i,i'\in[1:n],j,j'\in[1:m],\\
     0  &\text{elsewhere}.
 \end{cases}\label{eq: M_hat}
 \end{align}    
Then, the GW-variant problem \eqref{eq: GW variant} can be written as 
\begin{align}
\widehat{GW}(\hat{\mathbb{X}},\hat{\mathbb{Y}})=\min_{\hat\gamma\in \Gamma(\hat{\mathrm{p}},\hat{\mathrm{q}})} \mathcal{L}_{\hat M}(\hat \gamma). \label{eq: variant gw discrete}
\end{align}
Based on Proposition \ref{pro: pgw and gw} (which relates $PGW_\lambda^L(\cdot,\cdot)$ with $\widehat{GW}(\cdot,\cdot)$),
we propose two versions of the Frank-Wolfe algorithm \cite{frank1956algorithm} that can solve the PGW problem \eqref{eq:pgw_discrete}. Apart from Algorithm 1 in \cite{chapel2020partial}, which solves a different formulation of partial GW, and Algorithm 1 in \cite{sejourne2021unbalanced}, which applies the Sinkhorn algorithm to solve an entropic regularized version of \eqref{eq:ugw}, to the best of our knowledge, a precise computational method for the discrete PGW problem \eqref{eq:pgw_discrete} has not been studied.

Here, we discuss another version of the FW Algorithm for solving the PGW problem \eqref{eq:pgw_discrete}.
The main idea relies on solving first the GW-variant problem \eqref{eq: GW variant}, and, at the end of the iterations, by using Proposition \ref{pro: pgw and gw},  convert the solution of the GW-variant problem to a solution for the original partial GW problem \eqref{eq:pgw_discrete}. 

First, construct $\hat{\mathrm{p}},\hat{\mathrm{q}},\hat{\mathrm{M}}$ as described in Proposition \ref{pro: pgw and gw}. Then, for each iteration $k$, perform the following three steps.

\textbf{Step 1: Computation of gradient and optimal direction}.
Solve the OT problem: 
\begin{align}
  \hat\gamma^{(k)'}&\gets \arg\min_{\hat\gamma\in\Gamma(\hat{\mathrm{p}},\hat{\mathrm{q}})} \langle\mathcal{L}_{\hat M}(\hat{\gamma}^{(k)}),\hat\gamma\rangle_F.\nonumber  
\end{align}
The gradient $\mathcal{L}_{\hat M}(\gamma^{(k)})$ can be computed in a similar way as described in Lemma \ref{lem: grad M_tilde}. We refer to Section \ref{sec:gradient} for details.

\textbf{Step 2: Line search method}. 
Find optimal step size $\alpha^{(k)}$: 
$$\alpha^{(k)}=\arg\min_{\alpha\in[0,1]}\{\mathcal{L}_{\hat{M}}((1-\alpha)\hat\gamma^{(k)}+\alpha\hat\gamma^{(k)'})\}.$$
Similar to Solver 1, let 
\begin{align}
\begin{cases}
&\delta\hat\gamma^{(k)}=\hat\gamma^{(k)'}-\hat\gamma^{(k)},\\
&a=\langle \hat M\circ \delta\hat\gamma^{(k)},\delta\hat\gamma^{(k)} \rangle_F,\\
&b=2\langle \hat M \circ \delta \hat \gamma^{(k)},\hat \gamma^{(k)} \rangle_F .   
\end{cases}
\label{eq: a,b,v2}
\end{align}
Then the optimal $\alpha^{(k)}$ is given by formula  \eqref{eq:a,b,v1}. 
See Appendix \ref{sec:line_search_v2} for a detailed discussion. 

\textbf{Step 3}. Update $\hat\gamma^{(k+1)}\gets (1-\alpha^{(k)})\hat\gamma^{(k)}+\alpha^{(k)}\hat\gamma^{(k)'}$.

\begin{algorithm}[tb]
   \caption{Frank-Wolfe Algorithm for partial GW, ver 2}
   \label{alg:pgw_v2}
\begin{algorithmic}
   \STATE {\bfseries Input:}
   $\mu=\sum_{i=1}^np_i^X\delta_{x_i}, \nu=\sum_{j=1}^mq_j^Y\delta_{y_j},\gamma^{(1)}$
   \STATE {\bfseries Output:}
   $\gamma^{(final)}$
   \STATE Compute $C^{X},C^{Y},\hat{\mathrm{p}},\hat{\mathrm{q}},\hat{\gamma}^{(1)}$\\
   \FOR{$k=1,2,\ldots$}
   \STATE 
   $\hat{G}^{(k)}\gets 2\hat{M}\circ \hat{\gamma}^{(k)}$ // Compute gradient 
   \STATE $\hat{\gamma}^{(k)'}\gets \arg\min_{\hat\gamma\in \Gamma(\hat{\mathrm{p}},\hat{\mathrm{q}})}\langle\hat{G}^{(k)}, \hat\gamma\rangle_F$ // Solve the OT problem \\ 
   \STATE Compute $\alpha^{(k)}\in[0,1]$ via \eqref{eq: a,b,v2}, \eqref{eq:a,b,v1} // Line Search
   \STATE $\hat\gamma^{(k+1)}\gets (1-\alpha^{(k)})\hat\gamma^{(k)'}+\alpha \hat\gamma^{(k)}$// Update $\hat\gamma$ 
   \STATE if convergence, break
   \ENDFOR
 \STATE $\gamma^{(final)}\gets\hat\gamma^{(k)}[1:n,1:m]$
\end{algorithmic}
\end{algorithm}

\section{Gradient Computation in Algorithms 1 and 2}\label{sec:gradient}
In this section, we discuss the computation of Gradient $\nabla \mathcal{L}_{\tilde M}(\gamma)$ in Algorithm 1 and $\nabla \mathcal{L}_{\hat {M}}(\hat\gamma)$ in Algorithm 2. 

\begin{proposition}[Proposition 1 \cite{peyre2016gromov}]\label{pro: tensor product gw}
If the cost function can be written as 
\begin{equation}
  L(r_1,r_2)=f_1(r_1)+f_2(r_2)-h_1(r_1)h_2(r_2) \label{eq:L_cond} 
\end{equation}
then 
\begin{equation}
  M\circ \gamma=u(C^X,C^Y,\gamma)-h_1(C^X)\gamma h_2(C^Y)^\top, \label{eq: tensor dot}
\end{equation}
where $u(C^X,C^Y,\gamma):=f_1(C^X)\gamma_1 1_{m}^\top+1_{n}\gamma_2^\top f_2(C^Y)$.
\end{proposition}

Additionally, the following lemma builds the connection between $\tilde M\circ\gamma$ and $M\circ \gamma$. 
\begin{lemma}\label{lem: grad M_tilde}
For any $\gamma\in \mathbb{R}^{n\times m}$, we have: 
\begin{equation}
\tilde{M}\circ \gamma= M\circ \gamma -2\lambda |\gamma|1_{n,m}\label{eq: tensor dot v1}.
\end{equation}
\end{lemma}
 
\begin{proof}
For any $\gamma\in\mathbb{R}^{n\times m}$, we have 
\begin{align}
\tilde{M}\circ \gamma&=(M 1_{n,n,m,m}-2\lambda)\circ\gamma\nonumber\\
&=(M-2\lambda 1_{n,n,m,m})\circ \gamma \nonumber\\
&=M\circ\gamma-2\lambda 1_{n,m,n,m}\circ \gamma\nonumber\\
&=M\circ \gamma-2 (\langle 1_{n,m},\gamma\rangle_F)1_{n,m}\nonumber\\
&=M\circ \gamma-2\lambda|\gamma|1_{n,m}\nonumber
\end{align}
where the second equality follows from Lemma \ref{lem: tensor prod}.
\end{proof}



Next, in the setting of Algorithm 2, for any $\hat\gamma\in \mathbb{R}^{(n+1)\times(m+1)}$, we have 
\begin{align}
\nabla \mathcal{L}_{\hat M}(\hat\gamma)=2 \hat M\circ \hat \gamma
\end{align}
and $\hat M\circ \hat \gamma$ can be computed by the following lemma. 
\begin{lemma}\label{lem: grad M_hat}
For each $\hat\gamma\in \mathbb{R}^{(n+1)\times (m+1)}$, we have
$\hat M\circ \hat \gamma\in \mathbb{R}^{(n+1)\times (m+1)}$ with the following: 
\begin{align}
(\hat M\circ \hat\gamma)_{ij}=\begin{cases}
    (\tilde M\circ \hat\gamma[1:n,1:m])_{ij} &\text{if }i\in[1:n], j\in[1:m]\\
    0&\text{elsewhere} 
\end{cases}.
\end{align}
\end{lemma}
\begin{proof}
Recall the definition of $\hat M$ is given by \eqref{eq: M_hat}, choose $i\in[1:n],j\in[1:m]$, we have 
\begin{align}
(\hat M\circ \hat \gamma)_{ij}&=\sum_{i'=1}^n\sum_{j'=1}^m \hat{M}_{i,j,i',j'}\hat\gamma_{i',j'}+\sum_{j'=1}^{m}\hat M_{i,j,n+1,j}\hat\gamma_{n+1,j'}+\sum_{i'=1}^n\hat M_{i,j,i',m+1}\hat\gamma_{i,m+1}\nonumber\\
&\qquad +\hat{M}_{i,j,n+1,m+1}\hat\gamma_{n+1,m+1}\nonumber\\
&=\sum_{i'=1}^n\sum_{j'=1}^m\hat{M}_{i,j,i',j'}\hat\gamma_{i',j'}+0+0+0\nonumber=\sum_{i'=1}^n\sum_{j'=1}^m\tilde{M}_{i,j,i',j'}\hat \gamma_{i',j'}\nonumber\\
&=(\tilde M\circ (\hat\gamma[1:n, 1:m]))_{ij}\nonumber 
\end{align}
If $i=n+1$, we have 
\begin{equation}
(\hat M\circ \hat\gamma)_{n+1,j}=\sum_{i'=1}^{n+1}\sum_{j'=1}^{m+1}\hat M_{n+1,j,i',j'}\hat \gamma_{i',j'}=0 \nonumber 
\end{equation}
Similarly, $(\hat M\circ \hat\gamma)_{i,m+1}=0$. 
Thus, we complete the proof. 
\end{proof}
\section{Line Search in Algorithm 1}\label{sec:line_search_v1}
In this section, we discuss the derivation of the line search algorithm. 

We observe that in the partial GW setting, for each $\gamma\in\Gamma_\leq(\mu,\nu)$, the marginals of $\gamma$ are not fixed. Thus, we can not directly apply the classical algorithm (e.g., \cite{titouan2019optimal}).

In iteration $k$, let $\gamma^{(k)},\gamma^{(k)'}$ be the previous and new transportation plans from step 1 of the algorithm. For convenience, we denote them as $\gamma$ and $\gamma'$, respectively.

The goal is to solve the following problem: 
\begin{align}
\min_{\alpha\in[0,1]}\mathcal{L}(\tilde{M},(1-\alpha)\gamma+\alpha \gamma')\label{eq: line search goal}
\end{align}
where 
$\mathcal{L}(\tilde M,\gamma)=\langle \tilde{M}\circ \gamma,\gamma\rangle_F$. By denoting
$\delta\gamma=\gamma'-\gamma$, we have  
$$
\mathcal{L}(\tilde{M},(1-\alpha)\gamma+\alpha \gamma')=\mathcal{L}(\tilde{M},\gamma+\alpha \delta\gamma).$$
Then, 
\begin{align*}
&\langle\tilde{M}\circ (\gamma+\alpha\delta\gamma),(\gamma+\alpha\delta\gamma)\rangle_F\nonumber\\
&=\langle \tilde{M}\circ \gamma,\gamma\rangle_F+\alpha\left(\langle\tilde{M}\circ \gamma,\delta\gamma\rangle_F+\langle \tilde{M}\circ \delta\gamma,\gamma\rangle_F\right)+\alpha^2\langle \tilde{M}\circ\delta\gamma,\delta\gamma\rangle_F
\end{align*}
Let  
\begin{align}
 a=&\langle \tilde{M}\circ\delta\gamma,\delta\gamma\rangle_F,\nonumber\\
 b=&\langle\tilde{M}\circ \gamma,\delta\gamma\rangle_F+\langle \tilde{M}\circ \delta\gamma,\gamma\rangle_F=2\langle \tilde{M}\circ \gamma,\delta\gamma\rangle_F, \label{pf: b} \\ c=&\langle\tilde{M}\circ \gamma,\gamma\rangle_F,\notag
\end{align}
where the second identity in \eqref{pf: b} follows from Lemma \ref{lem: tensor prod} and the fact that $\tilde{M}=M 1_{n,n,m,m}-2\lambda 1_{n,m,n,m}$ is symmetric. 

Therefore, the above problem \eqref{eq: line search goal} becomes 
$$\min_{\alpha\in[0,1]}a\alpha^2+b\alpha+c.$$
The solution is the following: 
\begin{align}
\alpha^*=
\begin{cases}
1 &\text{if }a\leq 0, a+b\leq  0,\\ 
0 &\text{if }a\leq 0, a+b> 0,\\
\text{clip}(\frac{-b}{2a},[0,1]) &\text{if }a>0, 
\end{cases}\label{eq: optimal alpha 2}   
\end{align}
where
$$\text{clip}(\frac{-b}{2a},[0,1])=\min\left\{1, \max\{0,\frac{-b} {2a}\}\right\}=\begin{cases}
    \frac{-b}{2a} &\text{if }\frac{-b}{2a}\in[0,1],\nonumber \\
    0 &\text{if }\frac{-b}{2a}<0, \nonumber\\ 
    1 &\text{if }\frac{-b}{2a}>1. 
\end{cases}$$

We can further discuss the difference in computation of $a$ and $b$ in the PGW setting and the classical GW setting. 
If the assumption in Proposition \ref{pro: tensor product gw} holds, by \eqref{eq: tensor dot} and \eqref{eq: tensor dot v1}, we have

\begin{align}
a&=\langle \tilde{M}\circ \delta\gamma,\delta\gamma\rangle_F\nonumber\\
&=\langle (M\circ\delta\gamma-2\lambda|\delta\gamma|I_{n,m}),
\delta\gamma\rangle_F\nonumber\\
&=\langle M\circ \delta\gamma,\delta\gamma\rangle_F-2\lambda |\delta\gamma|^2\label{eq: a compute v1}\\
&=\left\langle u(C^X,C^Y,\delta\gamma)-h_1(C^X)\delta\gamma h_2(C^Y)^\top, \delta\gamma\right\rangle_F-2\lambda|\delta\gamma|^2, \nonumber\\
b&=2\langle \tilde{M}\circ \gamma,\delta\gamma \rangle_F\nonumber\\
&=2\langle M\circ \gamma-2\lambda|\gamma|I_{n,m},\delta\gamma\rangle\nonumber\\
&=2(\langle M\circ \gamma,\delta\gamma\rangle_F-2\lambda |\delta\gamma||\gamma|)
\label{eq: b compute v1}
 \end{align}
 
Note that in the classical GW setting \cite{titouan2019optimal}, the term $u(C^X,C^Y,\delta\gamma)=0_{n\times m}$ and $|\delta \gamma|=0$. Therefore, in such line search algorithm (Algorithm 2 in \cite{titouan2019optimal}), the terms $u(C^X, C^Y,\delta\gamma),2\lambda |\delta\gamma|1_{n\times m}$ are not required. In addition, in equation \eqref{eq: b compute v1}, $M\circ \gamma,2\lambda|\gamma|$ have been computed in the gradient computation step. Thus, these two terms can be directly applied in this step.

\section{Line Search in Algorithm 2}\label{sec:line_search_v2}
Similar to the previous section, in iteration $k$, let $\hat\gamma^{(k)},\hat\gamma^{(k)'}$ denote the previous transportation plan and the updated transportation plan. For convenience, we denote them as $\hat\gamma,\hat\gamma'$, respectively. 

Let $\delta\hat\gamma=\hat\gamma-\hat\gamma'$. 

The goal is to find the following optimal $\alpha$: 
\begin{align}
\alpha=\arg\min_{\alpha\in[0,1]}\mathcal{L}(\hat M, (1-\alpha)\hat\gamma,\alpha\hat\gamma')=\arg\min_{\alpha\in[0,1]}\mathcal{L}(\hat M, \alpha \delta \hat\gamma+\hat\gamma),
\end{align}
where $\hat{M}\in \mathbb{R}^{(n+1)\times(m+1)\times (n+1)\times(m+1)}$, with $\hat{M}[1:n,1:m,1:n,1:m]=\tilde{M}=M-2\lambda 1_{n\times m\times n\times m}$.

Similar to the previous section,
let 
\begin{align}
a&=\langle \hat M\circ \delta \hat\gamma,\delta \hat\gamma\rangle_F, \nonumber\\
b&=\langle \hat M\circ \delta \hat\gamma,\hat \gamma\rangle_F+\langle \hat M\circ  \hat\gamma,\delta\hat \gamma\rangle_F=2\langle \hat M\circ \delta \hat\gamma,\hat \gamma\rangle_F,\label{pf: b v2}\\
c&=\langle \hat M\circ \hat\gamma,\hat \gamma\rangle_F,\notag
\end{align}
where $\eqref{pf: b v2}$ holds since $\hat{M}$ is symmetric. 
Then, the optimal $\alpha$ is given by \eqref{eq: optimal alpha 2}. 

It remains to discuss the computation. By Lemma \ref{lem: tensor prod}, we set $\gamma=\hat\gamma[1:n,1:m],\delta\gamma=\delta\hat\gamma[1:n,1:m]$. Then,  
\begin{align}
    a&=\langle (\hat M\circ \delta \hat\gamma)[1:n,1:m],\delta\gamma\rangle_F=\langle ( \tilde M\circ \delta \gamma,\delta\gamma\rangle_F,\nonumber\\ 
    b&=\langle (\hat M\circ \delta \hat\gamma)[1:n,1:m],\gamma\rangle_F =\langle ( \tilde M\circ \delta \gamma,\gamma\rangle_F.\nonumber
\end{align}
Thus, we can apply \eqref{eq: a compute v1}, \eqref{eq: b compute v1} to compute $a,b$ in this setting by plugging in $\gamma=\hat\gamma[1:n,1:m]$ and $\delta\gamma=\delta\hat\gamma[1:n,1:m]$.

\section{Convergence}\label{sec: convergence}

As in \cite{chapel2020partial}, we will use the results from \cite{lacoste2016convergence} on the convergence of the Frank-Wolfe algorithm for non-convex objective functions.

Consider the minimization problems
\begin{equation}\label{eq: min prob appendix}
\min_{\gamma\in\Gamma_{\leq}(\mathrm{p},\mathrm{q})}\mathcal{L}_{\tilde M}(\gamma) \qquad \text{ and } \qquad \min_{\hat\gamma\in \Gamma(\hat{\mathrm{p}}, \hat{\mathrm{q}})} \mathcal{L}_{\hat M}(\hat\gamma)
\end{equation}
that corresponds to the discrete partial GW problem and the discrete GW-variant problem (used in version 2), respectively. 

Consider the objective functions 
$\gamma \mapsto \mathcal{L}_{\tilde{M}}(\gamma) = \tilde{M}\,\gamma^{\otimes 2}$ 
and 
$\hat{\gamma} \mapsto \mathcal{L}_{\hat{M}}(\hat{\gamma}) = \hat{M}\,\hat{\gamma}^{\otimes 2}$,
where $\tilde{M} = M - 2\lambda\,1_{n,m}$ for a fixed matrix $M \in \mathbb{R}^{n \times m}$ with $\lambda > 0$, and $\hat{M}$ is given by \eqref{eq: M_hat}. 
Although $\tilde{M}$ and $\hat{M}$ are symmetric for $\lambda > 0$, they are not positive semi-definite, so these objective functions are generally non-convex.

However, the constraint sets $\Gamma_{\leq}(\mathrm{p}, \mathrm{q})$ and $\Gamma(\hat{\mathrm{p}}, \hat{\mathrm{q}})$ are convex and compact in $\mathbb{R}^{n \times m}$ (see Proposition B.2 in \cite{liu2023ptlp}) and in $\mathbb{R}^{(n+1) \times (m+1)}$, respectively.

From now on, we will concentrate on the first minimization problem in \eqref{eq: min prob appendix}, and the convergence analysis for the second one will be analogous.  

Consider the \textit{Frank-Wolfe gap} of $\mathcal{L}_{\tilde M}$ at the approximation $\gamma^{(k)}$ of the optimal plan $ \gamma$:
\begin{equation}\label{eq:g_k}
    g_k=\min_{\gamma\in\Gamma_{\leq}({\mathrm{p}}, {\mathrm{q}})} \langle \nabla\mathcal{L}_{\tilde M}(\gamma^{(k)}),\gamma^{(k)}-\gamma\rangle_F.
\end{equation}
It provided a good criterion to measure the distance to a stationary point at iteration $k$. Indeed, a plan $ \gamma^{(k)}$ is a stationary transportation plan for the corresponding constrained optimization problem in \eqref{eq: min prob appendix} if and only if $g_k=0$. Moreover, $g_k$ is always non-negative ($g_k\geq 0$). 

From Theorem 1 in \cite{lacoste2016convergence}, after $K$ iterations, we have the following upper bound for the minimal Frank-Wolf gap:
\begin{equation}
g_K:=\min_{1\leq k\leq K}g_k\leq \frac{\max\{2L_1,\mathrm{Lip} \cdot  (\text{diam}(\Gamma_{\leq}({\mathrm{p}}, {\mathrm{q}})))^2\}}{\sqrt{K}}, \label{eq:g_k_bound}
\end{equation}
where $$L_1:=\mathcal{L}_{\tilde M}(\gamma^{(1)})-\min_{\gamma\in\Gamma_{\leq}({\mathrm{p}}, {\mathrm{q}})}\mathcal{L}_{\tilde M}( \gamma)$$ is the initial global sub-optimal bound for the initialization $ \gamma^{(1)}$ of the algorithm; $\text{Lip}$ is the Lipschitz constant of function 
$\gamma\mapsto\nabla\mathcal{L}_{\tilde M}$; and 
$$\text{diam}(\Gamma_{\leq}({\mathrm{p}}, {\mathrm{q}}))=\sup_{\gamma,\gamma'\in\Gamma_\leq(\mu,\nu)}\|\gamma-\gamma'\|_F$$ is the $\|\cdot\|_F$ diameter of $\Gamma_{\leq}({\mathrm{p}}, {\mathrm{q}})$ in $\mathbb{R}^{n\times m}$. 

The important thing to notice is that the constant $\max\{2L_1,D_L\}$ does not depend on the iteration step $k$. Thus, according to Theorem 1 in \cite{lacoste2016convergence}, the rate in $\tilde g_K$ is $\mathcal{O}(1/\sqrt{K})$. That is, the algorithm takes at most $\mathcal{O}(1/\varepsilon^2)$ iterations to find an approximate stationary point with a gap smaller than $\varepsilon$.

\begin{lemma}\label{lem:diam_Gamma_bound}
In the discrete PGW problem, we have 
\begin{align}
\text{diam}(\Gamma_\leq(p,q))\leq 2s:=2 \min(|p|,|q|)\label{eq:diam_Gamma_bound}
\end{align}
\end{lemma}
\begin{proof}
Choose $\gamma,\gamma'\in \Gamma_\leq(p,q)$, 
we apply the property
\begin{align}
(a-b)^2\leq 2a^2+2b^2, \forall a,b\in \mathbb{R} \label{eq:ineq_ab}
\end{align} and obtain 
\begin{align}
\|\gamma-\gamma'\|_{F}^2&=\sum_{i,j}^{n,m}|\gamma_{i,j}-\gamma'_{i,j}|^2\nonumber\\
&\leq\sum_{i,j}^{n,m}2|\gamma_{i,j}|^2+2|\gamma'_{i,j}|^2\nonumber\\
&\leq 2\left[(\sum_{i,j}\gamma_{i,j})^2+(\sum_{i,j}\gamma'_{i,j})^2\right]\nonumber\\
&= 2(|\gamma|^2+|\gamma'|^2)\\
&\leq 2 \min(|p|,|q|)^2+2\min(|p|,|q|)^2\nonumber\\
&=4\min(|p|,|q|)^2 \nonumber 
\end{align}
and thus, we complete the proof. 
\end{proof}
\begin{lemma}\label{lem:lip_bound}
For the Lip term in \eqref{eq:g_k_bound} can be bounded as following:
\begin{align}
\text{Lip}\leq nm\max_{i,i,j,j'}(2(C^X_{i,i'})^2+2(C^Y_{j,j'})^2,2\lambda)\label{eq:lip_bound}
\end{align}
\end{lemma}
\begin{proof}
Pick $\gamma,\gamma'\in\Gamma_{\leq}({\mathrm{p}}, {\mathrm{q}})$ we have,
\begin{align*}
&\|\nabla\mathcal{L}_{\tilde M}(\gamma)-\nabla\mathcal{L}_{\tilde M}(\gamma')\|_F^2\nonumber\\
&=\|\tilde M\circ \gamma - \tilde M\circ\gamma'\|_F^2\\
&=\|[M -2\lambda ]\circ(\gamma-\gamma')\|_F^2\nonumber\\
&=\sum_{i,j}\left(\left[(M-2\lambda)\circ(\gamma-\gamma')\right]_{i,j}\right)^2\\
&=\sum_{i,j}\left(\sum_{i',j'}(M_{i,j,i',j'}-2\lambda)(\gamma_{i',j'}-\gamma'_{i',j'})\right)^2\\
&\leq \sum_{i,j}\left(\sum_{i',j'}|M_{i,j,i',j'}-2\lambda||\gamma_{i',j'}-\gamma'_{i',j'}|\right)^2 \nonumber \\
&\leq\underbrace{\left(\max_{i,j,i',j'}\{M_{i,j,i',j'}-2\lambda\}\right)^2}_{A}\cdot
\underbrace{\sum_{i,j}^{n,m}\left(\sum_{i',j'}^{n,m}|\gamma_{i',j'}-\gamma'_{i',j'}|\right)^2}_B
\end{align*}

For the first term, we have: 
\begin{align*}
A&\leq \max\{2(C^X)^2+2(C^Y)^2,2\lambda\}^2
\end{align*}
where 
$$\max\{2(C^X)^2+2(C^Y)^2,2\lambda\}:=\max\{\max_{i,i',j',j'}2(C^X_{i,i'})^2+2(C^Y_{j,j'})^2,2\lambda\}$$

For the second term, we have: 
\begin{align*}
&\sum_{i,j}^{n,m}\left(\sum_{i',j'}^{n,m}|\gamma_{i',j'}-\gamma'_{i',j'}|\right)^2\\
&\leq \sum_{i,j}^{n,m}\left(nm\sum_{i',j'}^{n,m}\left|\gamma_{i',j'}-\gamma'_{i',j'}\right|^2\right)\nonumber\\
&\leq n^2m^2 \|\gamma-\gamma'\|^2_F
\end{align*}
Thus we obtain
\begin{align}
\text{Lip}\leq \frac{\max(2(C^X)^2+2(C^Y)^2,2\lambda)nm\|\gamma-\gamma'\|_F}{\|\gamma-\gamma'\|_F}=nm\max(2(C^X)^2+2(C^Y)^2,2\lambda) \nonumber 
\end{align}
and we complete the proof. 
\end{proof}

Combined the above two lemmas, we derive the convergence rate of the Frank-Wolf gap \eqref{eq:g_k}: 
\begin{proposition}\label{pro:gap_bound_pgw}
When $L(r_1,r_2)=|r_1-r_2|^2$ in the PGW problem, the Frank-Wolfe gap of algorithm \ref{alg:pgw_v1}, defined in \eqref{eq:g_k} at iteration $k$ satisfies the following: 
\begin{align}
g_k \leq \frac{\max\Bigg\{2L_1,4\min(|\mathrm p|,|\mathrm q|)^2\cdot nm(\max\{2(C^X)^2+2(C^Y)^2,2\lambda\})\Bigg\}}{\sqrt{k}}\label{eq:g_k_bound_pgw}
\end{align}
\end{proposition}
\begin{proof}
The proof directly follows from the upper bounds \eqref{eq:diam_Gamma_bound},\eqref{eq:lip_bound} and the inequality \eqref{eq:g_k_bound}.
\end{proof}
\begin{remark}
Note, if the cost function in PGW is defined by $|r_1-r_2|^p$ for some $p\neq 2$, it is straightforward to verify that the upper bound of $g_k$ is obtained by replacing the term $\max ((C^X)^2+(C^Y)^2,2\lambda)$ should be replaced by 
$$\max_{i,j,i',j'}\left[2^{p-1}((C^X)^p+(C^Y)^p),2\lambda\right].$$
\end{remark}
\begin{remark}
It is straightforward to verify that the Frank-Wolf gap for algorithm 2 is also upper bounded by \eqref{eq:g_k_bound_pgw}.
\end{remark}
\begin{remark}
From the proposition \eqref{pro:gap_bound_pgw}, to achieve an $\epsilon-$accurate solution, the required number of iterations is 
\begin{align}
\frac{\max\Bigg\{2L_1,2\min(|\mathrm p|,|\mathrm q|)\cdot n^2m^2\max(\{2(C^X)^2+2(C^Y)^2,2\lambda\})\Bigg\}^2}{\epsilon^2}\nonumber
\end{align}
In practice, by lemma \ref{lem:large_lambda}, when $\lambda$ is large, it is equivalent to set $\lambda=\max((C^X)^2-(C^Y)^2)$ and thus, the value $\lambda$ will not affect the number of iterations. 
\end{remark}
\begin{remark}
In this remark, we compare the upper \eqref{eq:gamma_pgw_gw} and the upper bounds in previous works, including \cite{chapel2020partial} and {cang2024supervised}.  We observe that: 

If we ignore the term $\lambda$ (or equivalently, set $\lambda=\max_{i,i',j,j'}(2(C^X)^2+2(C^Y)^2)$), the upper bound \eqref{eq:g_k_bound_pgw} aligns with the upper bound for the Frank-Wolfe Gap in Mass-constraint PGW (see \cite[Eq. (10)]{chapel2020partial}) and balanced GW (see \cite[p. 18]{cang2024supervised}.

However, there exist several differences in details between our result and the result in \cite{chapel2020partial}. 

\begin{itemize}
    \item First, in \cite{chapel2020partial}, the term $nm$ is omitted. We suspect this omission is a typo, as the $nm$ term is contained in \cite{cang2024supervised}. 
    \item In \cite{chapel2020partial}, ``$s$'' plays the role of $\min(|\mathrm p|,|\mathrm q|)$ in our result. 
    The assumption in \cite[Lemma 1]{chapel2020partial} is $$|\gamma|=|\gamma'|=1$$ 
    If we normalize $\mathrm p,\mathrm q$ such that $\max(|\mathrm p|,|\mathrm q|)\leq 1$, 
we will obtain 
$$2 \min(|\mathrm p|,|\mathrm q|)^2\leq 2 \min(|\mathrm p|,|\mathrm q|),$$
and our result will induce the upper bound in \cite[Lemma 1]{chapel2020partial}. 
\end{itemize}

\end{remark}

\section{Related Work: Mass-Constrained Partial Gromov-Wasserstein}\label{sec: related work chapel}
Partial Gromov-Wasserstein is first introduced in \cite{chapel2020partial}. To distinguish the PGW problem in \cite{chapel2020partial} and the PGW problem in this paper, we call the former one the Mass-Constrained Gromov-Wasserstein problem (MPGW): 
\begin{equation}
MPGW_{\rho}(\mathbb{X},\mathbb{Y}):=\inf_{\gamma\in\Gamma_\leq^\rho(\mu,\nu)}\gamma^{\otimes2}(L(d_X^q,d_Y^q))\label{eq:mpgw},
\end{equation}
where $\rho\in [0,\min\{|\mu|,|\nu|\}]$, and 
\begin{equation}\label{eq:mpgw plans}
    \Gamma_\leq^\rho(\mu,\nu):=\{\gamma\in\mathcal{M}_+(X\times Y): \gamma_1\leq \mu, \, \gamma_2\leq \nu, \,  \, |\gamma|=\rho\}. 
\end{equation}
Unlike the relationship between Partial OT and OT, it is not rigorous to claim that PGW and MPGW are equivalent. The objective function
\begin{equation}\label{eq: costt}
  \gamma \;\mapsto\; \int_{(X\times Y)^2} L\bigl(d_X^2(x,x'),\,d_Y^2(y,y')\bigr)\,d\gamma^{\otimes2}
\end{equation}
is non-convex, even if the map $(r_1, r_2)\mapsto L(r_1,r_2)$ itself is convex \cite{peyre2019computational}. 
If the problem were convex, then MPGW, viewed as the `Lagrangian formulation' of PGW (adding the partial constraint into the functional via Lagrange multipliers), would coincide with PGW. However, since these problems are not convex, we cannot assert their equivalence in principle.

We can still investigate their relationship by the following lemma, based on which we design the wall-clock time experiment in Section \ref{sec:wall_time}. 

\begin{proposition}\label{pro:pgw_mpgw}
For each $\lambda\ge 0$, there exists $\rho\in [0, \min(|\mu|,|\nu|)]$ such that, for each $\gamma\in\Gamma_\leq(\mu,\nu)$ with $|\gamma|=\rho$, $\gamma$ is optimal in $PGW_\lambda(\mathbb{X},\mathbb{Y})$ iff $\gamma$ is optimal in $MPGW_\rho(\mathbb{X},\mathbb{Y})$. Furthermore, 

$$PGW_\lambda(\mu,\nu)=MPGW_{\rho}(\mu,\nu)+\lambda(|\mu|^2+|\nu|^2-2\rho^2).$$

\end{proposition}
\begin{proof}
Pick $\gamma'\in\Gamma^\rho_\leq(\mu,\nu)\subset \Gamma_\leq(\mu,\nu)$, since $\gamma$ is optimal in $PGW_\lambda(\mu,\nu)$, we have 
\begin{align}
0&\leq C(\gamma;\lambda,\mu,\nu)-C(\gamma';\lambda,\mu,\nu)\nonumber\\
&=\int_{(X\times Y)^2}L(d_X^2(x,x'),d_Y^2(y,y'))d(\gamma^{\otimes2}-\gamma'^{\otimes2})\nonumber
\end{align}
Thus, $\gamma$ is optimal in $\Gamma^\rho_\leq(\mu,\nu)$ for $MPGW_\rho(\mathbb{X},\mathbb{Y})$ and we complete the proof. 
\end{proof}

\begin{remark}
The above proposition clarifies a connection between $PGW$ and $MPGW$. Specifically, it implies that an optimal transportation plan for $PGW$ can be derived from $MPGW$ by appropriately setting $\rho$. 
\textbf{However, it is important to note that MPGW and PGW admit distinct transportation costs. MPGW does not define a metric, whereas PGW does.} 
Therefore, in experiments that require transportation costs, such as the shape retrieval experiment described in the main text, the results produced by the two methods will differ significantly.
\end{remark}

\begin{examplee}\label{ex:mpgw}
    Consider the following  three mm-spaces
$$\mathbb{X}_1=\left(\mathbb{R}^3,\|\cdot\|,\sum_{i=1}^{1000}\alpha\delta_{x_i}\right), \quad 
    \mathbb{X}_2=\left(\mathbb{R}^3,\|\cdot\|,\sum_{i=1}^{800}\alpha\delta_{x_i}\right),
    \quad \mathbb{X}_3=\left(\mathbb{R}^3,\|\cdot\|,\sum_{i=1}^{400}\alpha\delta_{x_i}\right),$$ where $\alpha>0$ is the mass of each point. For numerical stability reasons, we set $\alpha=1/1000$. 
On the one hand, if we compute MPGW, the mass is fixed to be a value $\rho\in[0,0.4]$, since the total mass in $\mathbb{X}_3$ is $0.4$. For our experiment, we set $\rho=0.4$, and we observe: 
\begin{align*}       &MPGW_\rho(\mathbb{X}_1,\mathbb{X}_2;\rho=0.4)=MPGW_\rho(\mathbb{X}_2,\mathbb{X}_3;\rho=0.4)=MPGW_\rho(\mathbb{X}_1,\mathbb{X}_3;\rho=0.4)=0
    \end{align*}
On the other hand, if we compute our PGW, considering any $\lambda>0$ (in particular, we set $\lambda=10$), we obtain 
    \begin{align*}
       &PGW_\lambda(\mathbb{X}_1,\mathbb{X}_2; \lambda=10)=3.6\\
       &PGW_\lambda(\mathbb{X}_2,\mathbb{X}_3; \lambda=10)=4.8\\      &PGW_\lambda(\mathbb{X}_1,\mathbb{X}_3;\lambda=10)=8.4
    \end{align*}
In particular, one can verify the triangular inequality.

As a conclusion, in this example, MPGW can not describe the dissimilarity of any two datasets taken from $\{\mathbb{X}_1,\mathbb{X}_2,\mathbb{X}_3\}$. They are three distinct datasets, but MPGW returns zero for each pair. On the contrary, our PGW can measure dissimilarity. 

In addition, the discrepancy provided by our PGW formulation is consistent with the following intuitive observation: One expects the dissimilarity between $\mathbb{X}_1$ and $\mathbb{X}_3$ to be larger than the difference  $\mathbb{X}_1$ and $\mathbb{X}_2$ and than the difference between $\mathbb{X}_1$ and $\mathbb{X}_2$. This is because we are considering discrete measures, with the same mass at each point concentrated on the sets $\{x_1,\dots,x_{400}\}\subset\{x_1,\dots,x_{400}, \dots,x_{800}\}\subset \{x_1,\dots,x_{400}, \dots,x_{800},\dots,x_{1000}\}$ for the datasets $\mathbb{X}_3,\mathbb{X}_2,\mathbb{X}_1$, respectively. 
\end{examplee} 

\begin{remark}\label{ex:diff}

 Moreover, there may, in fact, exist instances of the two problems for which the solution sets are not equal for any value of the hyperparameters $\lambda$ and $\rho$; we illustrate this in the following example:
 
    Consider the mm-spaces given by 
    $$\mathbb{X}_1=\left(\mathbb{R}^{d_1},\|\cdot\|,\sum_{i=1}^n\delta_{x_i}\right) \quad \text{and} \quad \mathbb{X}_2=\left(\mathbb{R}^{d_2},\|\cdot\|,\sum_{j=1}^m\delta_{y_j}\right).$$
    If we let $\lambda = 0$, then $PGW_0(\mathbb{X}_1,\mathbb{X}_2)$ has $\delta_{(x_i,y_j)}$ as a solution for all $(i, j)$, as well as the zero measure.

    Now, we observe that for $\rho=0$, $\delta_{(x_i,y_j)}$ is not a solution of $MPGW_\rho(\mathbb{X}_1,\mathbb{X}_2)$; meanwhile, for any $\rho > 0$, the zero measure is not a solution of $MPGW_\rho(\mathbb{X}_1,\mathbb{X}_2)$.

    Hence, we see that the set of solutions between $PGW_0(\mathbb{X}_1,\mathbb{X}_2)$ and $MPGW_\rho(\mathbb{X}_1,\mathbb{X}_2)$ are distinct for all $\rho$.
\end{remark}

\begin{remark}
The other direction of the ``equivalence relation'' between PGW and MPGW given by Proposition \ref{pro:pgw_mpgw} may not hold. In particular, there exists a problem $MPGW_\rho(\mathbb{X}_1, \mathbb{X}_2)$ for some $\rho>0$, such that for all $\lambda> 0$, each solution to $MPGW_\rho(\mathbb{X}_1, \mathbb{X}_2)$ is not a solution to $PGW_\lambda(\mathbb{X}_1, \mathbb{X}_2)$, and each solution to $PGW_\lambda(\mathbb{X}_1, \mathbb{X}_2)$ is not a solution to $MPGW_\rho(\mathbb{X}_1, \mathbb{X}_2)$. 

As an example, consider $\mathbb{X}_1=\left(\mathbb{R}^{d},\|\cdot\|,\sum_{i=1}^{100}\delta_{x_i}\right)$ and $\mathbb{X}_2=\left(\mathbb{R}^{d},\|\cdot\|,\sum_{i=1}^{200}\delta_{x_i}\right)$, where $x_1,\ldots x_{200}$ are distinct. Then for each $\rho\in (0, 100]$, any solution to $MPGW_\rho(\mathbb{X}_1,\mathbb{X}_2)$ is not a solution for $PGW_\lambda(\mathbb{X}_1, \mathbb{X}_2)$ and vice versa.

\end{remark}




\section{Partial Gromov-Wasserstein Barycenter}\label{sec:barycenter}
We first introduce the classical Gromov-Wasserstein problem \cite{peyre2016gromov}:
Consider finite discrete probability measures $\mu^1,\ldots, \mu^K$, where $\mu^k=\sum_{i=1}^{n_k}p^k_{i}\delta_{x^k_i}$ and each $x^k_i\in \mathbb{R}^{d_k}$ for some $d_k\in \mathbb{N}$. Let $C^k=[\|x^k_i-x^k_{i'}\|^2]_{i,i'\in[1:n_k]}$ and $\mathrm{p}^k=[p^k_1,\dots,p^k_{n_k}]^\top$.
Given $\mathrm{p}\in \mathbb{R}_+^n$ with $|\mathrm{p}|=1$ for some $n\in\mathbb{N}$ and $\xi_1,\ldots, \xi_K\ge 0$ with $\sum_{k=1}^K\xi_k=1$, the GW barycenter problem is defined by: 

\begin{align}
\min_{ C,\gamma^k } \sum_{k=1}^K\xi_k\langle L(C,C^k)\circ \gamma^k, \gamma^k \rangle\label{eq:GW_bpc}, 
\end{align}
where the minimization is over all matrices $C\in \mathbb{R}^{n\times n},\gamma^k\in\Gamma(\mathrm{p},\mathrm{p}^k),\forall k\in[1:K]$. 

Similarly, we can extend the above definition into the PGW setting. In particular, we relax the assumptions $|\mathrm{p}|=1$ and $|\mathrm{p}^k|=1$ for each $k\in[1:K]$. Given $\lambda_1,\ldots, \lambda_K>0$, the PGW barycenter is the follow problem:  
\begin{align}
\min_{C,\gamma_k}\sum_k\xi_k\langle M(C,C^k)\circ \gamma^k,\gamma^k \rangle-2\lambda_k |\gamma^k|^2\label{eq:PGW_bpc_discrete}
\end{align}
where each $\gamma^k\in \Gamma_\leq(\mathrm{p},\mathrm{p}^k)$. 

The problem \eqref{eq:PGW_bpc_discrete} can be solved iterative by two steps: 

\textbf{Minimization with respect to $C$}: 
For each $k$, we solve the PGW problem $$\min_{\gamma^k\in\Gamma_\leq(p,p^k)}\langle M(C,C^k)\circ\gamma^k,\gamma^k\rangle-2\lambda_k|\gamma^k|^2$$
via solver \ref{alg:pgw_v1} or \ref{alg:pgw_v2}. 

\textbf{Minimization with respect to $\{\gamma^k\}_k$}:
\begin{align}
\min_C\sum_k\xi_k\langle M(C,C^k)\circ \gamma^k,\gamma^k\rangle \label{eq:pgw_bpc_C}
\end{align}
Note, we can ignore the $-2\lambda_k|\gamma^k|^2$ terms as $\gamma^k$ is fixed in this case. 

It has closed form solution due to the following lemma and proposition:

\begin{lemma}\label{lem:deriv_dot}
Given matrices $A\in \mathbb{R}^{n,m},B\in \mathbb{R}^{m,l}, C\in \mathbb{R}^{n,l}$, let 
$$\mathcal{L}=\langle AB,C\rangle,$$
then $\frac{d\mathcal{L}}{dA}=CB^\top$. 
\end{lemma}
\begin{proof}
For any $i\in [1:n], j\in [1:m]$, we have 
\begin{align}
\frac{d\mathcal{L}}{dA_{ij}}&:=\sum_{i',j'}\frac{d}{dA_{ij}}C_{i',j'}(AB)_{i',j'}\nonumber\\
&=\sum_{i',j'}C_{i',j'}\frac{d(\sum_{k}A_{i',k}B_{k,j'})}{dA_{ij}}\nonumber\\
&=\sum_{j'}C_{i,j'}B_{k,j'}=(CB^\top)_{ij}\nonumber. 
\end{align}
\end{proof}

\begin{proposition}\label{pro:pgw_bpc}
If $L$ satisfies \eqref{eq:L_cond}, and $f_1'/h_1'$ is invertible, then  \eqref{eq:pgw_bpc_C} can be solved by
\begin{align}
    C=\left(\frac{f_1'}{h_1'}\right)^{-1}\left(\frac{\sum_k\xi_k\gamma^kh_2(C^k)(\gamma_k)^\top}{\sum_k\xi_k\gamma^k_1(\gamma^k_1)^\top}\right),\label{eq:pgw_bpc_C_sol}
\end{align}
where 
$$\frac{A}{B}=\left[\frac{A_{ij}}{B_{ij}}\right]_{ij}, \text{with convention }\frac{0}{0}=0.$$

Special case: if $|\mathrm{p}|\leq |\mathrm{p}^k|,\forall k$, when $\lambda$ is sufficiently large, \eqref{eq:pgw_bpc_C_sol} and \cite[Proposition 3]{peyre2016gromov} coincide. 
\end{proposition}

\begin{proof}
From Proposition \ref{pro: tensor product gw}, the objective in \eqref{eq:pgw_bpc_C} becomes
\begin{align}
\mathcal{L}=&\sum_k\xi_k\langle f_1(C)\gamma_1^11^\top_{n_k}+1_{n}(\gamma_2^k)^\top f_2(C^k)-h_1(C)\gamma^kh_2(C^k)^\top,\gamma^k\rangle\nonumber\\
&=\sum_k\xi_k\langle f_1(C)\gamma_1^11^\top_{n_k},\gamma^k\rangle+\underbrace{\sum_k\xi_k\langle 1_n(\gamma_2^k)^\top f_2(C^k),\gamma^k\rangle}_{\text{constant}}-\sum_k \xi_k\langle h_1(C)\gamma^kh_2(C^k)^\top,\gamma^k\rangle\nonumber 
\end{align}
We set $\frac{d\mathcal{L}}{dC}=0$. From Lemma \ref{lem:deriv_dot}, we have: 
\begin{align}
0&=\frac{d\mathcal{L}}{dC}\nonumber\\
&=\sum_k\xi_k f_1'(C)\odot \gamma^k1_{n_k}(\gamma_1^k)^\top-\sum_k\xi_kh_1'(C)\odot\gamma^kh_2(C^k)(\gamma^k)^\top\nonumber\\
&=f_1'(C)\odot\sum_k\xi_k\gamma^k1_{n_k}(\gamma_1^k)^\top-h_1'(C)\odot \sum_k\xi_k\gamma^kh_2(C^k)(\gamma^k)^\top\nonumber \\
&=f_1'(C)\odot\underbrace{\sum_k\xi_k\gamma^k_1(\gamma_1^k)^\top}_B-h_1'(C)\odot \underbrace{\sum_k\xi_k\gamma^kh_2(C^k)(\gamma^k)^\top}_A.\label{pf:A/B}\end{align}
We claim $\frac{A}{B}$ is well-defined, i.e., if $B_{ij}=0$, then $A_{ij}=0$. 

For each $i,j\in[1:n]$, if $B_{ij}=0$, we have two cases:

Case 1: $\forall k\in[1:K]$, we have $\gamma^k_1[i]=0$. 

Thus, 
$\gamma^k[i,:]=0_{n_k}^\top$. So $A[i,:]=(\gamma^kh_2(C^k)(\gamma^k)^\top)[i,:]=0_{n_k}^\top$.  

Case 2: $\forall k\in[1:K]$, we have $\gamma^k_1[j]=0$. 

It implies $(\gamma^k)^\perp[:,j]=0_{n}$, thus $A[:,j]=(\gamma^kh_2(C^k))(\gamma^k)^\top[:,j]=0_{n_k}$. Therefore, $A_{ij}=0$. 

Thus $\frac{A}{B}$ is well-defined. 

In addition, in these two cases, if we change the value $C_{ij}^k$, $\mathcal{L}$ will not change.  

From \eqref{pf:A/B}, we have:  
\begin{align}
\left(\frac{f_1'}{h_1'}(C)\right)_{ij}=\frac{\left(\sum_k\xi_k\gamma^kh_2(C^k)(\gamma^k)^\top\right)_{ij}}{\left(\sum_k\xi_k\gamma^k_1(\gamma_1^k)^\top\right)_{ij}}\nonumber
\end{align}
if $B_{ij}>0$. In addition, if $B_{ij}=0$, there is no constraint for $C_{ij}$.  

Combining it with the fact that if $B_{i,j}=0$, then $C_{i,j}$ does not affect $\mathcal{L}$. 
Thus, 

we have the following is a solution: 
$$C=\left(\frac{f_1'}{h_1'}\right)^{-1}\left(\frac{\sum_k\xi_k\gamma^kh_2(C^k)(\gamma^k)^\top}{\sum_k\xi_k\gamma^k_1(\gamma_1^k)^\top}\right).$$

In particular case: $|\mathrm{p}|\leq |\mathrm{p}^k|,\forall k$, suppose $\lambda > \max \{c^2: c\in \bigcup_kC^k\cup C \}$, by lemma \ref{lem:gamma_mass}, we have for each $k$,  $|\gamma^k|=\min(|\mathrm{p}|,|\mathrm{p}|^k)=|\mathrm{p}|$, that is $\gamma^k_1=\mathrm{p}$. 

Thus, \begin{align}
\sum_k\xi_k\gamma^k_1(\gamma_1^1)^\top=\sum_k \xi_k\gamma_1^k(\gamma_1^k)^\top=\sum_k\xi_k\mathrm{p}\mathrm{p}^\top=\mathrm{p}\mathrm{p}^\top\nonumber 
\end{align}
Thus, $C=\left(\frac{f_1'}{h_1'}\right)^{-1}\left(\frac{\sum_k\xi_k\gamma^kh_2(C^k)(\gamma^k)^\top}{\mathrm{p}\mathrm{p}^\top}\right).$
\end{proof}
\begin{remark}
    In $l^2$ loss case, i.e. $L(r_1,r_2)=|r_1-r_2|^2$, \eqref{eq:pgw_bpc_C_sol} becomes    
\begin{align}
C=\frac{\sum_k\xi_k\gamma^kC^k(\gamma^k)^\top}{\sum_k\xi_k\gamma^k_1(\gamma^k_1)^\top}. \label{eq:pgw_bpc_C_sol_l2} 
\end{align}
\end{remark}
Since in this case, we can set  $$f_1(x)=x^2,f_2(y)=y^2,h_1(x)=2x,h_2(y)=y.$$
Thus $\frac{f_1'}{h_1'}(x)=\frac{2x}{2}=x$ and $\left(\frac{f_1'}{h_1'}\right)^{-1}(x)=x$. Therefore, \eqref{eq:pgw_bpc_C_sol} becomes \eqref{eq:pgw_bpc_C_sol_l2}. 


\begin{algorithm}[tb]
   \caption{Partial Gromov-Wasserstein Barycenter}
   \label{alg:pgw_bpc}
\begin{algorithmic}
   \STATE {\bfseries Input:}
   $\{C^k,\mathrm p^k,\lambda_k\}_{k=1}^K,\mathrm p$
   \STATE {\bfseries Output:}
   $C$
   \STATE Initialize $C$. 
   \FOR{$i=1,2,\ldots$}
   \STATE compute $\gamma^k\gets\arg\min_{\gamma\in\Gamma_\leq(\mathrm p,\mathrm p^k)}\langle \mathcal{L}(C,C^k)-2\lambda_k,\gamma\rangle,\forall k\in[1:K]$. 
   \STATE Update $C$ by \eqref{eq:pgw_bpc_C_sol}. 
   \STATE if convergence, break
   \ENDFOR
\end{algorithmic}
\end{algorithm}

Similarly, we can also extend the above PGW Barycenter into the MPGW setting: 
$$\min_{C,\gamma^k}\sum_{k=1}^K\xi_k\langle L(C,C^k)\circ\gamma^k,\gamma^k \rangle,$$
where, for  each $k\in[1:K]$, $\rho_k\in[0, \min(|\mathrm p|,|\mathrm p^ k|)]$, and the optimization is over $C\in \mathbb{R}^n$ and   $\gamma_k\in\Gamma^{\rho_k}_\leq(\mathrm p,\mathrm p^k)$ for $k\in[1:K]$. 

It can be solved by the following algorithm \ref{alg:mpgw_bpc}. 
\begin{algorithm}[tb]
   \caption{Mass-Constrained Partial Gromov-Wasserstein Barycenter}
   \label{alg:mpgw_bpc}
\begin{algorithmic}
   \STATE {\bfseries Input:}
   $\{C^k,\mathrm p^k,\lambda_k\}_{k=1}^K,\mathrm p$
   \STATE {\bfseries Output:}
   $C$
   \STATE Initialize $C$. 
   \FOR{$i=1,2,\ldots$}
   \STATE compute $\gamma^k\gets\arg\min_{\gamma\in\Gamma^{\rho_k}_\leq(\mathrm p,\mathrm p^k)}\langle \mathcal{L}(C,C^k),\gamma\rangle,\forall k\in[1:K]$. 
   \STATE Update $C$ by \eqref{eq:pgw_bpc_C_sol}. 
   \STATE if convergence, break
   \ENDFOR
\end{algorithmic}
\end{algorithm}

\subsection{Details of Shape Interpolation Experiment}\label{sec:shape_interpolation_2}
\textbf{Dataset and data processing.}
We apply the dataset in \cite{peyre2016gromov} with download \href{https://github.com/gpeyre/2016-ICML-gromov-wasserstein}{link}. The original data are images, which we convert into a point cloud using the k-mean algorithm, where $k=1024$ (see the second row of Figure \ref{fig:interpolation_data}).

\begin{figure}[h!]
    \centering
\includegraphics[width=1.0\textwidth]{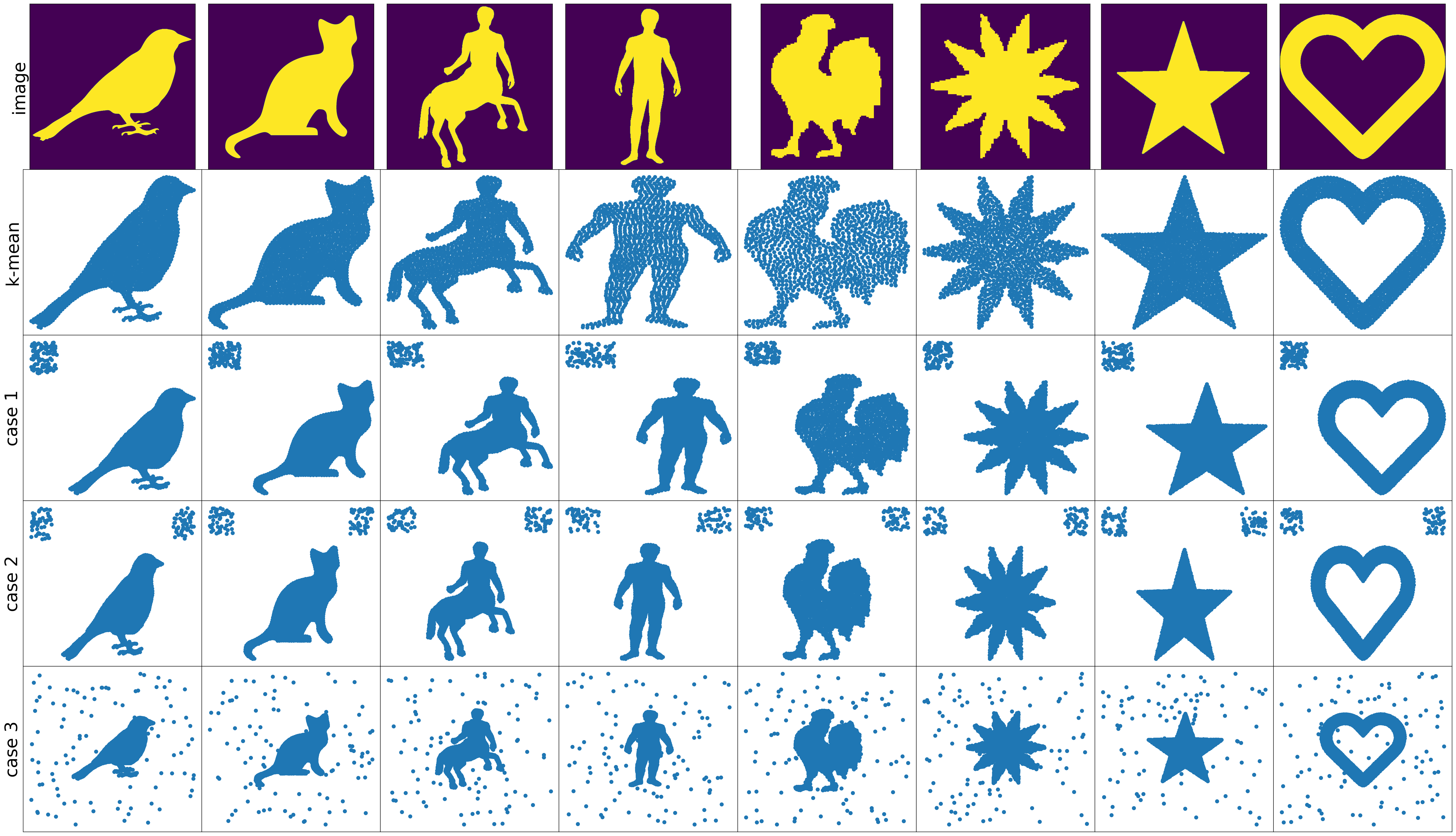}
\caption{We visualize the dataset in shape interpolation. The first row is the original images in \href{https://github.com/gpeyre/2016-ICML-gromov-wasserstein}{Link}. The second row is the point clouds obtained by the k-mean method, where $k=1024$.  }
    \label{fig:interpolation_data}
\end{figure}

\begin{figure}
    \centering    \includegraphics[width=\textwidth]{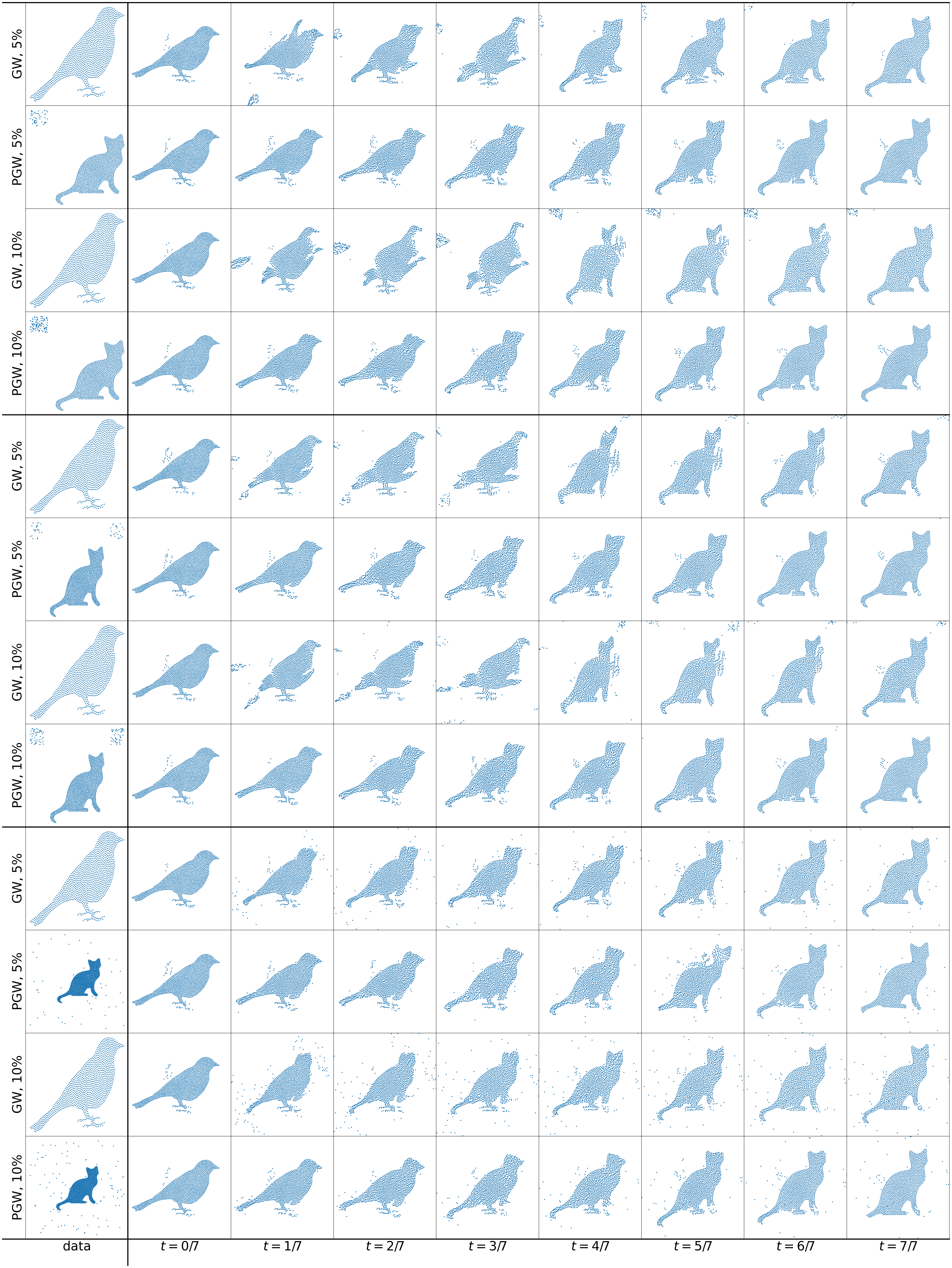}
    \caption{We test interpolation tasks in 3 scenarios: source data is clean, target data is selected from three cases as described in section \textbf{dataset and data processing}. In each scenario, we test $\eta=5\%,10\%$ respectively. In the first column, we present the source and target point cloud visualization in each task. In columns 2-9, we present GW, PGW barycenter for $t=0/7,1/7,\ldots, 7/7$. }
    \label{fig:interporlation_all}
\end{figure}

Suppose $\mathcal{D}\subset\mathbb{R}^2$ is a region that contains these point clouds. 
Let $\mathcal{R}\subset\mathbb{R}^2$ denote another region. In $\mathcal{R}$, we randomly select and add $n\eta$ noise points to these point clouds. In particular, we consider noise corruption in the following three cases: 

Case 1: $\mathcal{R}$ is a rectangle region which is disjoint to $\mathcal{D}$.  See the third row in  Figure \ref{fig:interpolation_data}. 

Case 2: 
$\mathcal{R}=\mathcal{R}_1\cup \mathcal{R}_2$, where $\mathcal{R}_1,\mathcal{R}_2$ are rectangles which are disjoint to $\mathcal{D}$. See the fourth row in Figure \ref{fig:interpolation_data}. 

Case 3: $\mathcal{R}$ contains $\mathcal{D}$. See the fifth row in Figure \ref{fig:interpolation_data}. 

\textbf{GW Barycenter and PGW Barycenter methods}. 
We select $t_1,\ldots, t_K$ with 
$0=t_1<t_2<\ldots< t_K=1$. For each $t\in \{t_1,\ldots, t_K\}$, we compute the GW Barycenter 
\begin{align}
\arg\min_{C,\gamma^1,\gamma^2}(1-t)\langle L(C,C^1)\circ\gamma^1,\gamma^1\rangle+t\langle L(C,C^2)\circ\gamma^2,\gamma^2\rangle, \label{ex:gw_bp}    
\end{align}
where $\gamma_1\in\Gamma(\mathrm p,\mathrm p^1),\gamma_2\in\Gamma(\mathrm p,\mathrm p^2)$.
Apply Smacof-MDS to the minimizer $C$, the resulting embedding, denoted as $X_t\in \mathbb{R}^{n\times 2}$ (where $n=1024$) is the GW-based interpolation. 

Replacing the GW Barycenter with the PGW Barycenter 
\begin{align}
\arg\min_{C,\gamma^1,\gamma^2}(1-t)(\langle L(C,C^1)\circ\gamma^1,\gamma^1\rangle+\lambda_1|\gamma^1|^2)+t(\langle L(C,C^2)\circ\gamma^2,\gamma^2\rangle+\lambda_2|\gamma^2|),\label{ex:pgw_bp}    
\end{align}
where $\lambda_1,\lambda_2>0, \gamma^1\in\Gamma_\leq(\mathrm p,\mathrm p^1),\gamma^2\in\Gamma_\leq(\mathrm p,\mathrm p^2)$. Then, we obtain PGW-based interpolation. 

\textbf{Problem setup}. 
We select one point cloud from the clean dataset denoted as $X=\{x_i\}_{i=1}^n$ (source point cloud), $n=1024$. 

Next, we select one noise-corrupted point cloud, as described in Case 1, Case 2, and Case 3, respectively. 
In these three scenarios, we test $\eta=0.5\%$ and $\eta=10\%$ where $\eta$ is the noise level. Therefore, we test $3*2=6$ different interpolation tasks for these two methods.
The size of the target point cloud is then  $m=n+n\eta$. 
See Figure \ref{fig:interporlation_all} for details.

\textbf{Numerical details.}
In the GW-barycenter method, because of the balanced mass setting, we set 
$$\mathrm p^1=\frac{1}{n}1_{n},\mathrm p^2=\frac{1}{m}1_m,\mathrm p=\frac{1}{n}1_n.$$ 
In PGW-barycenter, we set 
$$\mathrm p^1=\frac{1}{n}1_n,\mathrm p^2=\frac{1}{n}1_m,\mathrm p=\frac{1}{n}1_n.$$
In addition, we set $\lambda_1,\lambda_2$ such that $2\lambda_1,2\lambda_2\ge \max(\max(C_1)^2,\max(C_2)^2)$. 
We compute GW/PGW barycenter for $t=0/7,1/7,\ldots, 7/7$. 

In both GW and PGW barycenter algorithms, we set the largest number of iterations to 100. The threshold for convergence is set to be 1e-5.

\textbf{Performance analysis.}
Each interpolation task is essentially unbalanced: the source point cloud contains clean data, while the target point cloud contains clean and noise points. We observe that in the first two scenarios, the interpolation derived from GW is clearly disturbed by the noise data points. For example, in rows $1,3,5,7$, columns $t=1/7, 2/7, 3/7$, we see that the point clouds reconstructed by MDS have significantly different width-height ratios from those of the source and target point clouds.

In contrast, PGW is significantly less disturbed, and the interpolation is more natural. The width-height ratio of the point clouds generated by the PGW barycenter is consistent with that of the source/target point clouds.

In the third scenario, the noise data is uniformly selected from a large region that contains the domain of all clean point clouds. In this case, we observe that the GW and PGW barycenters perform similarly. However, at $t=1/7, 2/7,4/7$, GW-barycenters present more noise points than PGW-barycenters in the same truncated region.

\textbf{Limitations and future work}.
The main issue of the above GW/PGW techniques arises from the MDS method:

Given minimizer $C\in \mathbb{R}^{n \times n}$ of GW/PGW barycenter problem \eqref{ex:gw_bp} (or \eqref{ex:pgw_bp}), MDS studies the following problem:
\begin{align}
\min_{X \in \mathbb{R}^{n \times d}} \sum_{i,i'=1}^n \left| C_{i,i'}^{1/2} - \|X_i - X_{i'}\| \right|^2 \label{eq:mds}
\end{align}

Let \( O(n) \) denote the set of all \( n \times n \) orthonormal matrices. Suppose \( X^* \) is a minimizer, then \( RX^* \) is also a minimizer for the above problem for all \( R \in O(n) \).

In practice, this means manually setting suitable rotation and flipping matrices for each method at each step, especially for the GW method.

However, we understand that this issue stems from the inherent properties of the GW/PGW method. GW can be seen as a tool that describes the similarity between two graphs, which are rotation-invariant and flipping-invariant. Therefore, the GW/PGW barycenter essentially describes the interpolation between two graphs rather than two point clouds.

\subsection{Multi-Shape Interpolation}
In this section, we present the interpolation between 4 different shapes. In addition, let $\eta$ be the percentage of outliers; we test $\eta=0$ and $5\%$. 

\textbf{Experiment setup.} We select 4 shapes (bird, cat, human, and rooster) from the 2D clouds dataset.  Two of them, i.e., cat and rooster, are embedded into 4D space and are randomly rotated. The goal is to find the interpolation between them. 

\textbf{Baselines.} We select 4 baselines, optimal transport barycenter \cite{cuturi2014fast}, partial optimal transport barycenter \cite{Bonneel2019sliced},
GW barycenter \cite{peyre2016gromov} and our PGW barycenter. 

Note that OT and partial OT barycenter can not be directly applied to this setting. We first embed the 2D shapes into 4D space via mapping 
$$\mathbb{R}^2\ni x\mapsto [x;0;0]\in \mathbb{R}^4,$$
and then compute the OT/POT barycenter. Finally, we apply PCA to project the result back to 2D space for visualization.  

\textbf{Result analysis.}
We present the result in Figure \ref{fig:4_shapes_interpolation_0}.

\begin{figure}[h] 
    \centering
\begin{subfigure}[t]{0.49\textwidth}
        \centering
\includegraphics[width=\columnwidth]{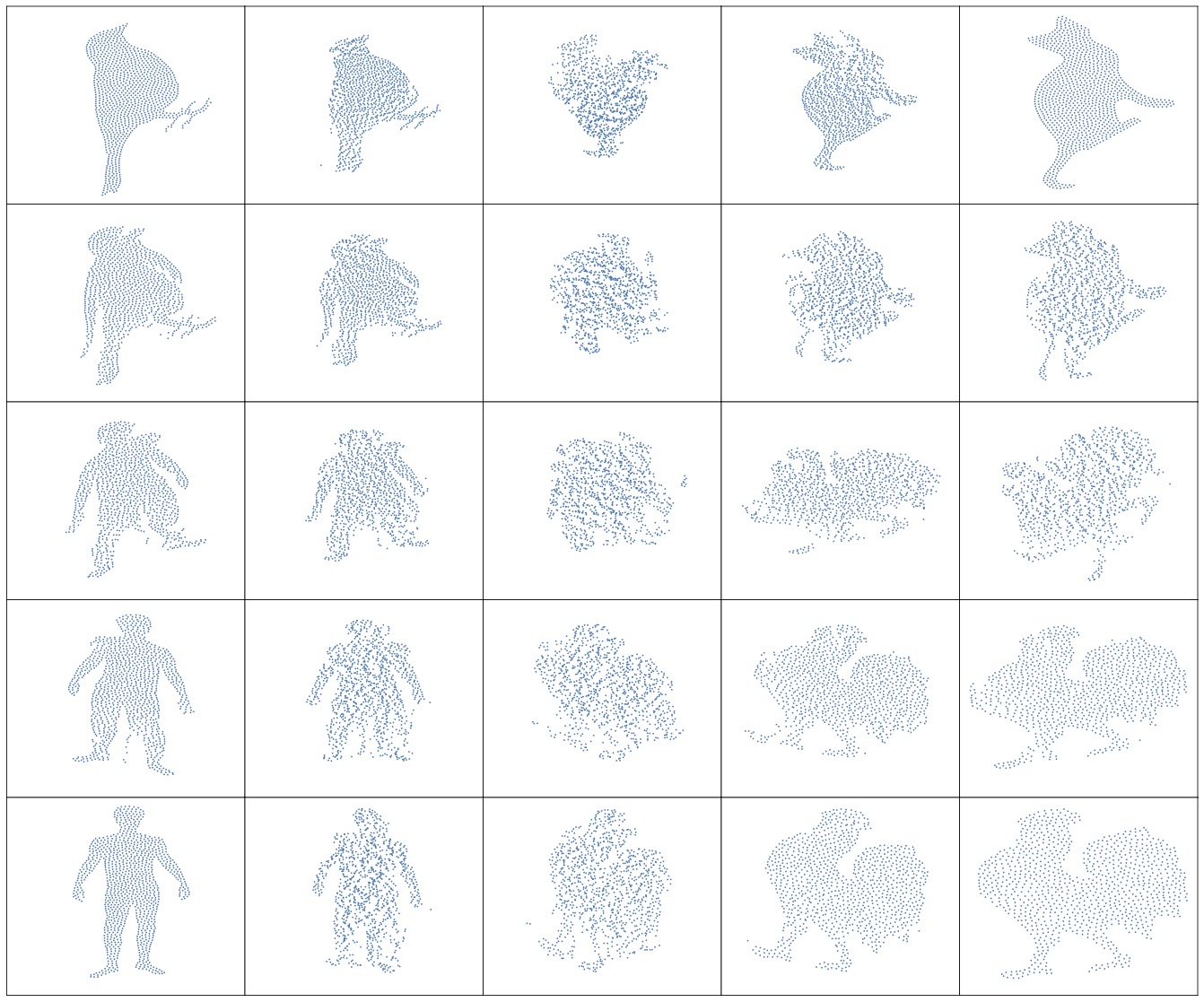}
\caption{OT barycenter}
\label{fig:ot_bt}
\end{subfigure}
    \begin{subfigure}[t]        {0.49\textwidth}
        \centering
\includegraphics[width=\columnwidth]{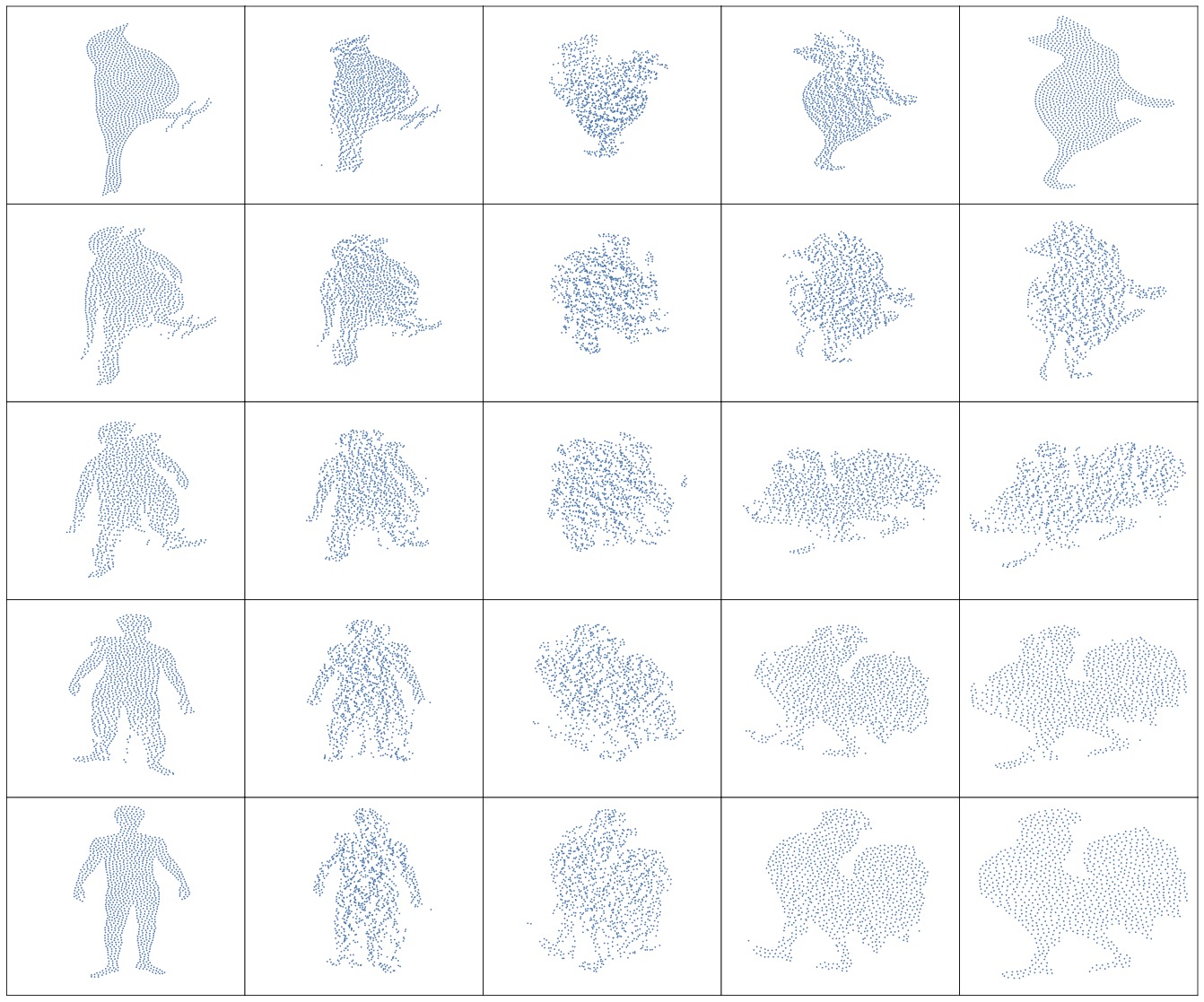}
\caption{Partial OT barycenter}
\label{fig:pot_bt}
    \end{subfigure}
\begin{subfigure}[t]{0.49\textwidth}
        \centering
\includegraphics[width=\columnwidth]{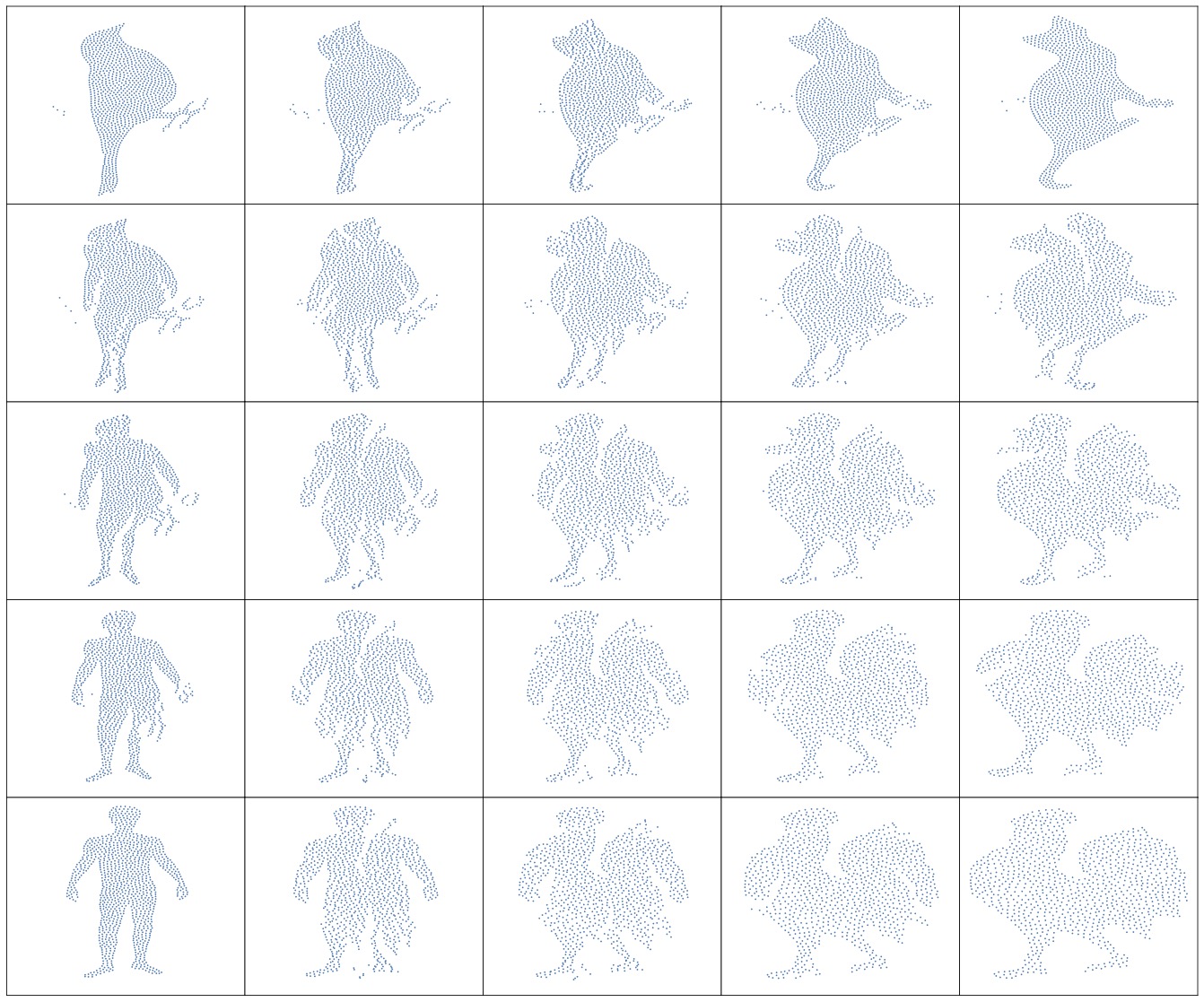}
\caption{GW barycenter}
\label{fig:gw_bt}
\end{subfigure}
    \begin{subfigure}[t]        {0.49\textwidth}
        \centering
\includegraphics[width=\columnwidth]{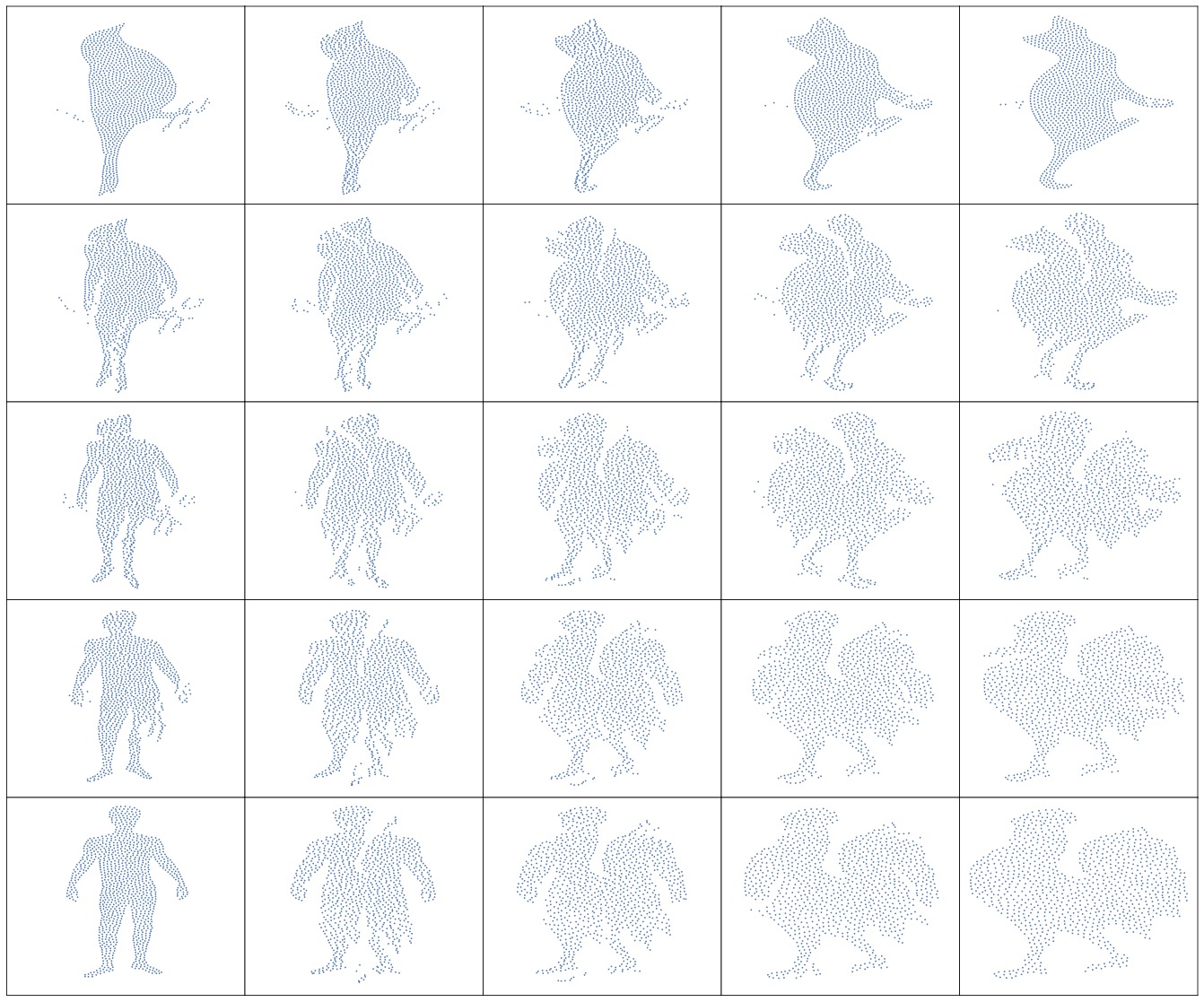}
\caption{Partial  GW barycenter}
\label{fig:pgw_bt}
    \end{subfigure}

\caption{We visualize multi-shapes interpolation, where the proposition of noise $\eta$ is 0. ``bird'' and ``human'' shapes are distributed in 2D space, ``cat'' and ``rooster'' are distributed in 4D space. }
    \label{fig:4_shapes_interpolation_0}
\end{figure}

\begin{figure}[h] 
    \centering
\begin{subfigure}[t]{0.49\textwidth}
        \centering
\includegraphics[width=\columnwidth]{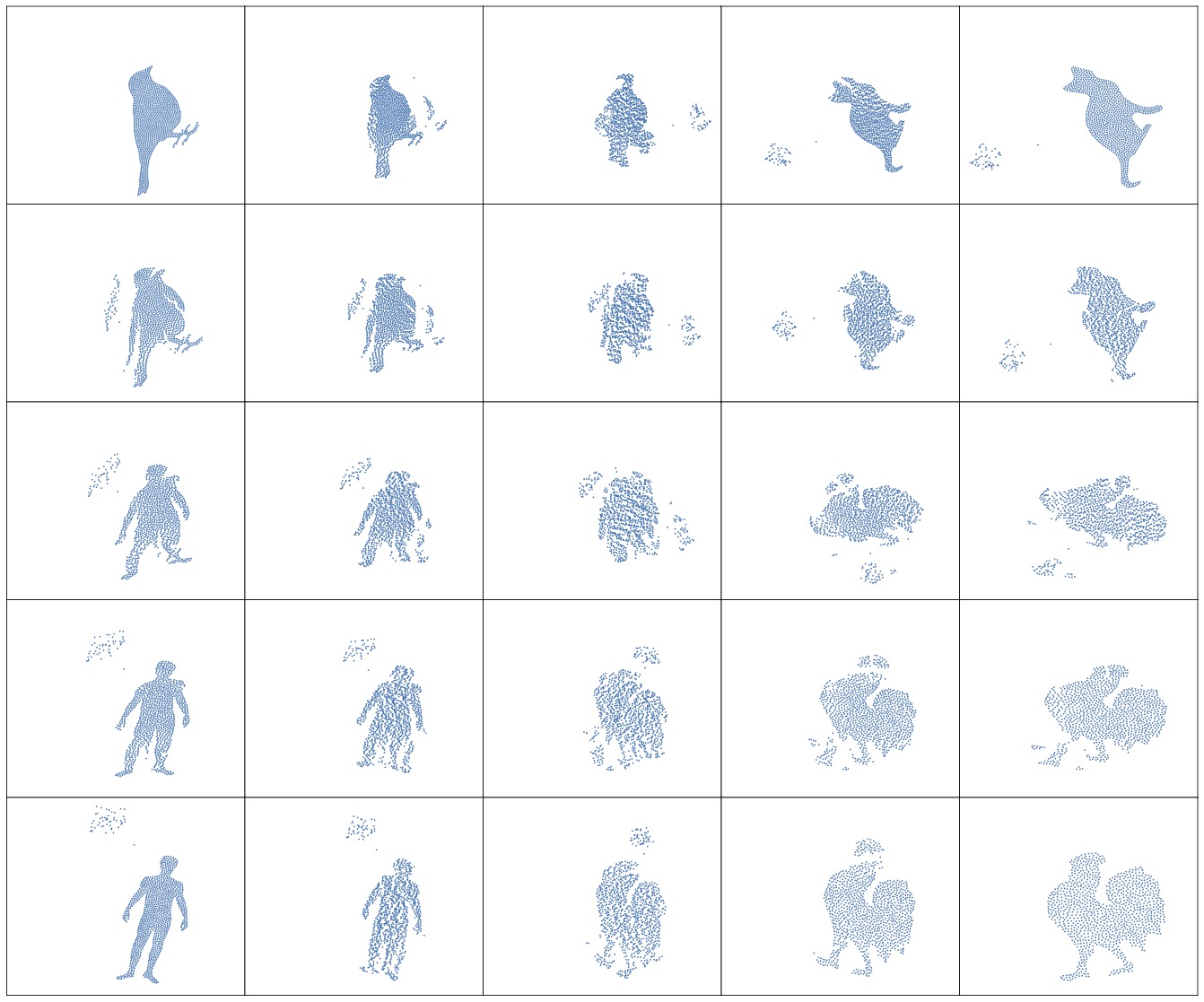}
\caption{OT barycenter}
\label{fig:ot_bt2}
\end{subfigure}
    \begin{subfigure}[t]        {0.49\textwidth}
        \centering
\includegraphics[width=\columnwidth]{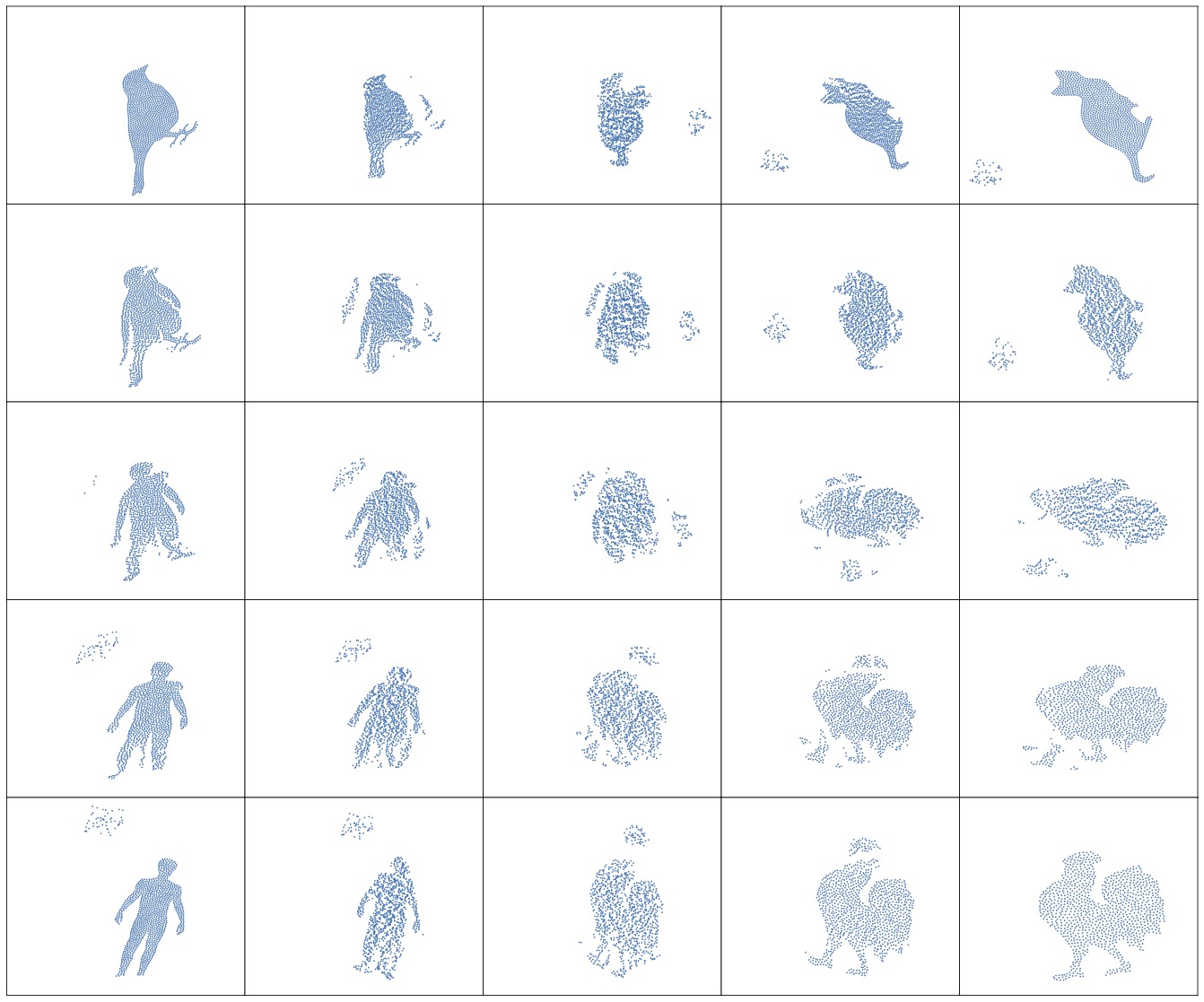}
\caption{Partial OT barycenter}
\label{fig:pot_bt2}
    \end{subfigure}
\begin{subfigure}[t]{0.49\textwidth}
        \centering
\includegraphics[width=\columnwidth]{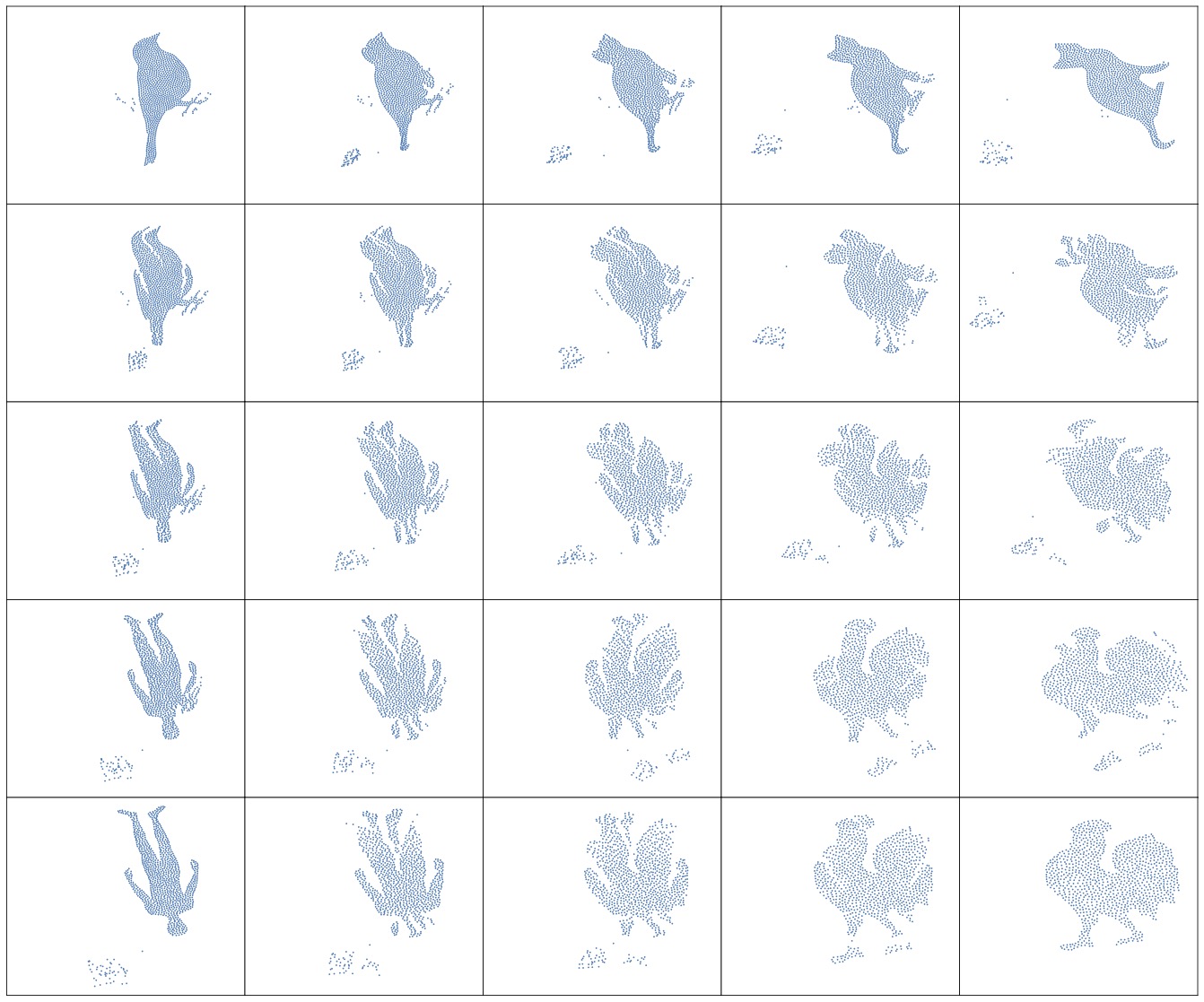}
\caption{GW barycenter}
\label{fig:gw_bt2}
\end{subfigure}
    \begin{subfigure}[t]        {0.49\textwidth}
        \centering
\includegraphics[width=\columnwidth]{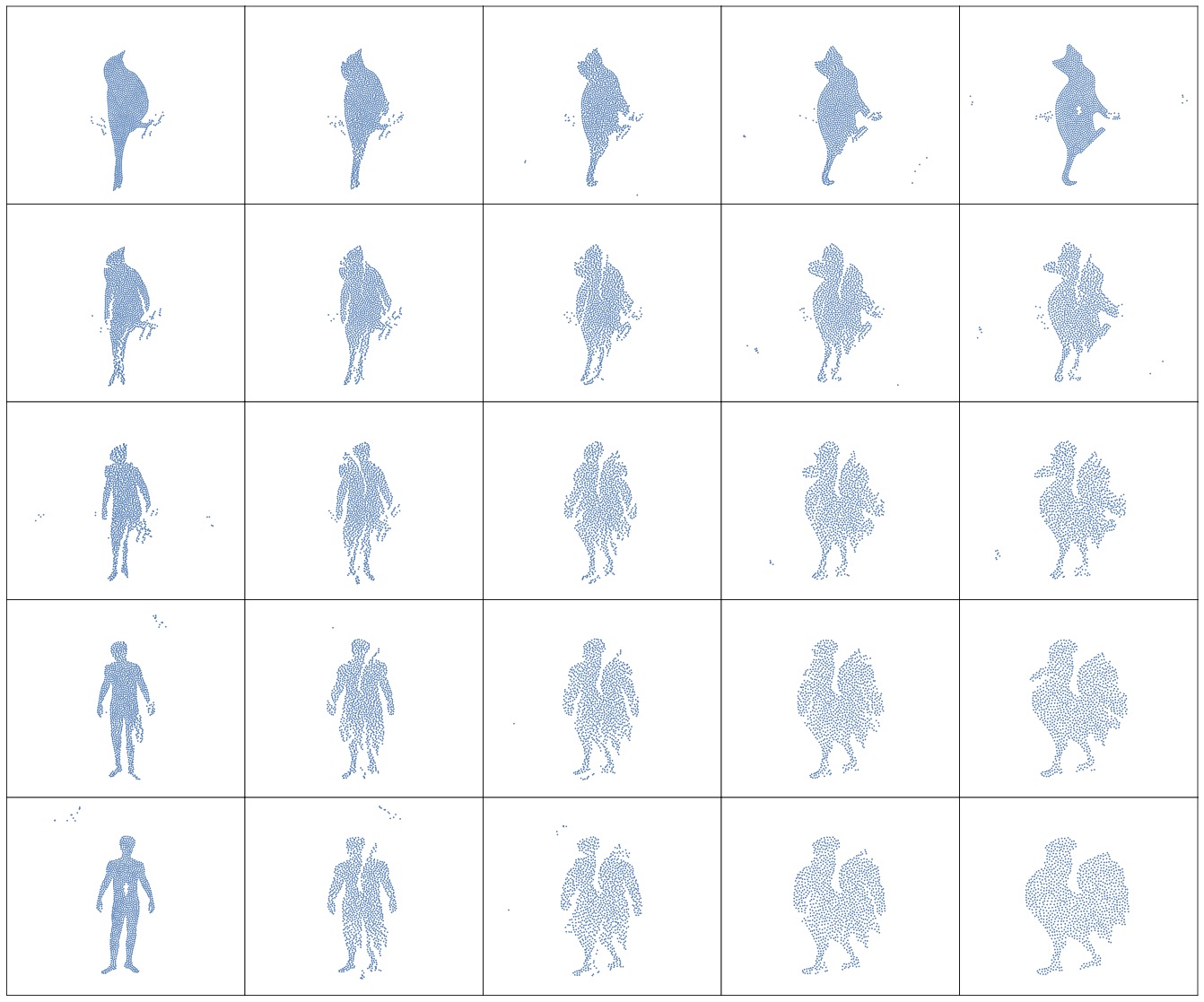}
\caption{Partial GW barycenter}
\label{fig:pgw_bt2}
    \end{subfigure}

\caption{We visulize of multi-shapes interporlation. ``bird'',``human'' shapes are distributed in 2D space, ``cat'', ``rooster'' are distributed in 4D space. The noise level, $\eta$, is $5\%$.  }
\label{fig:4_shape_interporlation_5}
\end{figure}

When $\eta=0$, from Figure \ref{fig:4_shapes_interpolation_0}, we observe that  OT/POT admits good interpolation results between 2D shapes (bird and human). However, the OT/POT interpolation between 2D and 4D is not as good as GW/PGW.  The main reason is that the OT/POT barycenter incorporates the absolute coordinates of these shapes for interpolation, which is affected by random rotation. However,  GW/PGW is not affected by the absolute coordinates of the shapes. 

In Figure \ref{fig:4_shape_interporlation_5}, we demonstrate the result for $\eta=5\%$. Two shapes are corrupted by the noise points (cat and rooster). We observe that GW's interpolation is affected by the outliers (see, e.g., from bird to human or from cat to rooster), while PGW is less affected by these outlier points and admits much better/smoother interpolation.

\section{Details of Shape Matching}\label{sec:shape_matching_2}
\textbf{Dataset setup}. 
In the Moon dataset (see \href{https://scikit-learn.org/stable/modules/generated/sklearn.datasets.make_moons.html}{link}), 
we apply $n=200$ and set Gaussian variance to be $0.2$. The outliers are sampled from region $[[-2,-1.5]\times[-3.5,-3]]$. In the second experiment, the circle data is uniformly sampled from a 2D circle 
$$\mathbb{S}^1=\{s\in \mathbb{R}^2: \|s\|^2=1\}$$
and spherical data is uniformly sampled from 3D sphere $$\mathbb{S}^{2}=\{s+[0,0,4]\in \mathbb{R}^2: \|s\|^2=1\},$$
where the shift $[0,0,4]$ is applied for visualization. We set sample size $n=200$ for both 2D and 3D samples. In both experiments, the number of outliers is $\eta n=0.2n=40$. 

\textbf{Numerical details}. 
In GW, we normalize the two-point clouds as 
$$\mathbb{X}=(X,d_X,\sum_{i=1}^n\frac{1}{n}\delta_{x_i}),\mathbb{Y}=(Y,d_Y,\sum_{j=1}^{n+n\eta}\frac{1}{n+n\eta}\delta_{y_j}).$$
In PGW, MPGW, and UGW, we define the point clouds as 
$$\mathbb{X}=(X,d_X,\sum_{i=1}^n\frac{1}{n}\delta_{x_i}),\mathbb{Y}=(Y,d_Y,\sum_{j=1}^{n+n\eta}\frac{1}{n}\delta_{y_j}).$$
In PGW, we choose $\lambda $ such that $\lambda\ge \max(\max((C^X)^2),\max((C^Y)^2))$, in particular, $\lambda=10.0$. 
In addition, we tested $\lambda=0.01,0.1,1.0,10.0$. In MPGW, we tested $\rho=0.3,0.5,0.8,1.0$ and selected $\rho=1.0$ In UGW, we tested $[1e-10,1e-1,1.0,10]$ and selected 
$\rho_1=\rho_2=1.0$, $\epsilon=0.05$. We refer to Figures \ref{fig:shape_matching_full_2d} and \ref{fig:shape_matching_full_3d} for the visualization of all tested parameters. 

\begin{figure}[h!]
    \centering
    \includegraphics[width=.75\textwidth]{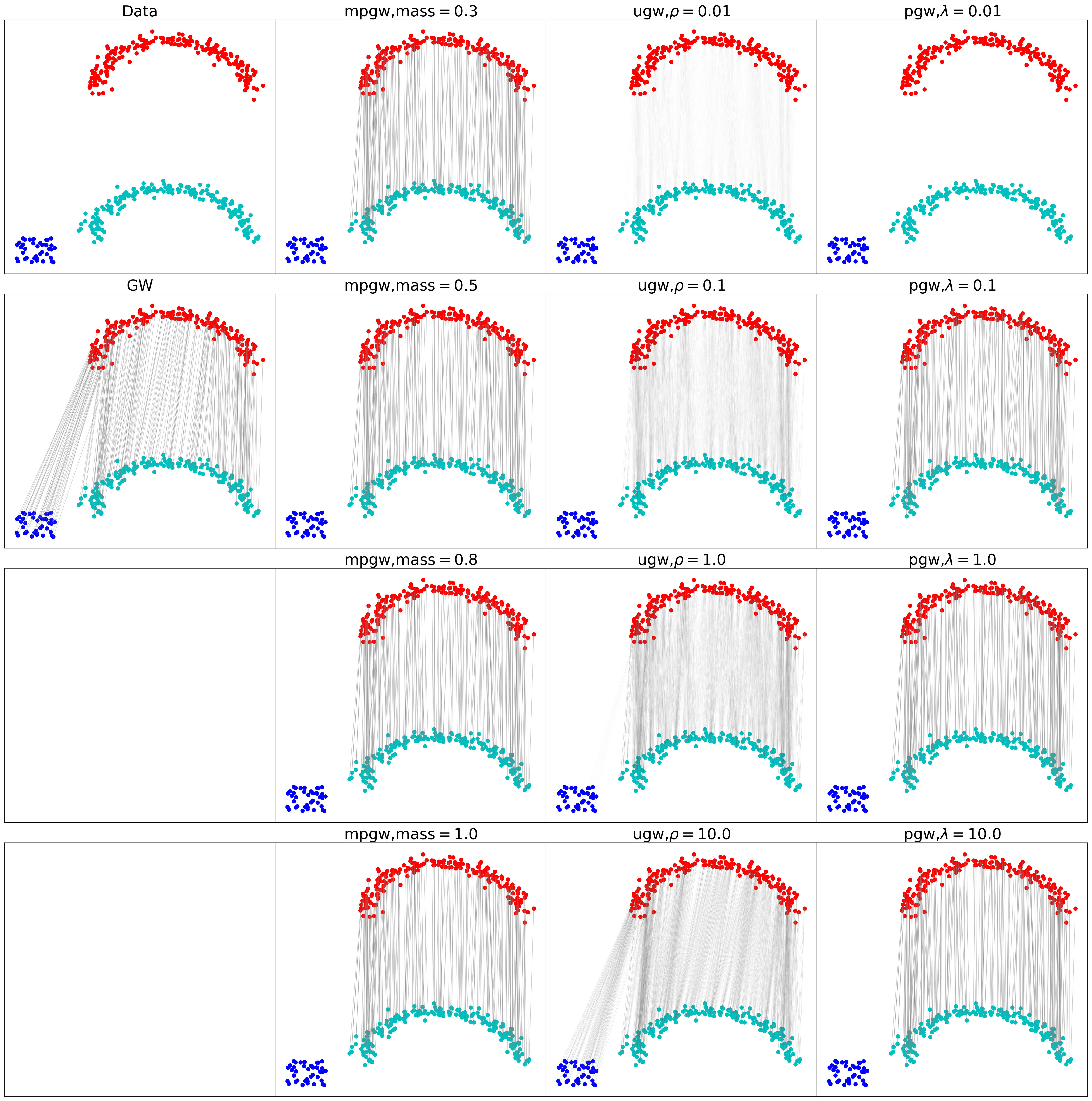}
    \caption{Visualization of shape matching result for 2D and 2D matching.}
    \label{fig:shape_matching_full_2d}
\end{figure}

\begin{figure}[t]
    \centering
    \includegraphics[width=.75\textwidth]{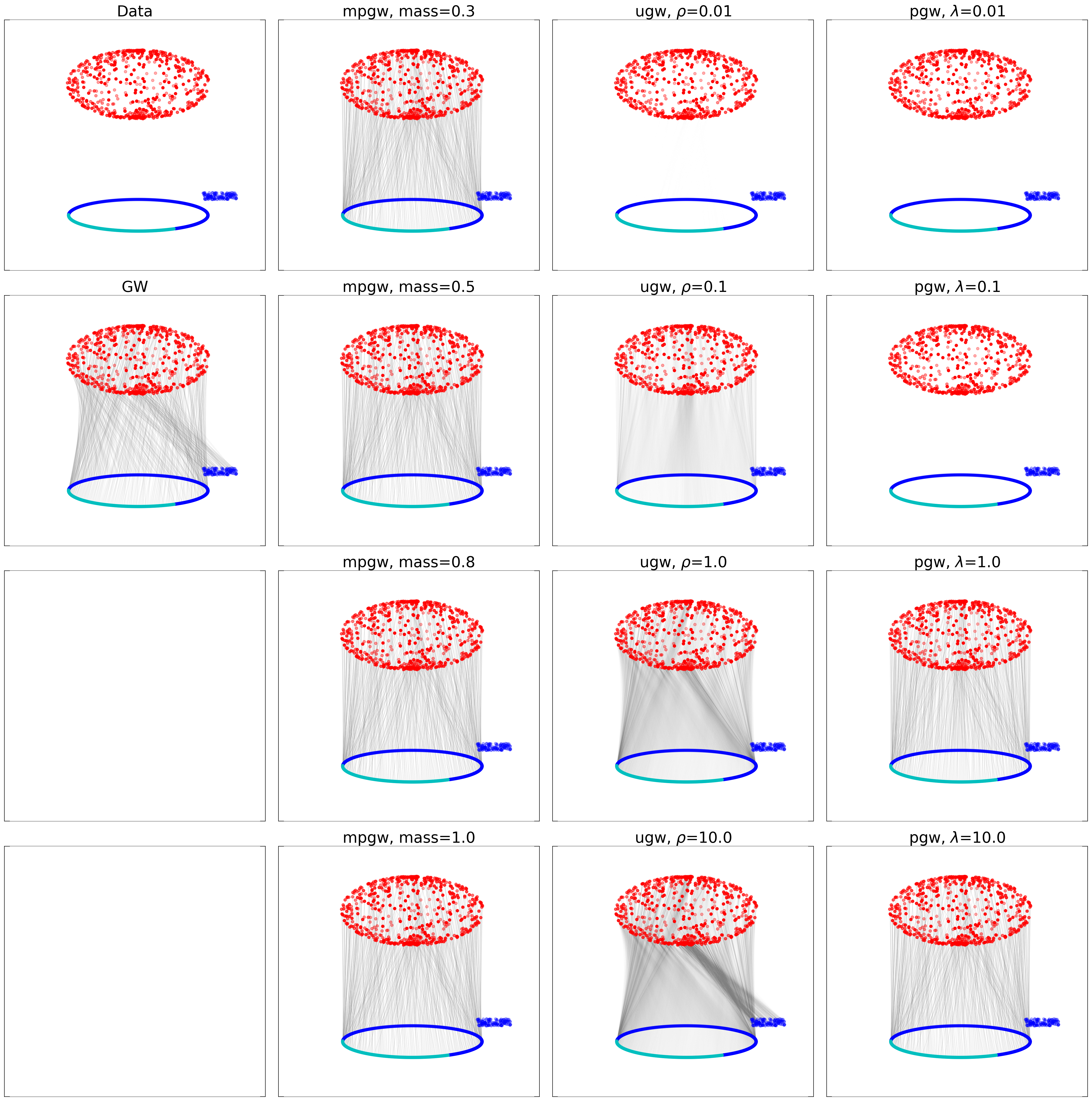}
    \caption{Visualization of shape matching result for 2D and 3D matching.}
    \label{fig:shape_matching_full_3d}
\end{figure}


\section{Details of Shape Retrieval Experiment}\label{sec:shape_retrieval_2}

\textbf{Dataset details.}
We test two datasets in this experiment, which we refer to as Dataset I and Dataset II. We visualize Dataset I in Figure \ref{fig:sr_data_1} and Dataset II in Figure \ref{fig:sr_data_2}. The complete datasets can be accessed from the supplementary materials.

\begin{figure}[h!] 
    \centering
    \begin{subfigure}[t]{0.5714\textwidth}
        \centering
\includegraphics[width=\columnwidth]{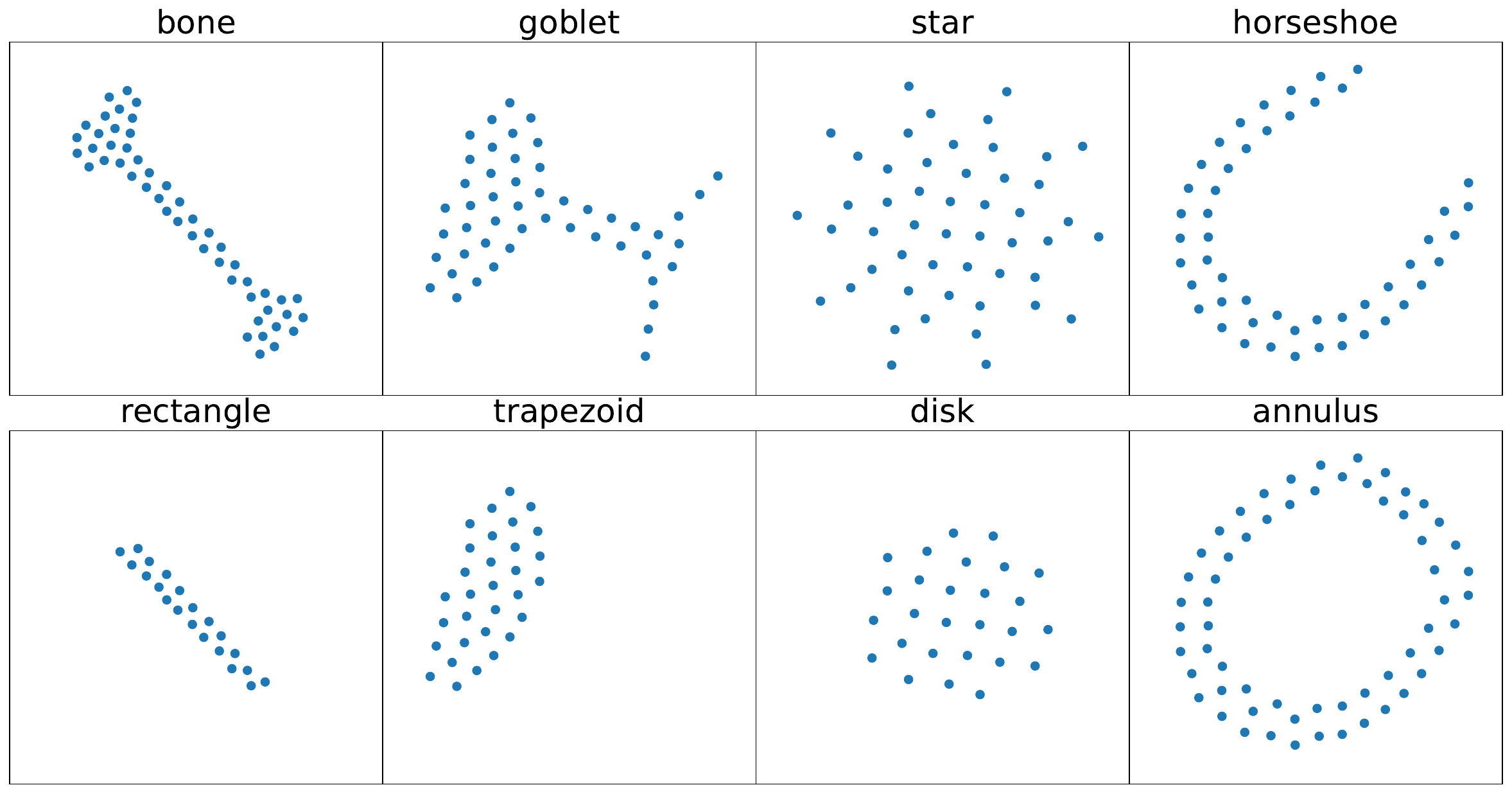}
\caption{Dataset I}
\label{fig:sr_data_1}
\end{subfigure}%
    ~ 
    \begin{subfigure}[t]        {0.4286\textwidth}
        \centering
\includegraphics[width=\columnwidth]{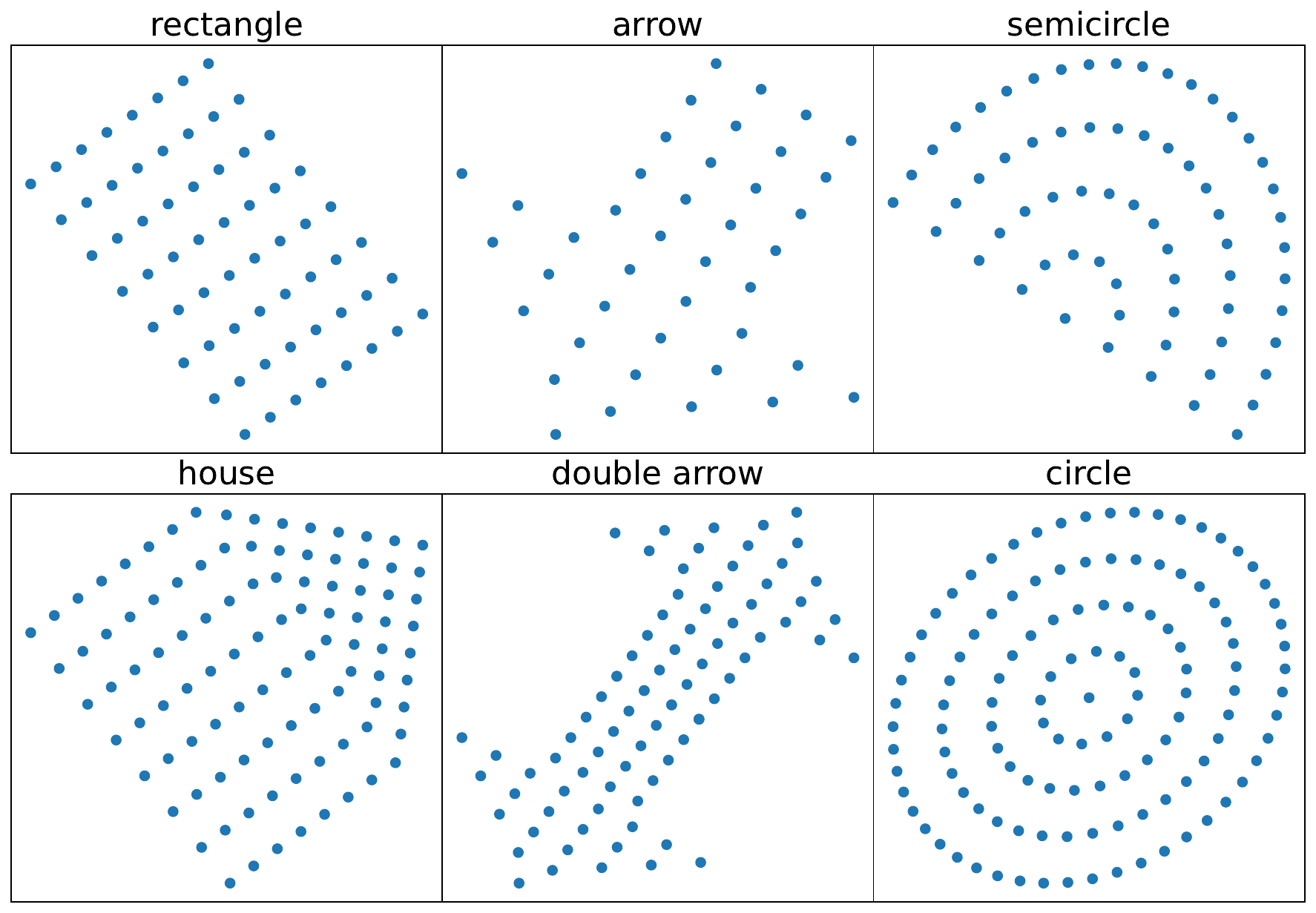}
\caption{Dataset II}
\label{fig:sr_data_2}
    \end{subfigure}
    \caption{Visualization of a representative shape from each class of the two datasets.}
    \label{fig:shape_retrieval_data}
\end{figure}

\textbf{Numerical details.}
We represent the shapes in each dataset as mm-spaces $\mathbb{X}^i= \left(\mathbb{R}^2, \| \cdot \|_2, \mu^i = \sum_{k=1}^{n^i} \alpha^i \delta_{x^i_k}\right)$. We use $\alpha^i = \frac{1}{n^i}$ to compute the GW distances for the balanced mass constraint setting. For the remaining distances, we set $\alpha=\frac{1}{N}$, where $N$ is the median number of points across all shapes in the dataset. For the SVM experiments, we use $\exp(-\sigma D)$ as the kernel for the SVM model. Here, we normalize the matrix $D$ and choose the best $\sigma \in \{0.001, 0.01, 0.1, 1, 10, 100, 1000, 10000\}$ for each method used in order to facilitate a fair comparison of the resulting performances. We note that the resulting kernel matrix is not necessarily positive semidefinite.

In computing the pairwise distances, for the PGW method, we set $\lambda$ such that $\lambda \leq \lambda_{max} = \max_i{(|C^i|^2)}$. In particular, we compute $\lambda_{max}$ for each dataset and use $\lambda = \frac{1}{5}\lambda_{max}$ for each experiment. For UGW, we use $\varepsilon = 10^{-1}$ and $\rho_1 = \rho_2 = 1$ for both experiments. Finally, for MPGW, we set the mass-constrained term to be $\rho=\min(|\mu^i|,|\mu^j|)$ when computing the similarity between shape $\mathbb{X}^i$ and $\mathbb{X}^j$.

\textbf{Performance analysis.}
The pairwise distance matrices are visualized for each dataset in Figure \ref{fig:dist_mats}, and the confusion matrices computed with each dataset are given in Figure \ref{fig:conf_mats}. Finally, the classification accuracy with the SVM experiments is reported in Table \ref{tb:svm_accs}. The results indicate that the PGW distance is able to obtain high performance across both datasets consistently.

\begin{figure}[h!]
    \centering
    \begin{subfigure}[t]{.92\textwidth}
        \centering
        \includegraphics[width=\columnwidth]{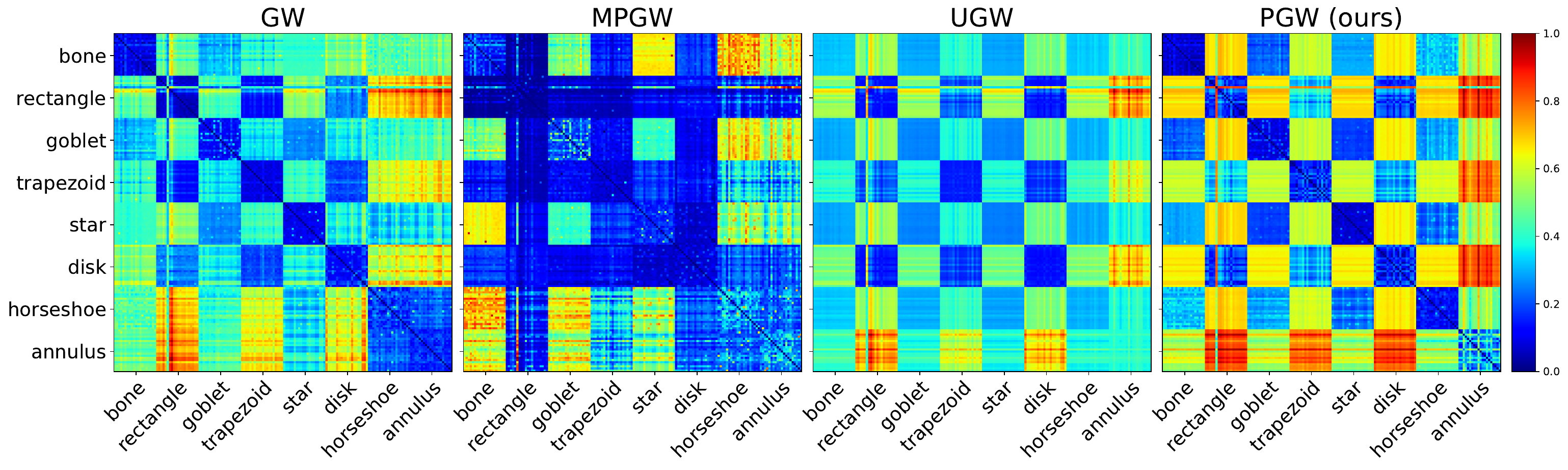}
        \caption{Dataset I}
        \label{fig:dist_data_1_2}
    \end{subfigure}
    \begin{subfigure}[t]{.92\textwidth}
        \centering
        \includegraphics[width=\columnwidth]{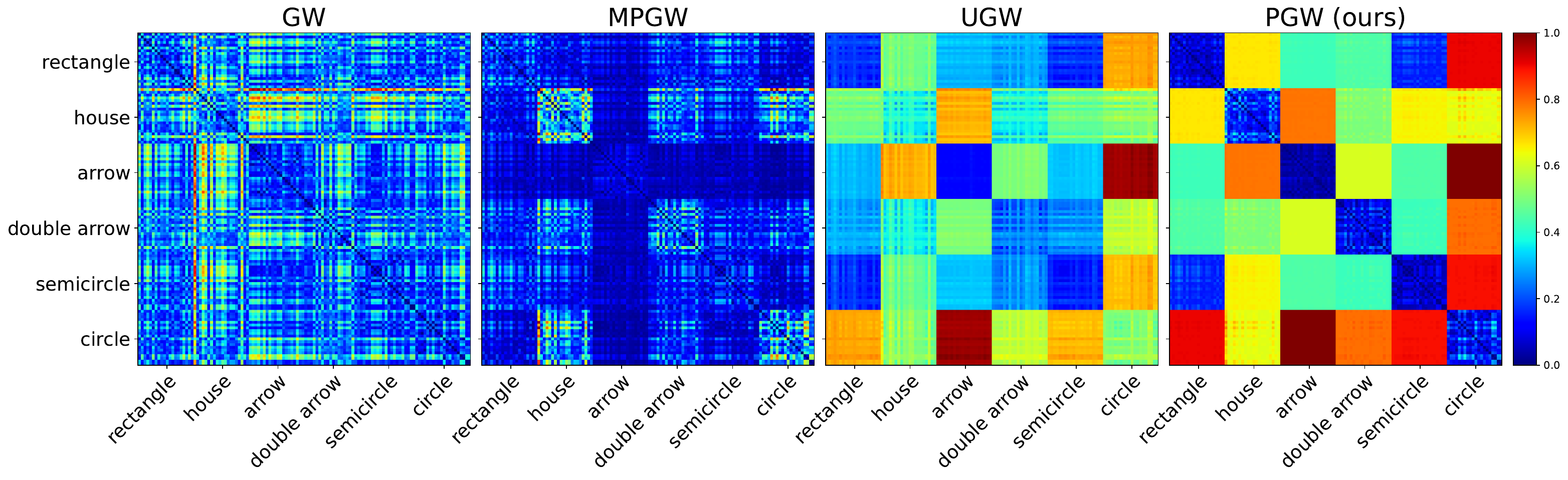}
        \caption{Dataset II}
        \label{fig:dist_data_2_2}
    \end{subfigure}
    \caption{Pairwise distance matrices computed for each dataset.}
    \label{fig:dist_mats}
\end{figure}

\begin{figure}[h!]
    \centering
    \begin{subfigure}[t]{0.92\textwidth}
        \centering
        \includegraphics[width=\columnwidth]{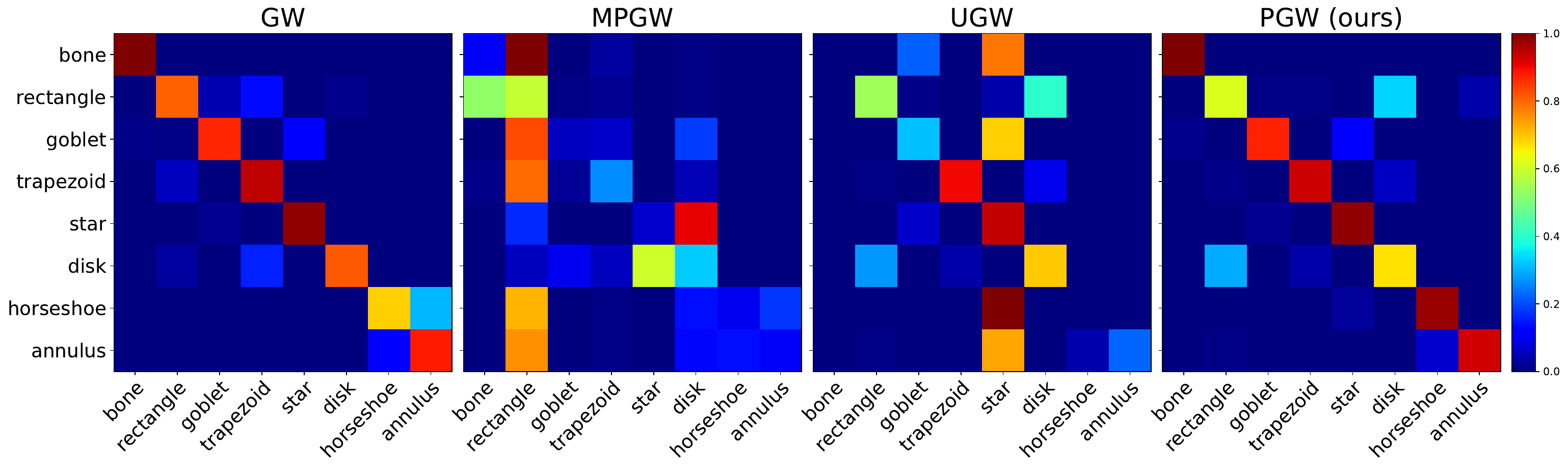}
        \caption{Dataset I}
        \label{fig:conf_data_1}
    \end{subfigure}
    
    \begin{subfigure}[t]{0.92\textwidth}
        \centering
        \includegraphics[width=\columnwidth]{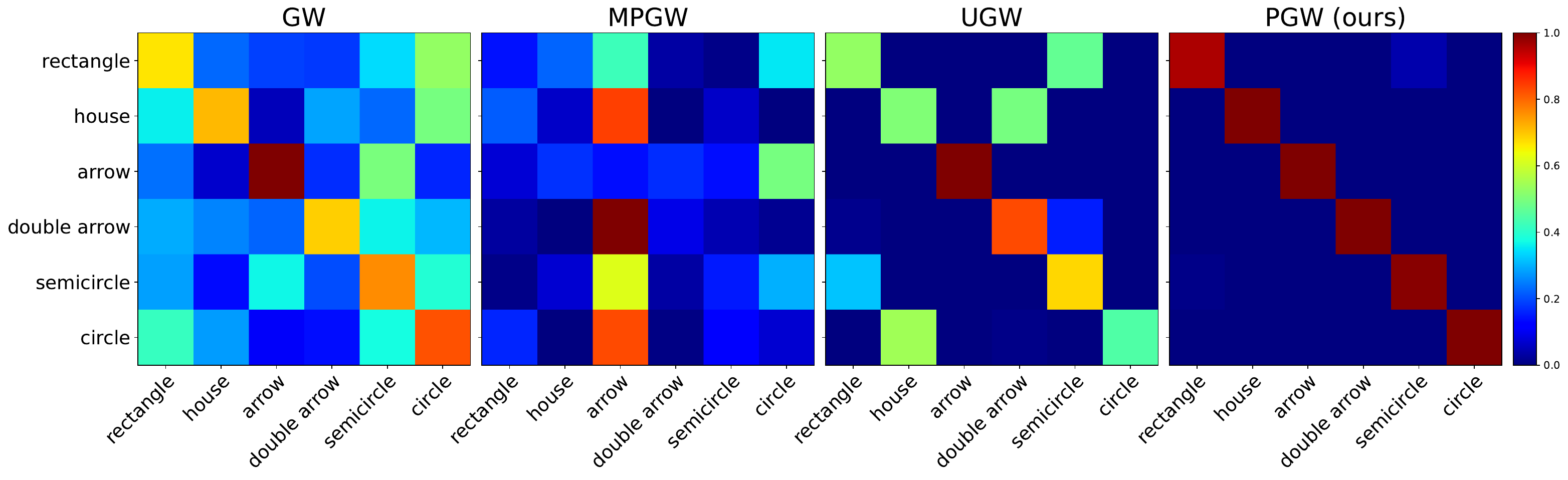}
        \caption{Dataset II}
        \label{fig:conf_data_2}
\end{subfigure}

    \caption{Confusion matrices computed from nearest neighbor classification experiments.}
    \label{fig:conf_mats}
\end{figure}

In addition, from Figure \ref{fig:dist_mats}, we observe that PGW qualitatively admits a more reasonable similarity measure compared to other methods. For example, in Dataset I, class ``bone'' and ``rectangle'' should have relatively smaller distances than ``bone'' and ``annulus''. Ideally, a reasonable distance should satisfy the following: 

$$0<d(\text{bone},\text{rectangle})< d(\text{bone},\text{anulus}).$$

However, we do not observe this relation in GW and UGW\footnote{For UGW, this is due to the Sinkhorn regularization term.}, and for the MPGW method, $$MPGW(\text{bone},\text{rectangle})\approx 0,$$ which is also undesirable. For PGW, however, we do observe this relation. Additionally, we report the wall-clock time comparison in Table \ref{tb:svm_time}.

\subsection{Additional Experiment with OT/UT Baselines}
In this subsection, we incorporate optimal transport distance \cite{Villani2003Topics} and unbalanced optimal transport distance \cite{chizat2018unbalanced} into the shape retrieval experiment.

\textbf{Dataset setup.}  
We perform shape retrieval experiments in two scenarios: the original 2D dataset and a 4D dataset. For the 4D scenario, we embed the shapes into 4D space and then apply a random rotation to each shape.

\textbf{Baselines.}  
We evaluate the performance of the following methods: Optimal Transport distance (OT) \cite{Villani2003Topics}, Unbalanced Optimal Transport (UOT) distance \cite{chizat2018unbalanced}, Gromov-Wasserstein distance (GW) \cite{gromov2001metric}, Mass-constrained Gromov-Wasserstein discrepancy (MPGW) \cite{chapel2020partial}, Unbalanced Gromov-Wasserstein discrepancy (UGW) \cite{sejourne2021unbalanced}, and our Partial Gromov-Wasserstein (PGW) distance.

The parameter settings for GW, MPGW, UGW, and PGW are described in the previous section. For the UOT method, we set the weight of Sinkhorn regularization to 0.1 and the weight of the marginal penalty to 0.2.

\textbf{Results.} In Dataset 1, when data is in the 2D space, OT achieves an accuracy of 95.6\%, while UOT achieves 73.5\%. OT's accuracy is slightly lower than GW/PGW, but it remains a strong classifier in this setting. In Dataset 2, UOT outperforms OT with an accuracy of 73.3\% compared to 55.8\%.

When the shapes are embedded into 4D space, the accuracy of OT and UOT drops significantly, ranging from 56.9\% to 79.3\%, far below the performance of GW and PGW. This decline highlights the reliance of OT/UOT on absolute coordinates, which becomes more problematic in higher dimensions. In contrast, GW-based methods (GW, MPGW, UGW, PGW) remain unaffected, as they are invariant to absolute locations.
Overall, regardless of whether the data is in 2D or 4D space, OT and UOT consistently perform worse than GW and its variants. The gap in performance becomes even more pronounced in 4D space.

\begin{table}[h!]
    \centering
    \begin{tabular}{lcc}
        \toprule
        Method & Dataset I & Dataset II \\
        \midrule
        OT ($d=2$) & 95.6\% & 55.8\% \\
        OT ($d=4$) & 79.3\% & 50.8\% \\
        UOT ($d=2$) & 73.5\% & 73.3\% \\
        UOT ($d=4$) & 56.9\% & 63.3\% \\
        GW ($d=2,4$) & 98.1\% & 80.8\% \\
        MPGW ($d=2,4$) & 23.7\% & 25.0\% \\
        UGW ($d=2,4$) & 89.4\% & 90.0\% \\
        PGW (ours, $d=2,4$) & 96.2\% & 100\% \\
        \bottomrule
    \end{tabular}
    \caption{Performance comparison across different methods and datasets.}
    \label{tb:shape_retrieval_2}
\end{table}

\section{Other Numerical Implementations}
\subsection{Initial Methods}
In this experiment, we discuss several different methods to define the initial guess in the Frank-Wolfe algorithm proposed in this paper. Note some of these methods have been applied in FW algorithms/Sinkhorn solvers in classical GW \cite{memoli2011gromov}, Mass-constraint GW \cite{chapel2020partial} and Unbalanced GW \cite{sejourne2021unbalanced}

Given two mm-spaces $\mathbb{X}=(X,d_X,\mu),\mathbb{X}=(Y,d_Y,\nu)$, we consider the two cases:

Case 1: Dimensions of $X$ and $Y$ are the same. Note, in this case, we can define classical OT/partial OT/unbalanced OT between $\mu$ and $\nu$. Thus, \cite{chapel2020partial} proposed the following ``\textbf{POT initilization}'' method: 

\begin{equation}
\gamma^{(1)}\gets \arg\min_{\gamma\in \Gamma_{\leq,\pi}(\mathrm{p},\mathrm{q})}\langle L(X,Y),\gamma\rangle_F,
\label{eq:pot_init}
\end{equation}
where $L(X,Y)\in \mathbb{R}^{n\times m}, \,  (L(X,Y))_{ij}=\|x_i-y_j\|^p$ for some fixed $p\ge 1$ and 
\begin{align}
&\Gamma_{\leq ,\pi}(\mathrm{p},\mathrm{q}):=\{\gamma\in\mathbb{R}_+^{n\times m}: (\gamma^\top 1_{n})_j\in \{q_j^Y,0\},\forall j; \gamma 1_{m}\leq \mathrm{p}, |\gamma|=\pi\}.\label{eq:pot_init_constraint}
\end{align}
The above problem can be solved by a Lasso ($L^1$ norm) regularized OT solver. 

Case 2: Dimensions of $X$ and $Y$ are different. The above technique can not be applied since the problem \eqref{eq:pot_init} (in particular $L(X, Y)$) is not well-defined. 

In  this case, \cite{sejourne2021unbalanced} introduced the \textbf{``FLB-UOT''} method: 
\begin{align}
\gamma^{(1)}\gets \arg\min_{\gamma\in \Gamma_\leq(\mathrm{p},\mathrm{q})}&\int_{X\times Y}|s_{X,p}(x)-s_{Y,p}(y)|^pd\gamma(x,y)+\lambda(D_{KL}(\gamma_1,\mathrm{p})+D_{KL}(\gamma_2,\mathrm{q})), \label{eq: FLB UOT}  
\end{align}
where $s_{X,p}(x)=\int_{X}|x-x'|^pd\mu(x)$ and $s_{Y,p}$ is defined similarly. The problem \eqref{eq: FLB UOT} is called Hellinger Kantorovich, which is a classical unbalanced optimal transport problem. The Sinkhorn solver can solve it \cite{chizat2018scaling}. 

Analog to the above method, we propose the third method, called ``\textbf{FLB-POT}'' (first lower bound-partial optimal transport)
\begin{align}
\gamma^{(1)}\gets\arg\min_{\gamma\in \Gamma_\leq(\mathrm{p},\mathrm{q})}&\int_{X\times Y}|s_{X,p}(x)-s_{Y,p}(y)|^2d\gamma(x,y)+\lambda(|\mathrm{p}-\gamma_1|+|\mathrm{q}-\gamma_2|). \label{eq: FLB POT}  
\end{align}
The above problem is a partial OT problem and can be solved by classical linear programming \cite{caffarelli2010free}. 

\subsection{Parameter Setting for PGW}
Setting parameter $\lambda$ in PGW is important in the numerical implementation. There are two scenarios to consider:
\begin{itemize}
    \item  Both source and target measures contain outliers:

    The parameter $\lambda$ acts as an upper bound for the transported distance. Specifically, if  
     $$
     \|d_X(x,x') - d_Y(y,y')\|^2 \geq 2\lambda,
     $$
     then either $(x, y)$ or $(x', y')$ will not be transported.  
   - When the distance between outliers and clean data is large, and the pairwise distances within the clean data are relatively small, $\lambda$ should be set to lie between these two scales.
   \item Only one measure contains outliers:
   
   As stated in Lemma \ref{lem:large_lambda}, $\lambda$ can be set sufficiently large to satisfy:  
     $$
     2\lambda \geq \max_{x,x' \in X, y,y' \in Y} \|d_X(x,x') - d_Y(y,y')\|^2.
     $$

In the interpolation and shape retrieval experiments, we fall under the second scenario, which does not require significant parameter tuning.
\end{itemize}



\section{Wall-Clock Time Comparison for Partial GW Solvers}\label{sec:wall_time}

In this section, we present the wall-clock time comparison between our method Algorithms \ref{alg:pgw_v1}, \ref{alg:pgw_v2}, the Frank-Wolf algorithm proposed in \cite{chapel2020partial}, and its Sinkhorn version \cite{peyre2016gromov,chapel2020partial}. Note that these two baselines solve a mass constraint version of the PGW problem, which we refer to as the ``MPGW'' problem. The proposed PGW formulation in this paper can be regarded as a ``Lagrangian formulation'' of MPGW\footnote{Due to the non-convexity of GW, we do not have a strong duality in some of the GW representations. Thus, the Lagrangian form is not a rigorous description.} formulation to the PGW problem defined in \eqref{eq:pgw}.
In this paper, we call these two baselines the ``MPGW algorithm'' and the ``Sinkhorn PGW algorithm''.  

\textbf{Numerical details.}
The data is generated as follows: let $\mu=\text{Unif}([0,2]^2)$ and $\nu=\text{Unif}([0,2]^3)$, we select i.i.d. samples $\{x_i\sim \mu\}_{i=1}^n, \{y_j\sim \nu\}_{j=1}^m$, where $n$ is selected from $[10,50,100,150,...,10000]$ and $m=n+100$,
$\mathrm{p}=1_{n}/m , \mathrm{q}=1_m/m $. For each $n$, we set $\lambda=0.2, 1.0, 10.0$. The mass constraint parameter for the algorithm in  \cite{chapel2020partial}, and Sinkhorn is computed by the mass of the transportation plan obtained by Algorithm \ref{alg:pgw_v1} or \ref{alg:pgw_v2}. The runtime results are shown in Figure \ref{fig:time}. 

Regarding the acceleration technique, for the POT problem in step 1, our algorithms and the MPGW algorithm apply the linear programming solver provided by Python OT package \cite{flamary2021pot}, which is written in C++. The Sinkhorn algorithm from Python OT does not have an acceleration technique. Thus, we only test its wall-clock time for $n\leq 2000$. The data type is a 64-bit float number.

From Figure \ref{fig:time}, we can observe the Algorithms \ref{alg:pgw_v1}, \ref{alg:pgw_v2}, and MPGW algorithm have a similar order of time complexity. However, using the column/row-reduction technique for the POT computation discussed in previous sections and the fact the convergence behaviors of Algorithms \ref{alg:pgw_v1} and \ref{alg:pgw_v2} are similar to the MPGW algorithm, we observe that the proposed algorithms \ref{alg:pgw_v1}, \ref{alg:pgw_v2} admits a slightly faster speed than MPGW solver.  

\begin{figure}[h]
    \label{fg: time}
    \centering    
    \includegraphics[width=\columnwidth]{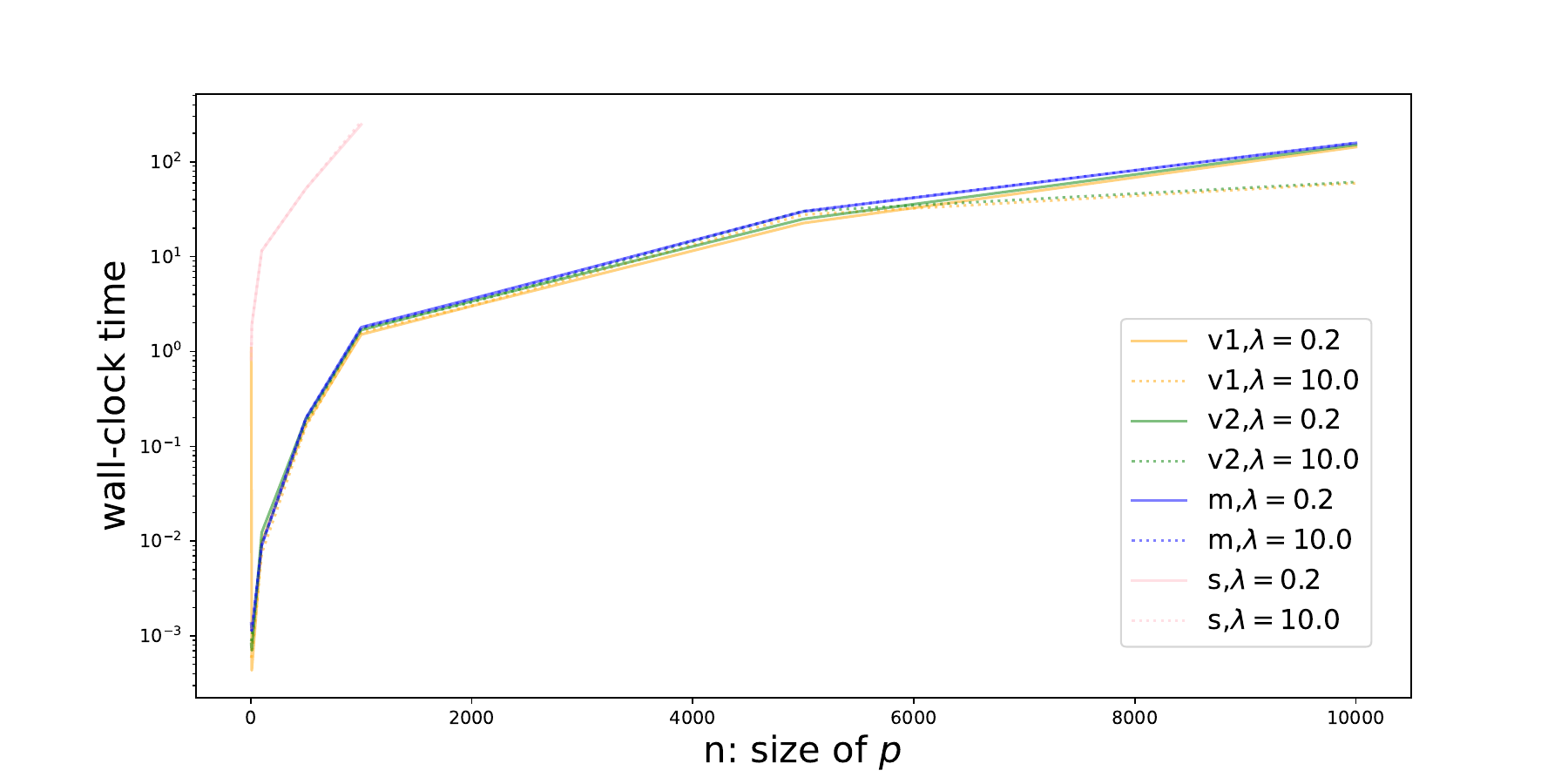}
    \caption{We test the wall-clock time of our Algorithm \ref{alg:pgw_v1} and Algorithm \ref{alg:pgw_v2}, the  MPGW solver (Algorithm 1 in \cite{chapel2020partial}) , and the Sinkhorn algorithm \cite{peyre2016gromov}. We denote these methods as ``v1'', ``v2'', ``m'' and ``s'', respectively. The linear programming solver applied in the first three methods is from POT \cite{flamary2021pot}, which is written in C++. The maximum number of iterations for all the methods is set to be $1000$. The maximum iteration for OT/OPT solvers is set to be $300n$. The maximum Sinkhorn iteration is set to be $1000$. The convergence tolerance for the Frank-Wolfe algorithm and the Sinkhorn algorithm are set to be $1e-5$. 
    To achieve the best performance, we set the number of dummy points to 1 for both MPGW and PGW.
}
    \label{fig:time}
\end{figure}

\section{Positive Unlabeled  Problem}

\subsection{Problem Setup}
Positive unlabeled (PU) learning  \cite{bekker2020learning,elkan2008learning,kato2018learning}  is a semi-supervised binary classification problem for which the training set only contains positive samples. 
In particular, suppose there exists a fixed unknown overall distribution over triples $(x,o,l)$, where $x$ is data, $l\in \{0,1\}$ is the label of $x$, $o\in\{0,1\}$ where $o=1$, $o=0$ denote that $l$ is observed or not, respectively.  In the PU task, the assumption is that only positive samples' labels can be observed, i.e.,
$\text{Prob}(o=1|x,l=0)=0.$
Consider training labeled data $X^{pu}=\{(x_i^{pu},l)\}_{i=1}^n\subset \{x: o=1\}$ and testing data $X^{un}=\{x_j^{un}\}_{j=1}^m\subset \{x: o=0\}$, where $x_ip_i^X\in \mathbb{R}^{d_1}, x^u_j\in \mathbb{R}^{d_2}$. In the classical PU learning setting, $d_2=d_1$. However, in \cite{sejourne2021unbalanced}, this assumption is relaxed. The goal is to leverage $X^p$ to design a classifier $\hat{l}: x^u\to \{0,1\}$ to predict $l(x^u)$ for all $x^u\in X^u$.\footnote{In the classical setting, the goal is to learn a classifier for all $x$. In this experiment, we follow the setting in \cite{sejourne2021unbalanced}.}

Following \cite{elkan2008learning,chapel2020partial,sejourne2021unbalanced}, in this experiment, we assume that the ``select completely at random'' (SCAR) assumption holds: 
$\text{Prob}(o=1|x,l=1)=\text{Prob}(o=1|l=1)$. 
In addition, we use $\pi=\text{Prob}(l=1)\in[0,1]$ to denote the ratio of positive samples in testing set\footnote{In the classical setting, the prior distribution $\pi$ is the ratio of positive samples of the original dataset. For convenience, we ignore the difference between this ratio in the original dataset and the test dataset.}. Following the PU learning setting in \cite{kato2018learning,hsieh2019classification, chapel2020partial,sejourne2021unbalanced}, we assume
$\pi$ is known. In all the PU learning experiments, we fix $\pi=0.2$. 

\subsection{Methodology} Similar to \cite{chapel2020partial} our method is designed as follows: 
We set $\mathrm{p}\in \mathbb{R}^n,\mathrm{q}\in\mathbb{R}^m$ as 
$p_i^X=\frac{\pi}{n},  i\in[1:n];\quad q_j^Y=\frac{1}{m}, j\in[1:m].$
Let $\mathbb{X}^p=(X^p,\|\cdot\|_{d_1},\sum_{i=1}^np_i^X\delta_{x_i}),\mathbb{X}^u=(X^u,\|\cdot\|_{d_2},\sum_{j=1}^nq_j^Y\delta_{y_j})$. 
We solve the partial GW problem $PGW_{\lambda}(\mathbb{X}^p,\mathbb{X}^u)$ and suppose $\gamma$ is a solution. Let $\gamma_2=\gamma^\top1_n$. The classifier $\hat{l}$ is defined by the indicator function
\begin{align}
\hat{l}_\gamma(x^u)=\mathbbm{1}_{\{x^u: \, \gamma_2(x^u)\ge \mathrm{quantile}\}}, \label{eq: l_hat}
\end{align}
where $\mathrm{quantile}$ is the quantile value of $\gamma_2$ according to $1-\pi$.

Regarding the initial guess $\gamma^{(1)}$,
\cite{chapel2020partial} proposed a POT-based approach when $X$ and $Y$ are sampled from the same domain, i.e., $d_1=d_2$, which we refer to as ``POT initialization.''

When $X,Y$ are sampled from different spaces, that is, $d_1\neq d_2$, the above technique \eqref{eq:pot_init} is not well-defined. Inspired by \cite{memoli2011gromov,sejourne2021unbalanced}, we propose the following ``first lower bound-partial OT'' (FLB-POT) initialization: 
\begin{align}
  \gamma^{(1)}=\arg\min_{\gamma\in \Gamma_\leq(\mathrm{p},\mathrm{q})}&\int_{X\times Y}|s_{X,2}(x)-s_{Y,2}(y)|^2d\gamma(x,y)\nonumber+\lambda(|\mathrm{p}-\gamma_1|+|\mathrm{q}-\gamma_2|), \label{eq: FLB POT}  
\end{align}
where $s_{X,2}(x)=\int_{X}|x-x'|^2d\mu(x)$ and $s_{Y,2}$ is defined similarly. The above formula is analog to Eq. (7) in \cite{sejourne2021unbalanced}, which is designed for the unbalanced GW setting. To distinguish them, in this paper, we call Eq. (7) in 
\cite{sejourne2021unbalanced} as ``FLB-UOT initialization''.

\subsection{Datasets} The datasets include \href{https://pytorch.org/vision/stable/generated/torchvision.datasets.MNIST.html}{MNIST}, \href{https://pytorch.org/vision/main/generated/torchvision.datasets.EMNIST.html}{EMNIST}, and the following three domains of \href{https://faculty.cc.gatech.edu/~judy/domainadapt/#datasets_code}{Caltech Office}:
Amazon (A), Webcam (W), and DSLR (D) \cite{saenko2010adapting}. For each domain, we select the SURF features \cite{saenko2010adapting} and DECAF features \cite{donahue2014decaf}.
For MNIST and EMNIST, we train an auto-encoder, respectively, and the embedding space dimension is $4$ and $6$, respectively. See Figure \ref{fig:pu_data} for the TSNE visualization of these datasets. 

\begin{figure}  
\begin{subfigure}[b]{0.52\linewidth}
    \centering
\includegraphics[height=0.22\textheight]{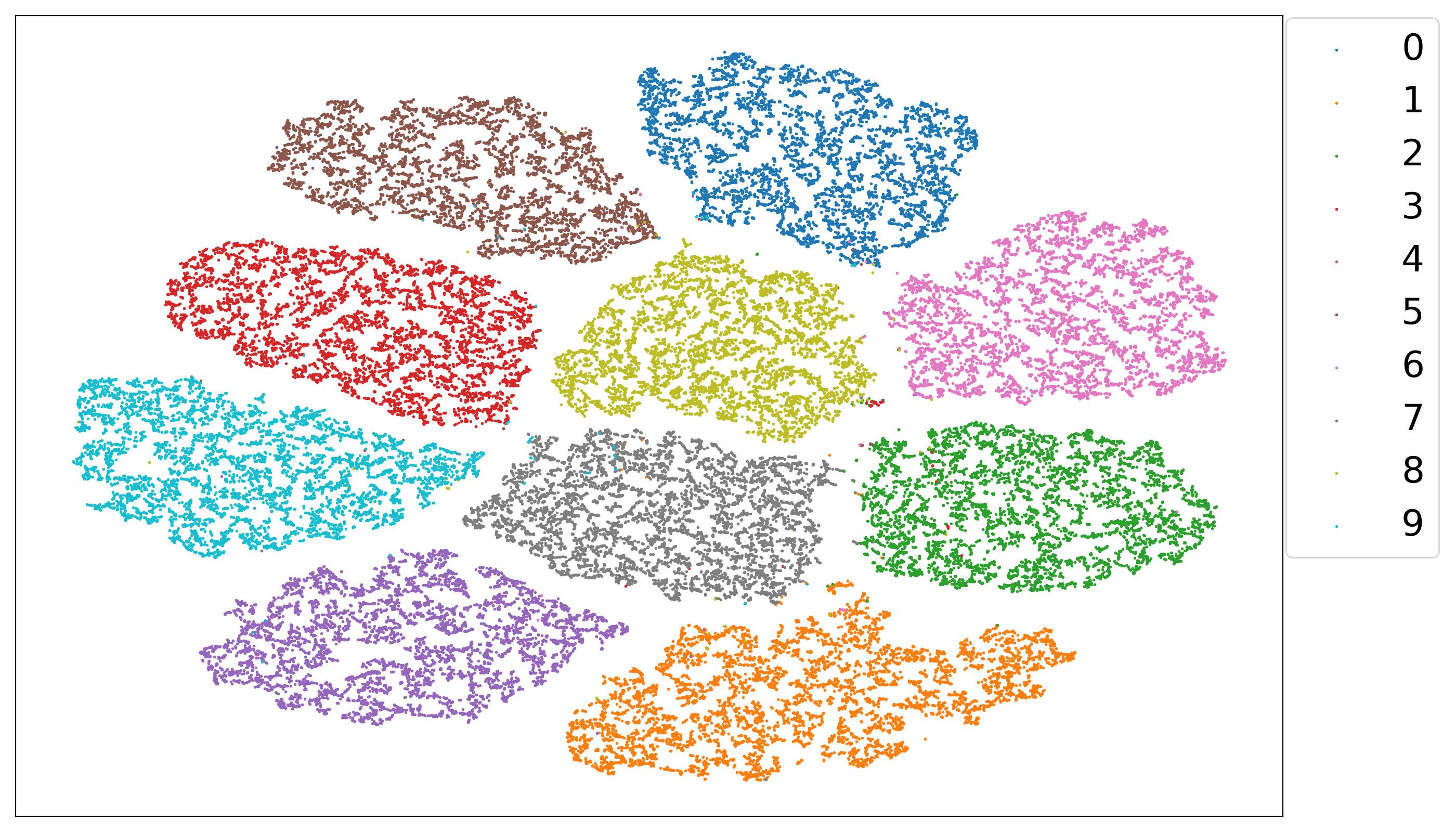} 
    \caption{MNIST} 
  \end{subfigure} 
  \begin{subfigure}[b]{0.52\linewidth}
    \centering
    \includegraphics[height=0.22\textheight]{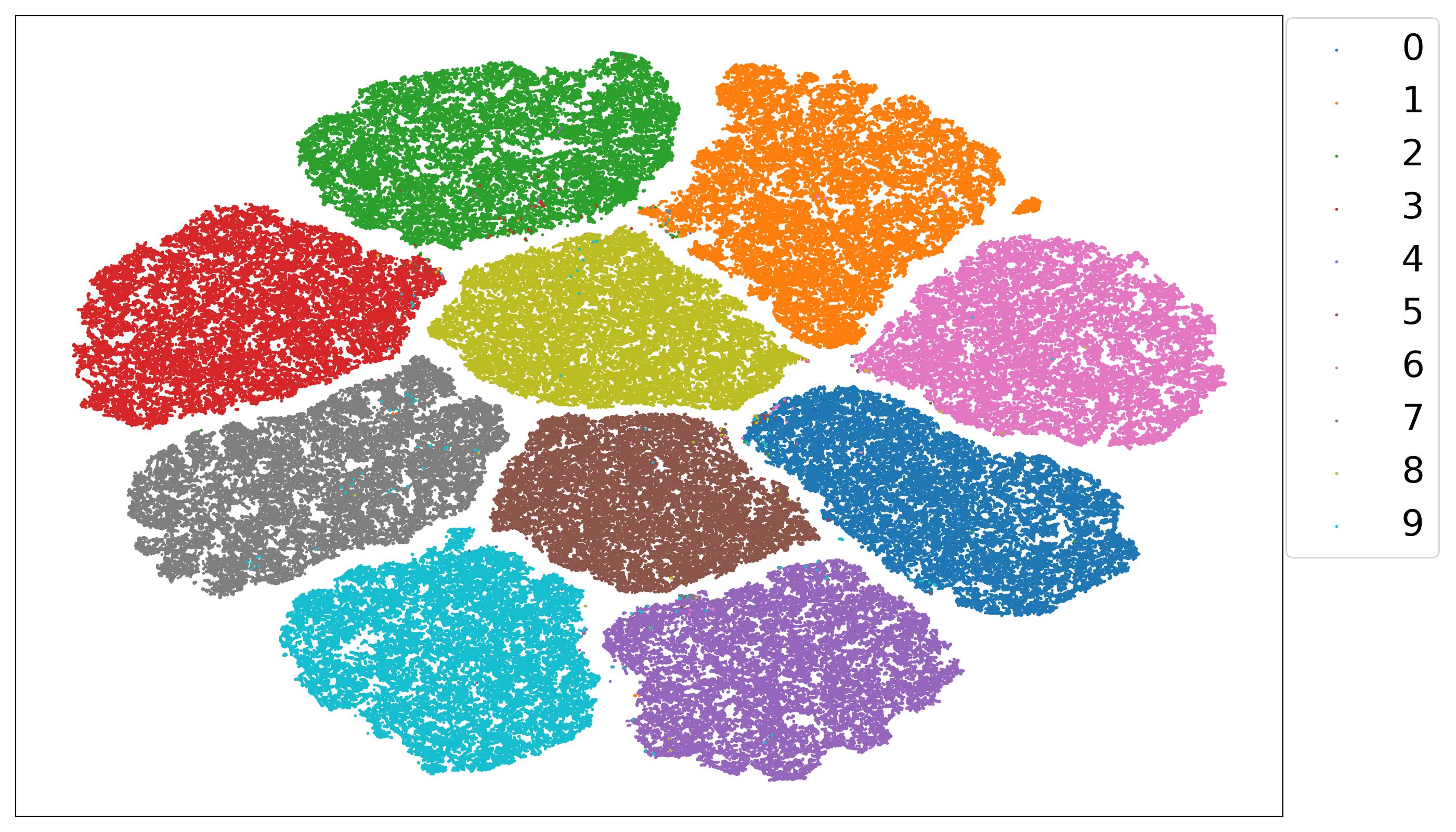} 
    \caption{EMNIST}  
  \end{subfigure} 
  
  \begin{subfigure}[b]{0.52\linewidth}
    \centering
    \includegraphics[height=0.22\textheight]{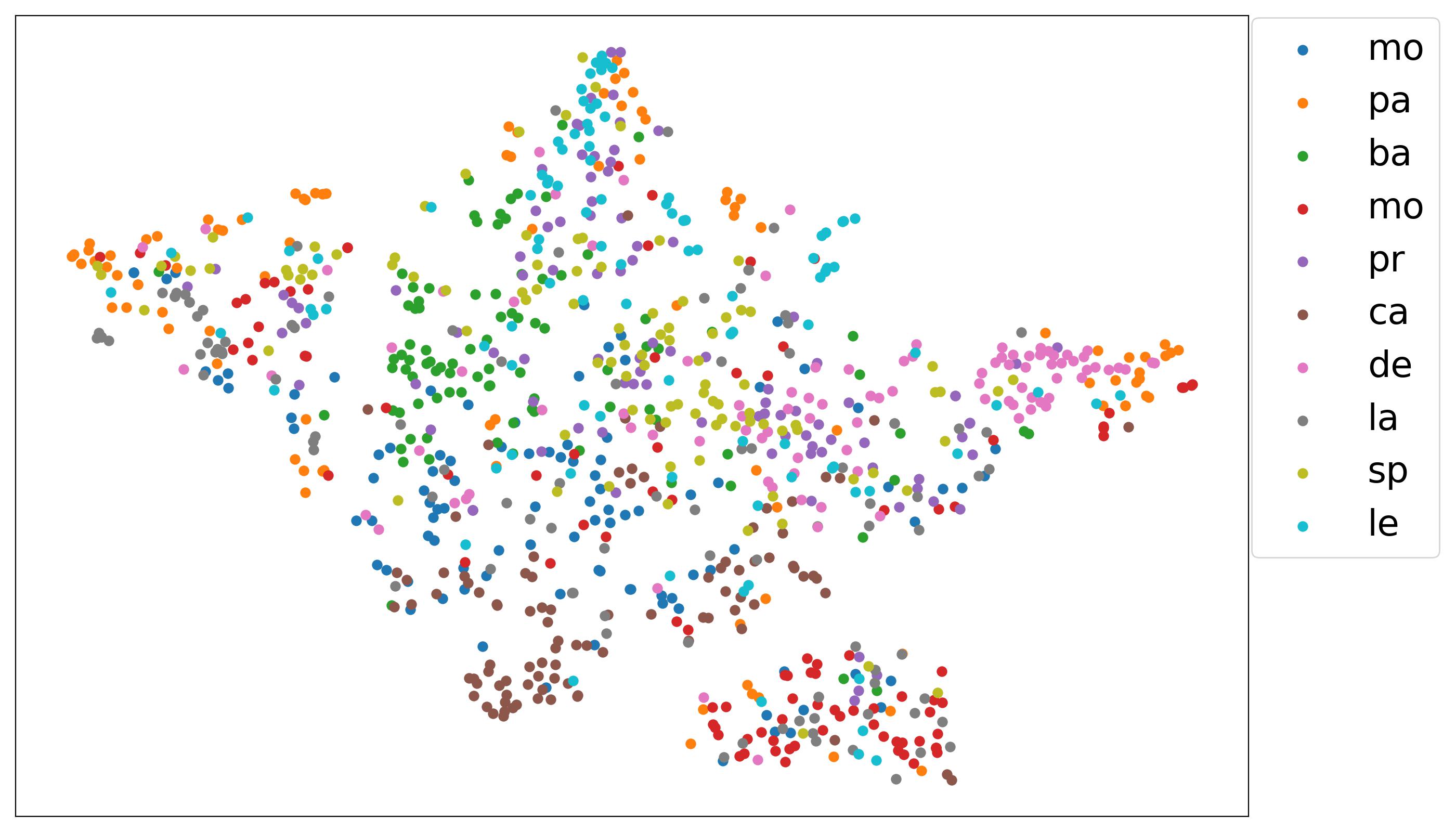} 
  \caption{Surf(A)}  
  \end{subfigure} 
  \begin{subfigure}[b]{0.52\linewidth}
    \centering
    \includegraphics[height=0.22\textheight]{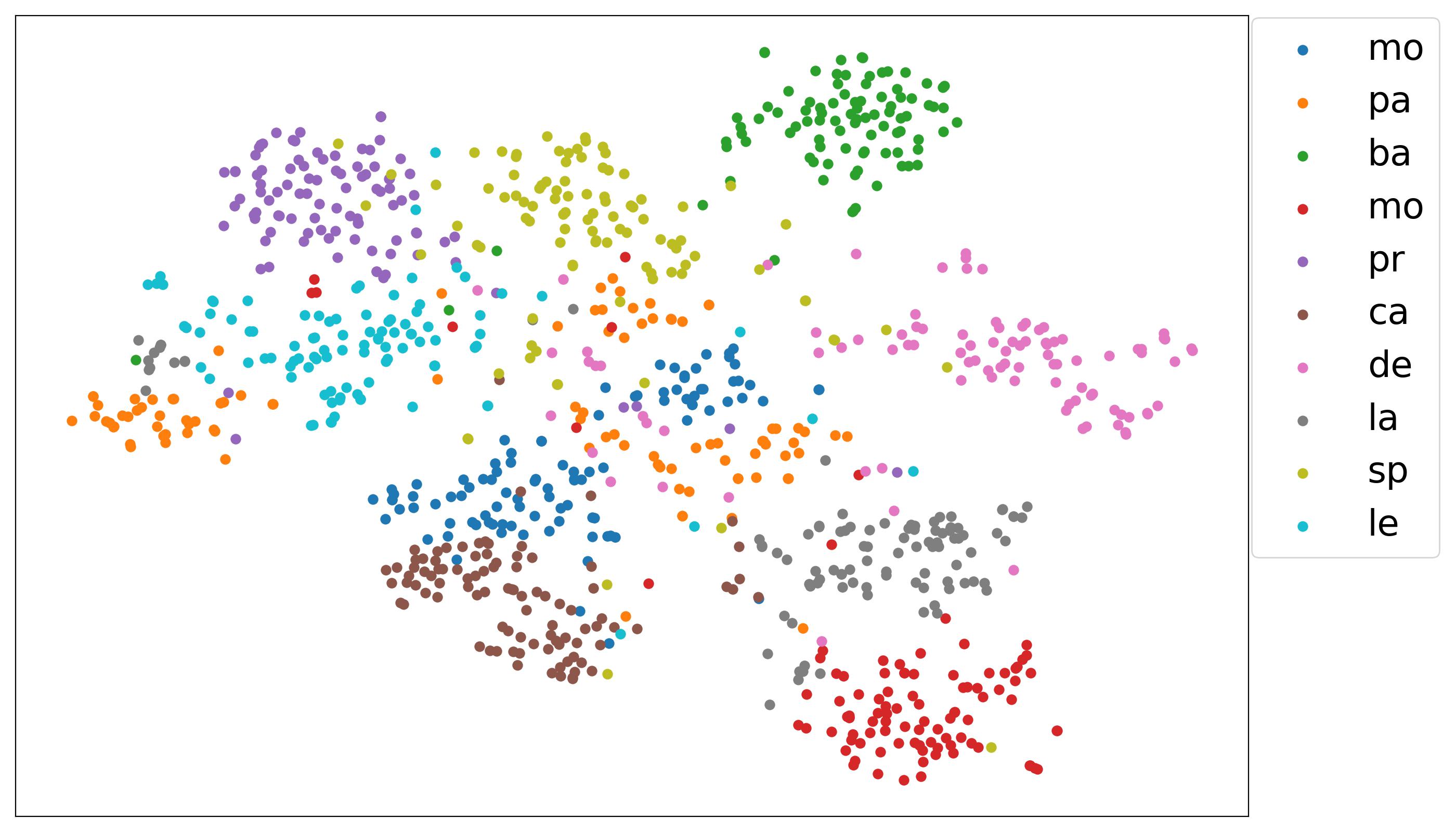} 
    \caption{Decaf(A)} 
    \label{fig: amazon decaf} 
  \end{subfigure} 
  \begin{subfigure}[b]{0.52\linewidth}
    \centering
    \includegraphics[height=0.22\textheight]{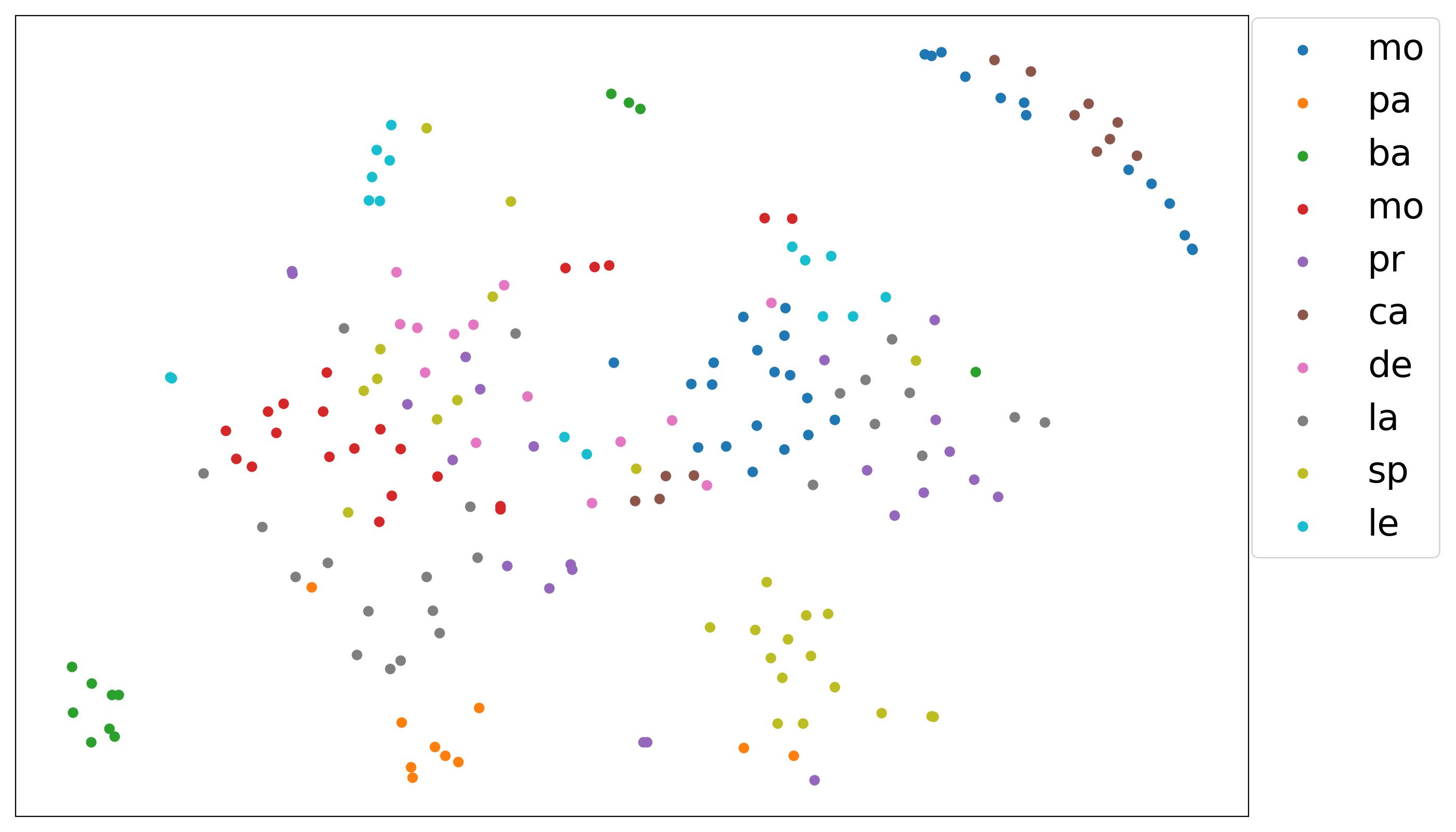}
    \caption{Surf(D)} 
  \end{subfigure}
  \begin{subfigure}[b]{0.52\linewidth}
    \centering
    \includegraphics[height=0.22\textheight]{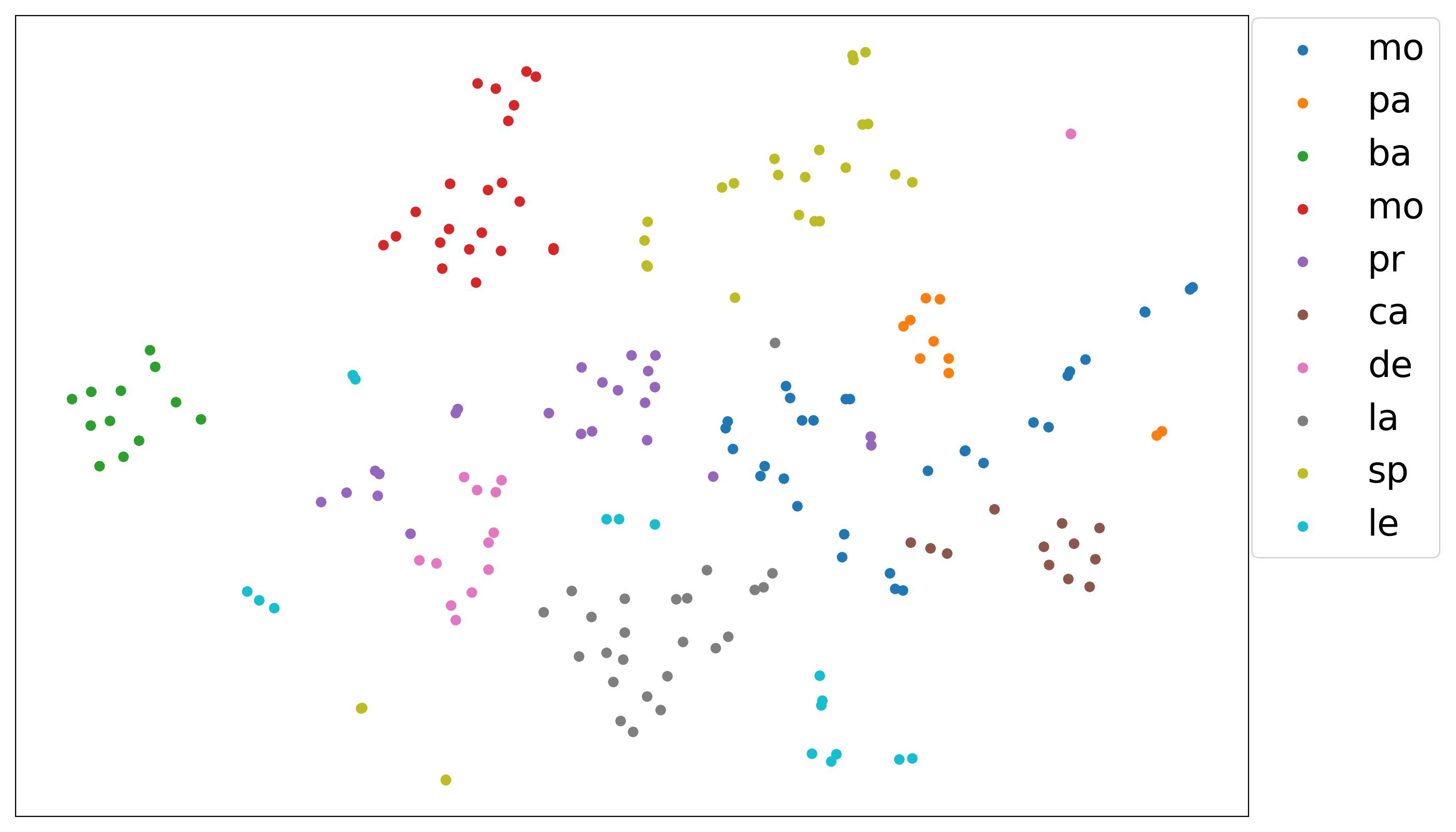} 
    \caption{Decaf(D)} 
    \label{fig: dslr decaf} 
  \end{subfigure} 
   \begin{subfigure}[b]{0.52\linewidth}
    \centering
    \includegraphics[height=0.22\textheight]{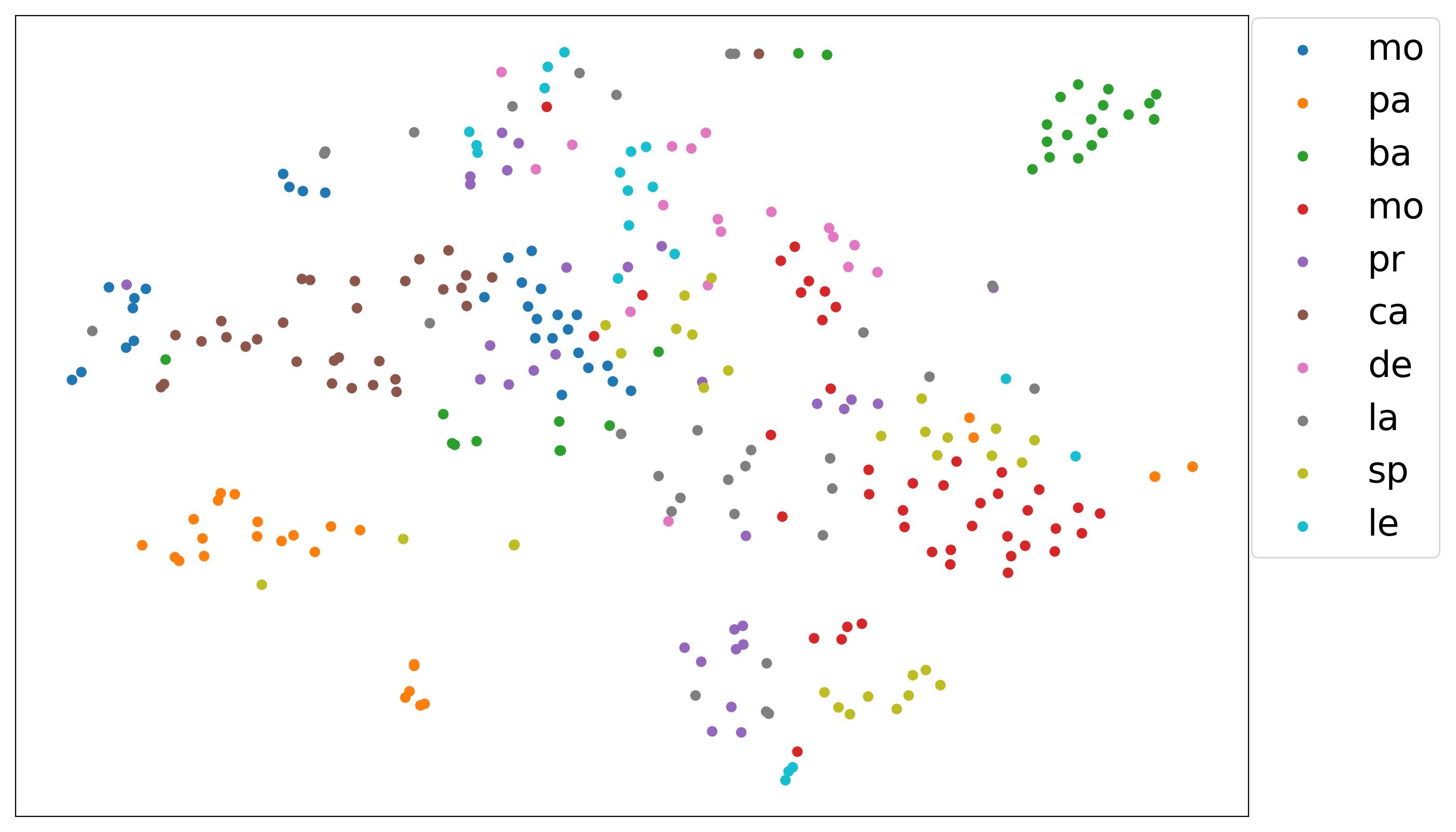} 
    \caption{Surf(W)} 
     
  \end{subfigure}
  \begin{subfigure}[b]{0.52\linewidth}
    \centering
    \includegraphics[height=0.22\textheight]{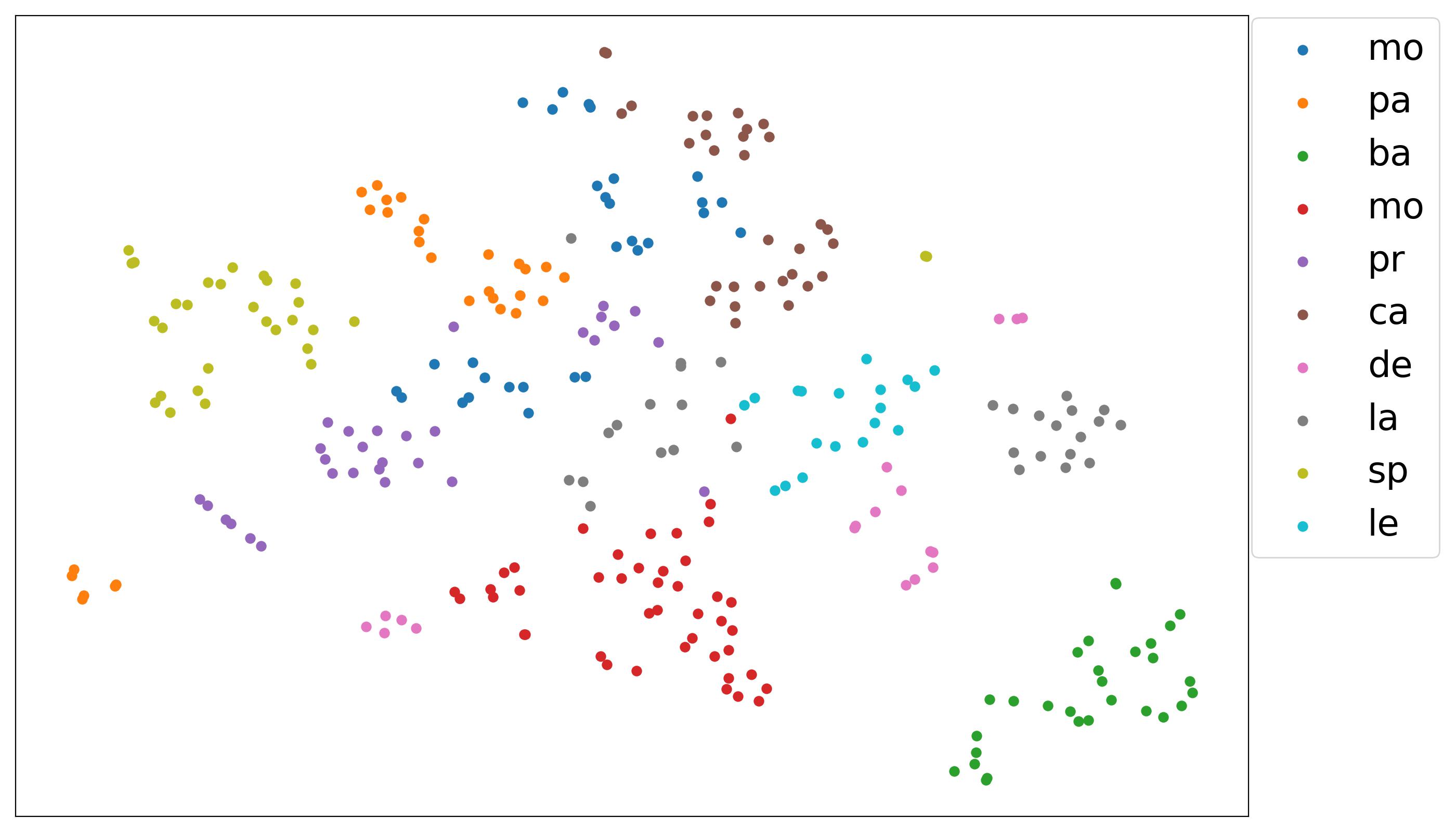} 
    \caption{Decaf(w)} 
    \label{fig: webcam decaf} 
  \end{subfigure} 
\caption{t-SNE visualization for datasets \href{https://pytorch.org/vision/stable/generated/torchvision.datasets.MNIST.html}{MNIST}, \href{https://pytorch.org/vision/main/generated/torchvision.datasets.EMNIST.html}{EMNIST}, \href{https://faculty.cc.gatech.edu/~judy/domainadapt/\#datasets_code}{Caltech Office}.
}
\label{fig:pu_data} 
\end{figure}


\subsection{Numerical Details and Performance} 
\textbf{Accuracy Comparison.}

In Table \ref{tb:pu_accuracy1} and Table \ref{tb:pu_accuracy2}, we present the accuracy results for the MPGW, UGW, and the proposed PGW methods when using three different initialization methods: POT, FLB-UOT, and FLB-POT.

Following \cite{chapel2020partial},  in the MPGW and PGW methods, we incorporate the prior knowledge $\pi$ into the definition of $p$ and $q$. Thus, it is sufficient to set $mass=\pi$ for MPGW and choose a sufficiently large value for $\lambda$ in the PGW method. This configuration ensures that the mass matched in the target domain $\mathcal{Y}$ is exactly equal to $\pi$. However, in the UGW method \cite{sejourne2021unbalanced}, the setting is $p=\frac{1}{n}1_n$ and $q=\frac{1}{m}1_m$. 

Overall, all methods show improved performance in MNIST and EMNIST datasets. One possible reason for this could be the better separability of the embeddings in MNIST and EMNIST, as illustrated in Figure \ref{fig:pu_data}. Additionally, since MPGW and PGW incorporate information from $r$ into their formulations, they exhibit slightly better accuracy in many experiments.


\textbf{Numerical details.} In this experiment, to prevent unexpected convergence to local minima in the Frank-Wolf algorithms, we manually set $\alpha=1$ during the line search step for both MPGW and PGW methods.

For the convergence criteria, we set the tolerance term for Frank-Wolfe convergence and the main loop in the UGW algorithm to be $1e-5$. Additionally, the tolerance for Sinkhorn convergence in UGW was set to $1e-6$. The maximum number of iterations for the POT solver in PGW and MPGW was set to $500n$. In addition, for MPGW, we set $\text{mass}=0.2$ and for PGW method, based on lemma \ref{lem:large_lambda}, we set $\lambda$ to be constant such that $2\lambda\ge (\max(|C^X|)^2+\max(|C^Y|)^2)$. 
For UGW, as we directly apply the numerical method in \cite{sejourne2023unbalanced}, where the prior knowledge $\pi=0.2$ is not Incorporated in the setting of $p$ and $q$. Thus, in each experiment, we test different parameters $(\rho,\rho_2,\epsilon)$ and select the ones that result in transported mass close to $\pi$.

Regarding data types, we used 64-bit floating-point numbers for MPGW and PGW and 32-bit floating-point numbers for UGW.

For the MNIST and EMNIST datasets, we set $n=1000$ and $m=5000$. In the Surf(A) and Decaf(A) datasets, each class contained an average of 100 samples. To ensure the SCAR assumption, we set $n=1/2*100=50$ and $m=250$. Similarly, for the Surf(D) and Decaf(D) datasets, we set $n=15$ and $m=75$. Finally, for Surf(W) and Decaf(W), we used $n=20$ and $m=100$.

\textbf{Wall-clock time}
In Table \ref{tb:pu_time}, we provide a comparison of wall-clock times for the MNIST and EMNIST datasets.

\section{Compute Resources}\label{sec:computation_resource}

All experiments presented in this paper are conducted on a computational machine with an AMD EPYC 7713 64-Core Processor, 8 $\times$ 32GB DIMM DDR4, 3200 MHz, and an NVIDIA RTX A6000 GPU.

\begin{table}[h!]
\vskip 0.15in
\begin{center}
\begin{sc}
\begin{tabular}{lccccc}
\toprule
Dataset & Init Method & Init Accuracy & MPGW & UGW & PGW (ours) \\
\midrule
M $\to$ M   & POT & 100\% & 100\% & 95\% & 100\% \\
M $\to$ M   & FLB-U & 75\% & 96\% & 95\% & 96\% \\
M $\to$ M   & FLB-P & 75\% & 99\% & 95\% & 99\% \\
M $\to$ EM  & FLB-U & 78\% & 94\% & 95\% & 94\% \\
M $\to$ EM  & FLB-P & 78\% & 94\% & 95\% & 94\% \\
EM $\to$ M  & FLB-U & 75\% & 97\% & 96\% & 97\% \\
EM $\to$ M  & FLB-P & 75\% & 97\% & 96\% & 97\% \\
EM $\to$ EM & POT & 100\% & 100\% & 95\% & 100\% \\
EM $\to$ EM & FLB-U & 78\% & 94\% & 95\% & 94\% \\
EM $\to$ EM & FLB-P & 78\% & 95\% & 95\% & 95\% \\
\bottomrule
\end{tabular}
\end{sc}
\end{center}
\vskip -0.0in
\caption{Accuracy comparison of the MPGW, UGW, and the proposed PGW method on PU learning. Here, `M' denotes MNIST, and `EM' denotes EMNIST.}
\label{tb:pu_accuracy1}
\end{table}

\clearpage

\begin{table}
\vskip 0.15in
\begin{center}
\begin{sc}
\begin{tabular}{ccccccc}
\toprule
Source & Target & Init Method & Init Time & MPGW & UGW & PGW (ours) \\
\midrule
M(1000) & M(5000) & POT & 0.5 & \textbf{7.2} & 152.0 & 7.4 \\
M(1000) & M(5000) & FLB-U & 0.02 & 30.5 & 152.6 & \textbf{27.8} \\
M(1000) & M(5000) & FLB-P & 0.5 & 27.8 & 144.9 & \textbf{26.9} \\
EM(1000) & EM(5000) & POT & 0.5 & \textbf{7.3} & 157.3 & 7.5 \\
EM(1000) & EM(5000) & FLB-U & 0.02 & 30.0 & 181.8 & \textbf{29.9} \\
EM(1000) & EM(5000) & FLB-P & 0.5 & \textbf{22.2} & 155.1 & 22.3 \\
M(1000) & EM(5000) & FLB-U & 0.02 & \textbf{34.0} & 157.9 & 34.4 \\
M(1000) & EM(5000) & FLB-P & 0.5 & \textbf{34.9} & 155.5 & 35.0 \\
EM(1000) & M(5000) & FLB-U & 0.02 & 24.3 & 139.3 & \textbf{22.2} \\
EM(1000) & M(5000) & FLB-P & 0.5 & 32.0 & 162.7 & \textbf{29.9} \\
\midrule
M(2000) & M(10000) & POT & 1.7 & \textbf{31.1} & 1384.8 & 32.1 \\
M(2000) & M(10000) & FLB-U & 0.1 & 209.0 & 1525.8 & \textbf{192.5} \\
M(2000) & M(10000) & FLB-P & 1.7 & 208.0 & 1418.4 & \textbf{192.1} \\
M(2000) & EM(10000) & FLB-U & 0.1 & 165.1 & 1606.1 & \textbf{164.2} \\
M(2000) & EM(10000) & FLB-P & 1.7 & 224.1 & 1420.7 & \textbf{223.7} \\
EM(2000) & M(10000) & FLB-U & 0.1 & 149.1 & 1426.5 & \textbf{138.1} \\
EM(2000) & M(10000) & FLB-P & 1.7 & 113.9 & 1407.6 & \textbf{103.9} \\
EM(2000) & EM(10000) & POT & 1.6 & \textbf{32.4} & 1445.9 & 33.4 \\
EM(2000) & EM(10000) & FLB-U & 0.1 & \textbf{233.0} & 1586.3 & 233.9 \\
EM(2000) & EM(10000) & FLB-P & 1.8 & \textbf{142.1} & 1620.6 & \textbf{142.1} \\
\bottomrule
\end{tabular}
\end{sc}
\end{center}
\caption{In this table, we present the wall-clock time for the MPGW, UGW, and the proposed PGW method, as well as three different initialization methods (POT, FLB-UOT, FLB-POT). In the ``Source'' (or ``Target'') column, M (or EM) denotes the MNIST (or EMNIST) dataset, and the value $1000$ (or $5000$) denotes the sample size of $X$ (or $Y$).  The units of all reported wall-clock times are seconds.}
\label{tb:pu_time}
\end{table}

\clearpage
\begin{table}

\vskip 0.15in
\begin{center}
\begin{sc}
\begin{tabular}{lccccc}
\toprule
Dataset & Init Method & Init Accuracy & MPGW & UGW & PGW (ours) \\
\midrule
surf(A) $\to$ surf(A) & POT & 81.2\% & 74.7\% & 66.5\% & 74.7\% \\
surf(A) $\to$ surf(A) & FLB-U & 64.9\% & 65.7\% & 66.5\% & 65.7\% \\
surf(A) $\to$ surf(A) & FLB-P & 63.3\% & 66.5\% & 66.5\% & 66.5\% \\

decaf(A) $\to$ decaf(A) & POT & 95.1\% & 95.1\% & 60.8\% & 95.1\% \\
decaf(A) $\to$ decaf(A) & FLB-U & 78.0\% & 67.4\% & 83.7\% & 67.4\% \\
decaf(A) $\to$ decaf(A) & FLB-P & 78.0\% & 74.7\% & 88.6\% & 74.7\% \\

surf(D) $\to$ surf(D) & POT & 100\% & 100\% & 89.3\% & 100\% \\
surf(D) $\to$ surf(D) & FLB-U & 62.7\% & 73.3\% & 84.0\% & 73.3\% \\
surf(D) $\to$ surf(D) & FLB-P & 60.0\% & 60.0\% & 78.7\% & 60.0\% \\

decaf(D) $\to$ decaf(D) & POT & 100\% & 100\% & 100\% & 100\% \\
decaf(D) $\to$ decaf(D) & FLB-U & 76.0\% & 68.0\% & 70.7\% & 68.0\% \\
decaf(D) $\to$ decaf(D) & FLB-P & 73.3\% & 73.3\% & 86.7\% & 73.3\% \\

surf(W) $\to$ surf(W) & POT & 100.0\% & 100.0\% & 81.3\% & 100.0\% \\
surf(W) $\to$ surf(W) & FLB-U & 76.0\% & 70.7\% & 81.3\% & 70.7\% \\
surf(W) $\to$ surf(W) & FLB-P & 73.3\% & 68.0\% & 78.7\% & 68.0\% \\

decaf(W) $\to$ decaf(W) & POT & 100\% & 100\% & 100\% & 100\% \\
decaf(W) $\to$ decaf(W) & FLB-U & 73.3\% & 68.0\% & 62.7\% & 68.0\% \\
decaf(W) $\to$ decaf(W) & FLB-P & 70.7\% & 70.7\% & 73.3\% & 70.7\% \\

\midrule

surf(A) $\to$ decaf(A) & FLB-U & 73.9\% & 83.7\% & 91.8\% & 83.7\% \\
surf(A) $\to$ decaf(A) & FLB-P & 73.9\% & 83.7\% & 87.8\% & 83.7\% \\

decaf(A) $\to$ surf(A) & FLB-U & 67.3\% & 67.3\% & 69.0\% & 67.3\% \\
decaf(A) $\to$ surf(A) & FLB-P & 67.3\% & 68.2\% & 71.4\% & 68.2\% \\

surf(D) $\to$ decaf(D) & FLB-U & 76.0\% & 76.0\% & 65.3\% & 76.0\% \\
surf(D) $\to$ decaf(D) & FLB-P & 76.0\% & 76.0\% & 65.3\% & 76.0\% \\

decaf(D) $\to$ surf(D) & FLB-U & 73.3\% & 62.7\% & 73.3\% & 62.7\% \\
decaf(D) $\to$ surf(D) & FLB-P & 73.3\% & 73.3\% & 73.3\% & 73.3\% \\

surf(W) $\to$ decaf(W) & FLB-U & 70.7\% & 70.7\% & 76.0\% & 70.7\% \\
surf(W) $\to$ decaf(W) & FLB-P & 70.7\% & 70.7\% & 76.0\% & 70.7\% \\

decaf(W) $\to$ surf(W) & FLB-U & 68.0\% & 68.0\% & 65.3\% & 68.0\% \\
decaf(W) $\to$ surf(W) & FLB-P & 68.0\% & 68.0\% & 70.7\% & 68.0\% \\
\bottomrule
\end{tabular}
\end{sc}
\end{center}
\vskip -0.0in
\caption{In this table, we present the accuracy comparison of the MPGW, UGW, and the proposed PGW method. We first describe the initialization method and report its accuracy, followed by the accuracy of MPGW, UGW, and PGW. The prior distribution $\pi=p(l=1)$ is set to be 0.2 in all experiments. To guarantee the SCAR assumption, for Surf(A) and Decaf(A), we set $n=50$, which is half of the total number of data in one single class. $m$ is set to be $250$. Similarly, we set suitable $n,m$ for Surf(D), Decaf(D), Surf(W), Decaf(W).}
\label{tb:pu_accuracy2}
\end{table}

\clearpage

\begin{table}

\vskip 0.15in
\begin{center}
\begin{sc}
\begin{tabular}{lccccc}
\toprule
Dataset & Init Method & Init Time & MPGW & UGW & PGW (ours) \\
\midrule
surf(A) $\to$ surf(A) & POT & 1.4e-3 & 1.9e-2 & 3.8 & 2.0e-2 \\
surf(A) $\to$ surf(A) & FLB-U & 2.2e-3 & 1.8e-2 & 3.6 & 1.9e-2 \\
surf(A) $\to$ surf(A) & FLB-P & 1.7e-3 & 1.8e-2 & 3.8 & 1.5e-2 \\

decaf(A) $\to$ decaf(A) & POT & 1.7e-3 & 1.9e-2 & 7.3 & 1.9e-2 \\
decaf(A) $\to$ decaf(A) & FLB-U & 9.6e-3 & 1.8e-2 & 6.8 & 1.5e-2 \\
decaf(A) $\to$ decaf(A) & FLB-P & 2.0e-3 & 1.8e-2 & 6.7 & 1.6e-2 \\

surf(D) $\to$ surf(D) & POT & 2.9e-4 & 5.8e-4 & 3.1 & 3.8e-4 \\
surf(D) $\to$ surf(D) & FLB-U & 1.4e-3 & 3.0e-3 & 5.4 & 2.2e-3 \\
surf(D) $\to$ surf(D) & FLB-P & 3.1e-4 & 2.9e-3 & 5.4 & 2.1e-3 \\

decaf(D) $\to$ decaf(D) & POT & 3.1e-4 & 6.0e-4 & 3.3 & 3.6e-4 \\
decaf(D) $\to$ decaf(D) & FLB-U & 1.4e-3 & 2.9e-3 & 5.8 & 2.1e-3 \\
decaf(D) $\to$ decaf(D) & FLB-P & 3.4e-4 & 2.8e-3 & 5.3 & 2.0e-3 \\

surf(W) $\to$ surf(W) & POT & 3.0e-4 & 6.0e-4 & 5.2 & 3.6e-4 \\
surf(W) $\to$ surf(W) & FLB-U & 1.3e-3 & 2.9e-3 & 5.1 & 2.1e-3 \\
surf(W) $\to$ surf(W) & FLB-P & 3.3e-4 & 2.9e-3 & 5.1 & 2.1e-3 \\

decaf(W) $\to$ decaf(W) & POT & 3.3e-4 & 6.2e-4 & 3.3 & 3.4e-4 \\
decaf(W) $\to$ decaf(W) & FLB-U & 1.2e-3 & 2.9e-3 & 5.8 & 2.1e-3 \\
decaf(W) $\to$ decaf(W) & FLB-P & 3.3e-4 & 2.8e-3 & 5.4 & 2.0e-3 \\

\midrule

surf(A) $\to$ decaf(A) & FLB-U & 1.1e-1 & 2.8e-2 & 6.7 & 2.6e-2 \\
surf(A) $\to$ decaf(A) & FLB-P & 1.9e-3 & 2.2e-2 & 0.2 & 2.1e-2 \\

decaf(A) $\to$ surf(A) & FLB-U & 0.1 & 5e-2 & 6.7 & 4e-2 \\
decaf(A) $\to$ surf(A) & FLB-P & 2e-3 & 1.8 & 6.8 & 1.5 \\

surf(D) $\to$ decaf(D) & FLB-U & 1.8e-3 & 5.3e-3 & 6.0 & 2.3e-3 \\
surf(D) $\to$ decaf(D) & FLB-P & 3.5e-4 & 3.9e-4 & 5.9 & 3.8e-4 \\

decaf(D) $\to$ surf(D) & FLB-U & 1.8e-3 & 0.296 & 5.6 & 0.165 \\
decaf(D) $\to$ surf(D) & FLB-P & 3.3e-4 & 0.218 & 5.6 & 0.170 \\

surf(W) $\to$ decaf(W) & FLB-U & 1.8e-3 & 5.3e-3 & 5.0 & 2.3e-3 \\
surf(W) $\to$ decaf(W) & FLB-P & 3.4e-4 & 4.1e-4 & 5.0 & 3.9e-4 \\

decaf(W) $\to$ surf(W) & FLB-U & 1.8e-3 & 5.1e-3 & 5.8 & 2.1e-3 \\
decaf(W) $\to$ surf(W) & FLB-P & 3.4e-4 & 2.9e-3 & 5.6 & 2.2e-3 \\
\bottomrule
\end{tabular}
\end{sc}
\end{center}
\vskip -0.0in
\caption{In this table, we present the wall-clock time comparison of the MPGW, UGW, and the proposed PGW method. We report the initialization method and its wall-clock time, followed by the wall-clock time of each of the methods MPGW, UGW, and PGW. The units of all reported wall-clock times are seconds. The prior distribution $\pi=p(l=1)$ is set to be 0.2 in all experiments. To guarantee the SCAR assumption, for Surf(A) and Decaf(A), we set $n=50$, which is half of the total number of data in one single class. $m$ is set to be $250$. Similarly, we set suitable $n,m$ for Surf(D), Decaf(D), Surf(W), Decaf(W).}
\label{tb:pu_time2}
\end{table}

\end{document}

%% file: iclr2025/math_commands.tex
\usepackage{amsmath,amsfonts,bm}

\def\eqref#1{equation~\ref{#1}}

\def\1{\bm{1}}

\DeclareMathAlphabet{\mathsfit}{\encodingdefault}{\sfdefault}{m}{sl}
\SetMathAlphabet{\mathsfit}{bold}{\encodingdefault}{\sfdefault}{bx}{n}



%% file: packages.tex
\usepackage{amsthm}
\usepackage{amsmath}
\usepackage{amssymb}
\usepackage{caption}
\usepackage{hyperref}
\usepackage{subcaption}
\usepackage[english]{babel}
\usepackage{booktabs}
\usepackage{bbm}

\usepackage{blindtext}
\usepackage{multicol}
\usepackage[shortlabels]{enumitem}
\usepackage[linesnumbered,ruled,vlined]{algorithm2e}

\SetCommentSty{mycommfont}

\SetKwInput{KwInput}{Input}                
\SetKwInput{KwOutput}{Output}              

\usepackage{graphicx} 

\usepackage{graphicx} %
\usepackage[linesnumbered,ruled,vlined]{algorithm2e}

\newtheorem{theorem}{Theorem}[section]
\newtheorem{proposition}[theorem]{Proposition}

\newtheorem{remark}[theorem]{Remark}

\newtheorem{lemma}[theorem]{Lemma}
\newtheorem{examplee}[theorem]{Example}

\usepackage{algorithmic}

%% file: arxiv_fixed.bbl
\begin{thebibliography}{10}

\bibitem{Villani2009Optimal}
Cedric Villani.
\newblock {\em Optimal transport: old and new}.
\newblock Springer, 2009.

\bibitem{arjovsky2017wasserstein}
Martin Arjovsky, Soumith Chintala, and L{\'e}on Bottou.
\newblock {Wasserstein} generative adversarial networks.
\newblock In {\em International conference on machine learning}, pages 214--223. PMLR, 2017.

\bibitem{gulrajani2017improved}
Ishaan Gulrajani, Faruk Ahmed, Martin Arjovsky, Vincent Dumoulin, and Aaron~C Courville.
\newblock Improved training of wasserstein gans.
\newblock {\em Advances in neural information processing systems}, 30, 2017.

\bibitem{courty2017joint}
Nicolas Courty, R{\'e}mi Flamary, Amaury Habrard, and Alain Rakotomamonjy.
\newblock Joint distribution optimal transportation for domain adaptation.
\newblock {\em Advances in neural information processing systems}, 30, 2017.

\bibitem{kolouri2020wasserstein}
Soheil Kolouri, Navid Naderializadeh, Gustavo~K Rohde, and Heiko Hoffmann.
\newblock Wasserstein embedding for graph learning.
\newblock In {\em International Conference on Learning Representations}, 2020.

\bibitem{chizat2018unbalanced}
Lenaic Chizat, Gabriel Peyr{\'e}, Bernhard Schmitzer, and Fran{\c{c}}ois-Xavier Vialard.
\newblock Unbalanced optimal transport: Dynamic and {Kantorovich} formulations.
\newblock {\em Journal of Functional Analysis}, 274(11):3090--3123, 2018.

\bibitem{figalli2010optimal}
Alessio Figalli.
\newblock The optimal partial transport problem.
\newblock {\em Archive for rational mechanics and analysis}, 195(2):533--560, 2010.

\bibitem{memoli2011gromov}
Facundo M{\'e}moli.
\newblock Gromov--wasserstein distances and the metric approach to object matching.
\newblock {\em Foundations of computational mathematics}, 11:417--487, 2011.

\bibitem{gilmer2017neural}
Justin Gilmer, Samuel~S Schoenholz, Patrick~F Riley, Oriol Vinyals, and George~E Dahl.
\newblock Neural message passing for quantum chemistry.
\newblock In {\em International conference on machine learning}, pages 1263--1272. PMLR, 2017.

\bibitem{alvarez2018gromov}
David Alvarez-Melis and Tommi Jaakkola.
\newblock Gromov-wasserstein alignment of word embedding spaces.
\newblock In {\em Proceedings of the 2018 Conference on Empirical Methods in Natural Language Processing}, pages 1881--1890, 2018.

\bibitem{fatras2021unbalanced}
Kilian Fatras, Thibault S{\'e}journ{\'e}, R{\'e}mi Flamary, and Nicolas Courty.
\newblock Unbalanced minibatch optimal transport; applications to domain adaptation.
\newblock In {\em International Conference on Machine Learning}, pages 3186--3197. PMLR, 2021.

\bibitem{caffarelli2010free}
Luis~A Caffarelli and Robert~J McCann.
\newblock Free boundaries in optimal transport and monge-ampere obstacle problems.
\newblock {\em Annals of mathematics}, pages 673--730, 2010.

\bibitem{figalli2010new}
Alessio Figalli and Nicola Gigli.
\newblock A new transportation distance between non-negative measures, with applications to gradients flows with dirichlet boundary conditions.
\newblock {\em Journal de math{\'e}matiques pures et appliqu{\'e}es}, 94(2):107--130, 2010.

\bibitem{nguyen2024POT}
Anh~Duc Nguyen, Tuan~Dung Nguyen, Quang Nguyen, Hoang Nguyen, Lam~M. Nguyen, and Kim-Chuan Toh.
\newblock On partial optimal transport: Revised sinkhorn and efficient gradient methods.
\newblock In {\em Proceedings of the AAAI Conference on Artificial Intelligence}, volume~38, 2024.

\bibitem{georgiou2008metrics}
Tryphon~T Georgiou, Johan Karlsson, and Mir~Shahrouz Takyar.
\newblock Metrics for power spectra: an axiomatic approach.
\newblock {\em IEEE Transactions on Signal Processing}, 57(3):859--867, 2008.

\bibitem{piccoli2014generalized}
Benedetto Piccoli and Francesco Rossi.
\newblock Generalized wasserstein distance and its application to transport equations with source.
\newblock {\em Archive for Rational Mechanics and Analysis}, 211:335--358, 2014.

\bibitem{guittet2002extended}
Kevin Guittet.
\newblock {\em Extended Kantorovich norms: a tool for optimization}.
\newblock PhD thesis, INRIA, 2002.

\bibitem{heinemann2023kantorovich}
Florian Heinemann, Marcel Klatt, and Axel Munk.
\newblock Kantorovich--rubinstein distance and barycenter for finitely supported measures: Foundations and algorithms.
\newblock {\em Applied Mathematics \& Optimization}, 87(1):4, 2023.

\bibitem{lellmann2014imaging}
Jan Lellmann, Dirk~A Lorenz, Carola Schonlieb, and Tuomo Valkonen.
\newblock Imaging with kantorovich--rubinstein discrepancy.
\newblock {\em SIAM Journal on Imaging Sciences}, 7(4):2833--2859, 2014.

\bibitem{chizat2018interpolating}
Lenaic Chizat, Gabriel Peyr{\'e}, Bernhard Schmitzer, and Fran{\c{c}}ois-Xavier Vialard.
\newblock An interpolating distance between optimal transport and {Fisher--Rao} metrics.
\newblock {\em Foundations of Computational Mathematics}, 18(1):1--44, 2018.

\bibitem{Liero2018Optimal}
Matthias Liero, Alexander Mielke, and Giuseppe Savare.
\newblock Optimal entropy-transport problems and a new {Hellinger--Kantorovich} distance between positive measures.
\newblock {\em Inventiones mathematicae}, 211(3):969--1117, 2018.

\bibitem{raghvendra2024new}
Sharath Raghvendra, Pouyan Shirzadian, and Kaiyi Zhang.
\newblock A new robust partial $ p $-wasserstein-based metric for comparing distributions.
\newblock {\em arXiv preprint arXiv:2405.03664}, 2024.

\bibitem{nietert2023robust}
Sloan Nietert, Rachel Cummings, and Ziv Goldfeld.
\newblock Robust estimation under the wasserstein distance.
\newblock {\em arXiv preprint arXiv:2302.01237}, 2023.

\bibitem{balaji2020robust}
Yogesh Balaji, Rama Chellappa, and Soheil Feizi.
\newblock Robust optimal transport with applications in generative modeling and domain adaptation.
\newblock {\em Advances in Neural Information Processing Systems}, 33:12934--12944, 2020.

\bibitem{nguyen2023unbalanced}
Quang~Minh Nguyen, Hoang~H Nguyen, Yi~Zhou, and Lam~M Nguyen.
\newblock On unbalanced optimal transport: Gradient methods, sparsity and approximation error.
\newblock {\em The Journal of Machine Learning Research}, 2023.

\bibitem{le2021robust}
Khang Le, Huy Nguyen, Quang~M Nguyen, Tung Pham, Hung Bui, and Nhat Ho.
\newblock On robust optimal transport: Computational complexity and barycenter computation.
\newblock {\em Advances in Neural Information Processing Systems}, 34:21947--21959, 2021.

\bibitem{bronstein2006generalized}
Alexander~M Bronstein, Michael~M Bronstein, and Ron Kimmel.
\newblock Generalized multidimensional scaling: a framework for isometry-invariant partial surface matching.
\newblock {\em Proceedings of the National Academy of Sciences}, 103(5):1168--1172, 2006.

\bibitem{memoli2009spectral}
Facundo M{\'e}moli.
\newblock Spectral gromov-wasserstein distances for shape matching.
\newblock In {\em 2009 IEEE 12th International Conference on Computer Vision Workshops, ICCV Workshops}, pages 256--263. IEEE, 2009.

\bibitem{xu2019scalable}
Hongteng Xu, Dixin Luo, and Lawrence Carin.
\newblock Scalable gromov-wasserstein learning for graph partitioning and matching.
\newblock {\em Advances in neural information processing systems}, 32, 2019.

\bibitem{edwards1975structure}
David~A Edwards.
\newblock The structure of superspace.
\newblock In {\em Studies in topology}, pages 121--133. Elsevier, 1975.

\bibitem{gromov1981structures}
Mikhael Gromov.
\newblock Structures m{\'e}triques pour les vari{\'e}t{\'e}s riemanniennes.
\newblock {\em Textes Math.}, 1, 1981.

\bibitem{gromov1981groups}
Michael Gromov.
\newblock Groups of polynomial growth and expanding maps (with an appendix by jacques tits).
\newblock {\em Publications Math{\'e}matiques de l'IH{\'E}S}, 53:53--78, 1981.

\bibitem{burago2001course}
Dmitri Burago, Yuri Burago, Sergei Ivanov, et~al.
\newblock {\em A course in metric geometry}, volume~33.
\newblock American Mathematical Society Providence, 2001.

\bibitem{sturm2023space}
Karl-Theodor Sturm.
\newblock {\em The space of spaces: curvature bounds and gradient flows on the space of metric measure spaces}, volume 290.
\newblock American Mathematical Society, 2023.

\bibitem{frank1956algorithm}
Marguerite Frank, Philip Wolfe, et~al.
\newblock An algorithm for quadratic programming.
\newblock {\em Naval research logistics quarterly}, 3(1-2):95--110, 1956.

\bibitem{lacoste2016convergence}
Simon Lacoste-Julien.
\newblock Convergence rate of frank-wolfe for non-convex objectives.
\newblock {\em arXiv preprint arXiv:1607.00345}, 2016.

\bibitem{cuturi2013sinkhorn}
Marco Cuturi.
\newblock Sinkhorn distances: Lightspeed computation of optimal transport.
\newblock {\em Advances in neural information processing systems}, 26, 2013.

\bibitem{papadakis2014optimal}
Nicolas Papadakis, Gabriel Peyr{\'e}, and Edouard Oudet.
\newblock Optimal transport with proximal splitting.
\newblock {\em SIAM Journal on Imaging Sciences}, 7(1):212--238, 2014.

\bibitem{benamou2014numerical}
Jean-David Benamou, Brittany~D Froese, and Adam~M Oberman.
\newblock Numerical solution of the optimal transportation problem using the monge--amp{\`e}re equation.
\newblock {\em Journal of Computational Physics}, 260:107--126, 2014.

\bibitem{benamou2015iterative}
Jean-David Benamou, Guillaume Carlier, Marco Cuturi, Luca Nenna, and Gabriel Peyr{\'e}.
\newblock Iterative bregman projections for regularized transportation problems.
\newblock {\em SIAM Journal on Scientific Computing}, 37(2):A1111--A1138, 2015.

\bibitem{peyre2019computational}
Gabriel Peyr{\'e}, Marco Cuturi, et~al.
\newblock Computational optimal transport: With applications to data science.
\newblock {\em Foundations and Trends{\textregistered} in Machine Learning}, 11(5-6):355--607, 2019.

\bibitem{chizat2018scaling}
Lenaic Chizat, Gabriel Peyr{\'e}, Bernhard Schmitzer, and Fran{\c{c}}ois-Xavier Vialard.
\newblock Scaling algorithms for unbalanced optimal transport problems.
\newblock {\em Mathematics of Computation}, 87(314):2563--2609, 2018.

\bibitem{Bonneel2019sliced}
Nicolas Bonneel and David Coeurjolly.
\newblock {SPOT}: sliced partial optimal transport.
\newblock {\em ACM Transactions on Graphics}, 38(4):1--13, 2019.

\bibitem{bai2022sliced}
Yikun Bai, Bernhard Schmitzer, Matthew Thorpe, and Soheil Kolouri.
\newblock Sliced optimal partial transport.
\newblock In {\em Proceedings of the IEEE/CVF Conference on Computer Vision and Pattern Recognition}, pages 13681--13690, 2023.

\bibitem{peyre2016gromov}
Gabriel Peyr{\'e}, Marco Cuturi, and Justin Solomon.
\newblock Gromov-wasserstein averaging of kernel and distance matrices.
\newblock In {\em International conference on machine learning}, pages 2664--2672. PMLR, 2016.

\bibitem{xu2019gromov}
Hongteng Xu, Dixin Luo, Hongyuan Zha, and Lawrence~Carin Duke.
\newblock Gromov-wasserstein learning for graph matching and node embedding.
\newblock In {\em International conference on machine learning}, pages 6932--6941. PMLR, 2019.

\bibitem{titouan2019optimal}
Vayer Titouan, Nicolas Courty, Romain Tavenard, and R{\'e}mi Flamary.
\newblock Optimal transport for structured data with application on graphs.
\newblock In {\em International Conference on Machine Learning}, pages 6275--6284. PMLR, 2019.

\bibitem{sejourne2021unbalanced}
Thibault S{\'e}journ{\'e}, Fran{\c{c}}ois-Xavier Vialard, and Gabriel Peyr{\'e}.
\newblock The unbalanced gromov wasserstein distance: Conic formulation and relaxation.
\newblock {\em Advances in Neural Information Processing Systems}, 34:8766--8779, 2021.

\bibitem{chapel2020partial}
Laetitia Chapel, Mokhtar~Z Alaya, and Gilles Gasso.
\newblock Partial optimal tranport with applications on positive-unlabeled learning.
\newblock {\em Advances in Neural Information Processing Systems}, 33:2903--2913, 2020.

\bibitem{de2022entropy}
Nicol{\`o} De~Ponti and Andrea Mondino.
\newblock Entropy-transport distances between unbalanced metric measure spaces.
\newblock {\em Probability Theory and Related Fields}, 184(1-2):159--208, 2022.

\bibitem{liu2020partial}
Weijie Liu, Chao Zhang, Jiahao Xie, Zebang Shen, Hui Qian, and Nenggan Zheng.
\newblock Partial gromov-wasserstein learning for partial graph matching.
\newblock {\em arXiv preprint arXiv:2012.01252}, 2020.

\bibitem{thual2022aligning}
Alexis Thual, Quang~Huy Tran, Tatiana Zemskova, Nicolas Courty, R{\'e}mi Flamary, Stanislas Dehaene, and Bertrand Thirion.
\newblock Aligning individual brains with fused unbalanced gromov wasserstein.
\newblock {\em Advances in Neural Information Processing Systems}, 35:21792--21804, 2022.

\bibitem{villani2021topics}
C{\'e}dric Villani.
\newblock {\em Topics in optimal transportation}, volume~58.
\newblock American Mathematical Soc., 2021.

\bibitem{kong2024outlier}
Lemin Kong, Jiajin Li, Jianheng Tang, and Anthony Man-Cho So.
\newblock Outlier-robust gromov-wasserstein for graph data.
\newblock {\em Advances in Neural Information Processing Systems}, 36, 2024.

\bibitem{beier2022linear}
Florian Beier, Robert Beinert, and Gabriele Steidl.
\newblock On a linear gromov--wasserstein distance.
\newblock {\em IEEE Transactions on Image Processing}, 31:7292--7305, 2022.

\bibitem{vay2019fgw}
Vayer Titouan, Nicolas Courty, Romain Tavenard, Chapel Laetitia, and R{\'e}mi Flamary.
\newblock Optimal transport for structured data with application on graphs.
\newblock In Kamalika Chaudhuri and Ruslan Salakhutdinov, editors, {\em Proceedings of the 36th International Conference on Machine Learning}, volume~97 of {\em Proceedings of Machine Learning Research}, pages 6275--6284, Long Beach, California, USA, 09--15 Jun 2019. PMLR.

\bibitem{liu2023ptlp}
Xinran Liu, Yikun Bai, Huy Tran, Zhanqi Zhu, Matthew Thorpe, and Soheil Kolouri.
\newblock Ptlp: Partial transport $ l^p $ distances.
\newblock In {\em NeurIPS 2023 Workshop Optimal Transport and Machine Learning}, 2023.

\bibitem{santambrogio2015optimal}
Filippo Santambrogio.
\newblock Optimal transport for applied mathematicians.
\newblock {\em Birk{\"a}user, NY}, 55(58-63):94, 2015.

\bibitem{cagniart2010free}
Cedric Cagniart, Edmond Boyer, and Slobodan Ilic.
\newblock Free-form mesh tracking: a patch-based approach.
\newblock In {\em 2010 IEEE Computer Society Conference on Computer Vision and Pattern Recognition}, pages 1339--1346. IEEE, 2010.

\bibitem{cang2024supervised}
Zixuan Cang, Yaqi Wu, and Yanxiang Zhao.
\newblock Supervised gromov-wasserstein optimal transport.
\newblock {\em arXiv preprint arXiv:2401.06266}, 2024.

\bibitem{cuturi2014fast}
Marco Cuturi and Arnaud Doucet.
\newblock Fast computation of wasserstein barycenters.
\newblock In {\em International conference on machine learning}, pages 685--693. PMLR, 2014.

\bibitem{Villani2003Topics}
Cedric Villani.
\newblock {\em Topics in Optimal Transportation}.
\newblock American Mathematical Society, 2003.

\bibitem{gromov2001metric}
M~Gromov.
\newblock Metric structures for riemannian and non-riemannian spaces.
\newblock {\em Bulletin of the American Mathematical Society}, 38:353--363, 2001.

\bibitem{flamary2021pot}
R{\'e}mi Flamary, Nicolas Courty, Alexandre Gramfort, Mokhtar~Z. Alaya, Aur{\'e}lie Boisbunon, Stanislas Chambon, Laetitia Chapel, Adrien Corenflos, Kilian Fatras, Nemo Fournier, L{\'e}o Gautheron, Nathalie~T.H. Gayraud, Hicham Janati, Alain Rakotomamonjy, Ievgen Redko, Antoine Rolet, Antony Schutz, Vivien Seguy, Danica~J. Sutherland, Romain Tavenard, Alexander Tong, and Titouan Vayer.
\newblock Pot: Python optimal transport.
\newblock {\em Journal of Machine Learning Research}, 22(78):1--8, 2021.

\bibitem{bekker2020learning}
Jessa Bekker and Jesse Davis.
\newblock Learning from positive and unlabeled data: A survey.
\newblock {\em Machine Learning}, 109:719--760, 2020.

\bibitem{elkan2008learning}
Charles Elkan and Keith Noto.
\newblock Learning classifiers from only positive and unlabeled data.
\newblock In {\em Proceedings of the 14th ACM SIGKDD international conference on Knowledge discovery and data mining}, pages 213--220, 2008.

\bibitem{kato2018learning}
Masahiro Kato, Takeshi Teshima, and Junya Honda.
\newblock Learning from positive and unlabeled data with a selection bias.
\newblock In {\em International conference on learning representations}, 2018.

\bibitem{hsieh2019classification}
Yu-Guan Hsieh, Gang Niu, and Masashi Sugiyama.
\newblock Classification from positive, unlabeled and biased negative data.
\newblock In {\em International Conference on Machine Learning}, pages 2820--2829. PMLR, 2019.

\bibitem{saenko2010adapting}
Kate Saenko, Brian Kulis, Mario Fritz, and Trevor Darrell.
\newblock Adapting visual category models to new domains.
\newblock In {\em Computer Vision--ECCV 2010: 11th European Conference on Computer Vision, Heraklion, Crete, Greece, September 5-11, 2010, Proceedings, Part IV 11}, pages 213--226. Springer, 2010.

\bibitem{donahue2014decaf}
Jeff Donahue, Yangqing Jia, Oriol Vinyals, Judy Hoffman, Ning Zhang, Eric Tzeng, and Trevor Darrell.
\newblock Decaf: A deep convolutional activation feature for generic visual recognition.
\newblock In {\em International conference on machine learning}, pages 647--655. PMLR, 2014.

\bibitem{sejourne2023unbalanced}
Thibault S{\'e}journ{\'e}, Cl{\'e}ment Bonet, Kilian Fatras, Kimia Nadjahi, and Nicolas Courty.
\newblock Unbalanced optimal transport meets sliced-wasserstein.
\newblock {\em arXiv preprint arXiv:2306.07176}, 2023.

\end{thebibliography}
